\documentclass{article}
\usepackage{iclr2023_conference}
\iclrfinalcopy 

\usepackage{preamble}  %
\graphicspath{{figs/}} %
\hypersetup{final}  %

\usepackage{lipsum}

\usepackage{amsmath,amsfonts,bm}

\newcommand{\captiona}{{\em (a)}}
\newcommand{\captionb}{{\em (b)}}

\def\eqref#1{equation~\ref{#1}}

\def\1{\bm{1}}

\def\eps{{\epsilon}}

\def\rv{{\textnormal{v}}}

\def\va{{\bm{a}}}

\def\vh{{\bm{h}}}

\def\mI{{\bm{I}}}

\def\mW{{\bm{W}}}

\DeclareMathAlphabet{\mathsfit}{\encodingdefault}{\sfdefault}{m}{sl}
\SetMathAlphabet{\mathsfit}{bold}{\encodingdefault}{\sfdefault}{bx}{n}

\newcommand{\affnum}[1]{{\normalsize\rm\textsuperscript{\,#1}}}
\newcommand{\affiliation}[2]{\normalsize\rm\textsuperscript{#1}#2}
\newcommand{\nname}[1]{\textbf{#1}}

\title{Learnable Graph Convolutional\\ Attention Networks}

\newcounter{phase}
\newcommand{\added}[2]{\ifthenelse{\value{phase}=#1}{\orange{#2}}{#2}}
\newcommand{\removed}[2]{\ifthenelse{\value{phase}=#1}{\sout{#2}}{\nop{}}}
\newcommand{\replaced}[3]{\ifthenelse{\value{phase}=#1}{\red{\sout{#2}} \orange{$\rightarrow$ #3}}{#3}}

\AtBeginDocument{\colorlet{defaultcolor}{.}}
\newenvironment{added*}[1]{\ifthenelse{\value{phase}=#1}{\color{orange}}{}}{\color{defaultcolor}}

\crefname{Corollary}{Cor.}{Corollaries}
\crefname{corollary}{Cor.}{Corollaries}
\crefname{Theorem}{Thm.}{Theorems}
\crefname{theorem}{Thm.}{Theorems}
\crefname{Appendix}{App.}{Apps.}
\crefname{appendix}{App.}{Apps.}

\author{%
\nname{Adri\'an~Javaloy}\affnum{1,}\thanks{Equal contribution. Correspondence to: \texttt{\href{mailto:ajavaloy@cs.uni-saarland.de}{\{ajavaloy,sanchez\}@cs.uni-saarland.de}}.}   \hspace{5pt}
\nname{Pablo~S\'anchez-Mart\'in}\affnum{1,2,$*$}  \hspace{5pt}
\nname{Amit~Levi}\affnum{3}  \hspace{5pt}
\nname{Isabel~Valera}\affnum{1,4} \\[6pt]
\affiliation{1}{Department of Computer Science of Saarland University, Saarbrücken, Germany} \\
\affiliation{2}{Max Planck Institute for Intelligent Systems, Tübingen, Germany}\\
\affiliation{3}{Huawei Noah’s Ark Lab, Montreal, Canada}\\
\affiliation{4}{Max Planck Institute for Software Systems, Saarbrücken, Germany}\\
}

\begin{document}
	\doparttoc %
	\faketableofcontents %

    \maketitle

\vspace{-0.4cm}
\begin{abstract}

Existing Graph Neural Networks (GNNs) compute the message exchange between nodes by either aggregating uniformly (\emph{convolving})  the features of all the neighboring nodes, or by applying a non-uniform  
score (\emph{attending}) to the features. %
Recent works have shown %
the strengths and weaknesses of the resulting GNN architectures, respectively, GCNs and GATs. 
In this work, we aim at exploiting the strengths of both approaches to their full extent. 
To this end, we first introduce the graph convolutional attention layer (\ours), which relies on convolutions to compute the attention scores. 
Unfortunately, as in the case of GCNs and GATs, we  show that there exists no clear winner between the three---neither theoretically nor in practice---as their performance %
directly depends on the nature of the data (i.e., of the graph and features). 
This result brings us to the main contribution of our work, the learnable graph convolutional attention network (\ours*): a GNN architecture that  automatically interpolates between GCN, GAT and \ours in each layer, by adding  two scalar parameters. %
Our results demonstrate that \ours* is able to efficiently combine different GNN layers along the network, outperforming competing methods in a wide range of datasets, and resulting in a more robust model that reduces the need of cross-validating. %
\end{abstract}

\section{Introduction\pages{1}} \label{sec:intro}

\msg{Intro about the importance and raise of GNNs. For ideas see \url{https://dl.acm.org/doi/pdf/10.1145/3495161}}

In recent years, Graph Neural Networks~(GNNs)~\citep{scarselli2008gnns} have become ubiquitous in machine learning, emerging as the standard approach in many settings.
For example, they have been successfully applied for tasks such as topic prediction in citation networks~\citep{sen2008collective}; molecule prediction~\citep{gilmer2017neural}; and link prediction in recommender systems~\citep{wu2020graph}.
These applications typically make use of message-passing GNNs~\citep{gilmer2017neural}, whose idea is fairly simple: 
in each layer, nodes are updated by aggregating the information (messages) coming from their neighboring nodes.

Depending on how this aggregation is implemented, we can define different types of GNN layers.
Two important and widely adopted layers are graph convolutional networks~(GCNs)~\citep{kipf2016semi}, which uniformly average the neighboring information; and graph attention networks~(GATs)~\citep{velivckovic2017graph}, which instead perform a weighted average, based on an attention score between receiver and sender nodes.
\msg{Previous work by Kimon et. al. studied the scenarios and limits of both convolutional and attention-based layers on a controlled environment (stochastic block model). KEY IDEA: hard to tell apriori which one to use.}
More recently, a number of works have shown the strengths and limitations of both approaches from a theoretical~\citep{kimongat,kimongcn,baranwal2022effects}, and empirical~\citep{knyazev2019understanding} point of view.\ajnote{more refs?}
These results show that their performance depends on the nature of the data at hand (i.e., the graph and the features), 
thus the standard approach %
is to select between GCNs and GATs %
via computationally demanding cross-validation. %

\msg{We leverage these studies and propose to enhance GAT models by convolving the input of the score. We add two learnable parameters that control the amount of convolution and convolved-score, unifying GCNs and GATs.}

In this work, we aim to  exploit the benefits of both convolution and attention operations in the design of GNN architectures. 
To this end, we first introduce a novel  graph convolutional attention layer (\ours), which extends existing attention layers by 
taking the convolved features as inputs of the score function. 
Following~\citep{kimongat}, we rely on  a contextual stochastic block model to theoretically compare  GCN, GAT, and \ours architectures. %
Our analysis shows that, unfortunately, no free lunch exists among these three GNN architectures since their performance, as expected, is fully data-dependent.

This motivates the main contribution of the paper, the \emph{learnable graph convolutional attention network} (\ours*): a novel GNN  which, in each layer, automatically interpolates between the three operations by introducing only two scalar parameters.
As a result, \ours* is able to learn the proper operation to apply at each  layer, thus  combining different layer types in the same GNN architecture while overcoming the need to cross-validate---a process that was prohibitively expensive prior to this work.
Our extensive empirical analysis demonstrates the capabilities of \ours* on a wide range of datasets, outperforming existing baseline GNNs in terms of both performance, and robustness to input noise and network initialization.

\vspace{-0.1cm}
\section{Preliminaries\pages{0.75}} \label{sec:preliminaries}

\msg{Graph}

Assume as input an undirected graph ${G} = ({V}, {E})$, where ${V} = [n]$ denotes the set of vertices of the graph, and ${E} \subseteq {V} \times {V}$ the set of edges. Each node $i\in[n]$ is represented by a $d$-dimensional feature vector $\bX_i\in \R^d$, and the goal is to produce a set of predictions $\{\hat \by_i\}_{i=1}^n$.
\msg{Message-passing GNNs. Explain what are they in general terms.}
To this end, %
a message-passing GNN layer yields a representation $\tilde\vh_i \in \R^{d^\prime}$ for each node $i$, by collecting and aggregating the information from each of its neighbors into a single message; and using the aggregated message to update its representation from the previous layer, $\vh_i \in \R^{d}$.
For the purposes of this work, we can define this operation as the following:
\begin{equation}
    \tilde\vh_i = f(\vh_i^\prime) \quad \text{where} \quad \vh_i^\prime \eqdef \sum_{j\in N_i^*} \gamma_{ij} \mW_v \vh_j\;, 
    \label{eq:gnn}
\end{equation}
where $N_i^*$ is the set of neighbors of node $i$ (including $i$), 
$\mW_v\in\R^{d^\prime\times d}$ a learnable  matrix, $f$ an elementwise function, and $\gamma_{ij}\in[0,1]$ are coefficients such that $\sum_{j} \gamma_{ij} = 1$ for each node $i$.

Let the input features be $\vh_i^0 = \bX_i$, and $\vh_i^L = \hat \by_i$ the predictions, then we can readily define a message-passing GNN~\citep{gilmer2017neural} as a sequence of $L$ layers as defined above.
Depending on the way the coefficients $\gamma_{ij}$ are computed, we identify different GNN flavors. 

\msg{Graph Convolutional Neural Networks. Explain the initial formulation by Kipf.}

\paragraph{Graph convolutional networks (GCNs)}~\citep{kipf2016semi} 
are simple yet effective. 
In short, GCNs
compute the average of the messages, \ie, they assign the same coefficient $\gamma_{ij} = {1}/{|N_i^*|}$ to every neighbor:
\begin{equation}
    \tilde\vh_i = f(\vh_i^\prime) \quad \text{where} \quad \vh_i^\prime \eqdef \frac{1}{|N_i^*|} \sum_{j \in N_i^*} \mW_v\vh_j\;, 
    \label{eq:gcn}
\end{equation}

\msg{Graph Attention Networks. Explain GAT and GAT2.}

\paragraph{Graph attention networks} 
take a different approach. Instead of assigning a fixed value to each coefficient $\gamma_{ij}$, they  compute it as a function of the sender and receiver nodes.
A general formulation for these models can be written as follows:
\begin{equation} 
	\tilde\vh_i = f(\vh_i^\prime) \quad \text{where} \quad \vh_i^\prime \eqdef \sum_{j \in N_i^*} \gamma_{ij}\mW_v\vh_j\quad\text{and}\quad \gamma_{ij} \eqdef \frac{\exp({\Psi(\vh_i, \vh_j) })}{\sum_{\ell\in N^*_i} \exp({\Psi(\vh_i, \vh_\ell)}) }\;\label{eq:gat}. 
\end{equation}
Here, $\Psi(\vh_i, \vh_j) \eqdef \alpha(\mW_q\vh_i, \mW_k\vh_j)$ is known as the \emph{score function} (or \emph{attention architecture}), and provides a score value between the messages $\vh_i$ and $\vh_j$ (or more generally, between a learnable mapping of the messages).
From these scores, the (attention) coefficients are obtained by normalizing them, such that $\sum_j \gamma_{ij} = 1$.
We can find in the literature different attention layers and, %
throughout this work, we focus on the original GAT~\citep{velivckovic2017graph} and its extension GATv2~\citep{brody2021attentive}:
\begin{align}
    \text{GAT:} & \quad \Psi(\vh_i, \vh_j) = \LeakyRelu\left( \va^\top [\mW_q\vh_i || \mW_k \vh_j] \right)\;,
    \label{eq:score_gatv2} \\
    \text{GATv2:} & \quad \Psi(\vh_i, \vh_j) =  \va^\top \LeakyRelu \left( \mW_q\vh_i +  \mW_k\vh_j \right)\;,
    \label{eq:score_gat}
\end{align}
where the learnable parameters are now the attention vector $\va$; and the matrices $\mW_q$, $\mW_k$, and $\mW_v$. 
Following previous work \citep{velivckovic2017graph,brody2021attentive}, we assume that these matrices are coupled, \ie, $\mW_q = \mW_k = \mW_v$.
Note that the difference between the two layers lies in the position of the vector $\va$: by taking it out of the nonlinearity, \citet{brody2021attentive} increased the expressiveness of GATv2. Now, the product of $\va$ and a weight matrix does not collapse into another vector. 
More importantly, the addition of two different attention layers will help us show the versatility of the proposed models later in \cref{sec:experiments}.

\section{Previous work\pages{0.5}} \label{sec:related-work}

In recent years, there has been a surge of research in GNNs. %
Here, we discuss other GNN models, attention mechanisms, and the recent findings on the limitations of GCNs and GATs.

The literature on GNNs is extensive \citep{wu2020comprehensive, hamilton2017representation,battaglia2018relational,lee2019attention}, and more abstract definitions of a message-passing GNN are possible, leading to other lines of work trying different ways to compute messages, aggregate them, or update the final message~\citep{hamilton2017inductive,xu2018powerful,corso2020principal}\ajnote{double check}.
Alternatively, another line of work fully abandons message-passing, working instead with higher-order interactions~\citep{morris2019weisfeiler}.
While some of this work is orthogonal---or directly applicable---to the proposed model, %
in the main paper we focus on convolutional and attention graph layers, as
they are the most widely used (and cited) as of today. %

\msg{\todopablo Mention all the different variants of GCN/GATs that there are.}

While we consider the original GAT~\citep{velivckovic2017graph} and GATv2~\citep{brody2021attentive}, our work can be directly applied to any attention model that sticks to the formulation in \cref{eq:gat}. 
For example, some works propose different metrics for the score function, like the dot-product~\citep{brody2021attentive}, cosine similarity~\citep{thekumparampil2018attention}, or a combination of various functions~\citep{kim2022find}. 
Other works introduce transformer-based mechanisms~\citep{vaswani2017attention} based on positional encoding~\citep{dwivedi2020generalization,kreuzer2021rethinking} or on the set transformer~\citep{wang2020multi}. 
Finally, there also exist attention approaches designed for specific type of graphs, such as relational~\citep{yun2019graph,busbridge2019relational} or heterogeneous graphs~\citep{wang2019heterogeneous,hu2020heterogeneous}.  

\subsection{On the limitations of GCN and GAT networks\pages{0.75}} \label{subsec:limitations}

\msg{\todokimon \todoamit Introduce the stochastic block model example}

\ajnote{I rewrote this part to match the notation and definitions we've given so far.}

\msg{\todokimon \todoamit Explain how this example has been used before \citep{kimongcn} to show the strengths and weaknesses of GCNs (as a function of $p-q$)}

\citet{kimongcn} studied
classification %
on a simple stochastic block model, %
showing that,
when the graph is neither too sparse nor noisy, applying one layer of graph convolution increases the regime in which the data is linearly separable. 
However, this result is highly sensitive to the graph structure, as convolutions essentially collapse the data to the same value in the presence of enough noise.
\msg{\todokimon \todoamit Explain how this example has been used before \citep{kimongat} to show the strengths and weaknesses of GATs (easy vs. hard regimes)}
More recently, 
\citet{kimongat} showed that GAT is able to remedy the above issue, and provides perfect node separability regardless of the noise level in the graph.
However, a classical argument (see~\citet{anderson1962introduction}) states that
\emph{in this particular setting}
a linear classifier already achieves perfect separability. %
These works, in summary, showed scenarios for which GCNs can be beneficial in the absence of noise, and that GAT can outperform GCNs in other scenarios, leaving open the question of which architecture is preferable in terms of performance.

\section{Convolved attention: benefits and hurdles}
\label{sec:convolved-GAT}
\msg{\todokimon \todoamit Introductory paragraph talking about the interpretation/intuition of the limitations of GAT in the toy example, and how to solve it with convolved score inputs}

\msg{\todoadri Introduce \ours (GAT with convolved score).}

In this section, we propose to combine attention with convolution operations.
To motivate it, we complement the results of~\citet{kimongat}, providing a synthetic dataset for which \emph{any} $1$-layer GCN fails, but $1$-layer GAT does not. 
Thus, proving a clear distinction between GAT and GCN layers.
Besides, we show that convolution helps GAT as long as the graph noise is reasonable.
The proofs for the two statements in this section appear in Appendix~\ref{app:theory} and follow similar arguments as in~\citet{kimongat}.

This dataset is based on the \emph{contextual stochastic block model} ($\CSBM$)~\citep{deshpande2018contextual}.  
Let $\eps_1,\ldots,\eps_n$ be iid. uniform samples from $\{-1,0,1\}$.
Let $C_k=\{j\in [n]\mid \eps_j=k\}$ for $k\in \{-1,0,1\}$. We set the feature vector $\bX_i\sim \Normal(\eps_i\bmu, \bI\cdot \sigma^2)$ where $\bmu\in \R^d$, $\sigma\in \R$, and $\bI\in \zo^{d\times d}$ is the identity matrix. For a given pair $p,q\in [0,1]$ we consider the stochastic adjacency matrix $\bA\in\zo^{n\times n}$ defined as follows: for $i,j\in [n]$ in the same class (\emph{intra-edge}), we set $a_{ij}\sim \Ber(p)$;\footnote{$\Ber(\cdot)$ denote the Bernoulli distribution.} for $i,j$ in different classes (\emph{inter-edge}), we set $a_{ij}\sim\Ber(q)$. We denote by $(\bold{X},\bA)\sim \CSBM(n,p,q,\bmu,\sigma^2)$ a sample obtained according to the above random process. Our task is then to distinguish (or separate) nodes from $C_0$ vs. $C_{-1}\cup C_1$.
Note that, in general, it is impossible to separate $C_0$ from $C_{-1}\cup C_1$ with a linear classifier and, 
using one convolutional layer is detrimental for node classification on the CSBM:\footnote{We note that this problem can be easily solved by two layers of GCN~\citep{baranwal2022effects}.}
 although the convolution brings the means closer and shrinks the variance, the geometric structure of the problem does not change.
On the other hand, we prove that GAT is able to achieve perfect node separability when the graph is not too sparse: 
\begin{theorem} \label{prop:gat} Suppose that $p,q=\Omega(\log^2 n/n)$ and $\|\bmu\|_2=\omega(\sigma \sqrt{\log n})$. Then, there exists a choice of attention architecture $\Psi$ such that, with probability at least $1-o_n(1)$ over the data $(\bold{X},\bA) \sim \CSBM(n,p,q,\bmu,\sigma^2)$, GAT separates nodes $C_0$ from $C_1\cup C_{-1}$.
\end{theorem}

Moreover, we show using methods from~\citet{kimongcn}, that the above classification threshold $\|\bmu\|$ can be improved when the graph noise is reasonable. %
Specifically, \emph{by applying convolution prior to the attention score}, the variance of the data is greatly reduced, and if the graph is not too noisy, the operation dramatically lowers the bound  in \cref{prop:gat}.
We  exploit this insight by introducing the \emph{graph convolutional attention layer} (\ours):
\begin{equation}
	\Psi(\vh_i, \vh_j) =  \alpha(\mW\conv\vh_i, \mW\conv\vh_j) \quad\text{where}\quad \conv\vh_i = \frac{1}{|N^*_i|} \sum_{\ell\in N^*_i} \vh_\ell\;, %
	\label{eq:ours-score}
\end{equation}%
where $\conv\vh_i$ are the convolved features of the neighborhood of node $i$. 
As we show now, \ours improves over GAT by combining convolutions with attention, when the graph noise is low.
\begin{corollary} \label{prop:cat} Suppose $p,q=\Omega(\log^2 n/n)$ and  $\|\bmu\|\ge \omega\left(\sigma \sqrt{\frac{(p+2q)\log n}{n(p-q)^2}}\right)$. Then, there is a choice of attention architecture $\Psi$ such that \ours separates nodes $C_0$ from $C_1\cup C_{-1}$, with probability at least $1-o(1)$ over the data $(\bold{X},\bA)\sim\CSBM(n,p,q,\bmu,\sigma^2)$.
\end{corollary}

\msg{\todoamit \todokimon Here we introduce the new toy example, as well as the theoretical analysis we might have on its benefits and limitations.}

The above proposition shows that under the $\CSBM$ data model, convolving prior to attention changes the regime for perfect node separability by a factor of $|p-q|\sqrt{{n}/{(p+2q)}}$. This is desirable when  $|p-q|\sqrt{{n}/{(p+2q)}}>1$, since the regime for perfect classification is increased. 
Otherwise, applying convolution prior to attention reduces the regime for perfect separability. %
Therefore, it is not always clear whether convolving prior to attention is beneficial.

\section{L-CAT: Learning to interpolate\pages{0.5}}

From the previous analysis, we can conclude that it is hard to know \textit{a priori} whether attention, convolution, or convolved attention, will perform the best. 
In this section, we argue that this issue can be easily overcome by learning to interpolate between the three.

\msg{\todoadri Introduce first $\lambda_1$ using the general formulation.}

\msg{\todoadri Introduce the final model with interpolations by introducing $\lambda2$.}

First, note that GCN and GAT only differ in that GCN weighs all neighbors equally (\cref{eq:gcn}) and, the more similar the attention scores are (\cref{eq:gat}), the more uniform the coefficients~$\gamma_{ij}$ are. 
Thus, we can interpolate between GCN and GAT by introducing a learnable parameter. %
Similarly, the formulation of GAT (\cref{eq:gat}) and \ours (\cref{eq:ours-score}) differ in the convolution within the score, which can be interpolated with another learnable parameter. %

Following this observation, we propose the \emph{learnable convolutional attention layer} (\ours*), which can be formulated as an attention layer with the following score:
\begin{equation}
	\Psi(\vh_i, \vh_j) = \lambda_1 \cdot \alpha(\mW\conv\vh_i, \mW\conv\vh_j) \quad\text{where}\quad \conv\vh_i = \frac{\vh_i + \lambda_2 \sum_{\ell\in N_i} \vh_\ell}{1 + \lambda_2 |N_i|}\;, 
    \label{eq:ours-score-lmbda}
\end{equation}
where $\lambda_1, \lambda_2 \in [0, 1]$. %
As mentioned before, 
this formulation lets
\ours* learn to interpolate between GCN ($\lambda_1 = 0$), GAT ($\lambda_1 = 1$ and $\lambda_2 = 0$), and \ours ($\lambda_1 = 1$ and $\lambda_2 = 1$). 

\ours* enables a number of non-trivial benefits.
Not only can it switch between existing layers, but it also learns the amount of attention necessary for each use-case.
Moreover, by comprising the three layers in a single learnable formulation, it removes the necessity of cross-validating the type of layer, as their performance is data-dependent (see \cref{subsec:limitations,sec:convolved-GAT}).
Remarkably, it allows to easily combine different layer types within the same architecture.

While we focus on GCN, \ours* can be easily used with other GNN architectures such as PNA~\citep{corso2020principal} and GCNII~\citep{chen2020simple}, as \ours* interpolates between two different adjacency matrices. For further details and results, refer to~\cref{app:extra-results-baselines}.

\msg{\todoadri Talk about the interpolation}

\vspace{-0.2cm}
\section{Experiments}\label{sec:experiments}

In this section, {we first} validate our theoretical findings on synthetic data (\cref{sec:exps-synhetic-data}).
Then, we show through various node classification tasks that \bothours is as competitive as the baseline models~(\cref{sec:experiments-graphgym}). %
Lastly, we move to more demanding scenarios from the Open Graph Benchmark~\citep{hu2020open}, \replaced{3}{assessing their performance and robustness to feature and edge noise, as well as network initialization}{demonstrating that \ours* is a more flexible and robust alternative to its baseline methods} (\cref{sec:ogb-experiments}) that reduces the need of cross-validating without giving up on performance. \added{3}{Refer to \cref{app:toy_v2,app:dataset,app:extra-results-node,app:ogb,app:extra-results-baselines} for details and additional results.} The code to reproduce the experiments can be found at \url{https://github.com/psanch21/LCAT}.

\subsection{Synthetic data\pages{1}} \label{sec:exps-synhetic-data}
\begin{figure}[!t]
\vspace{-0.6cm}
	\centering
	\begin{subfigure}{.33\textwidth} %
		\centering
		\includegraphics[width=.99\linewidth]{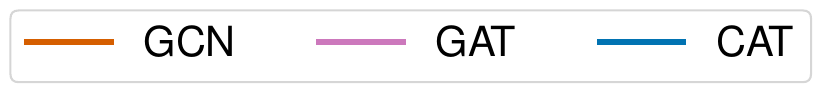}
	\end{subfigure}
	\begin{subfigure}{.24\textwidth}  %
		\centering
		\includegraphics[width=.99\linewidth]{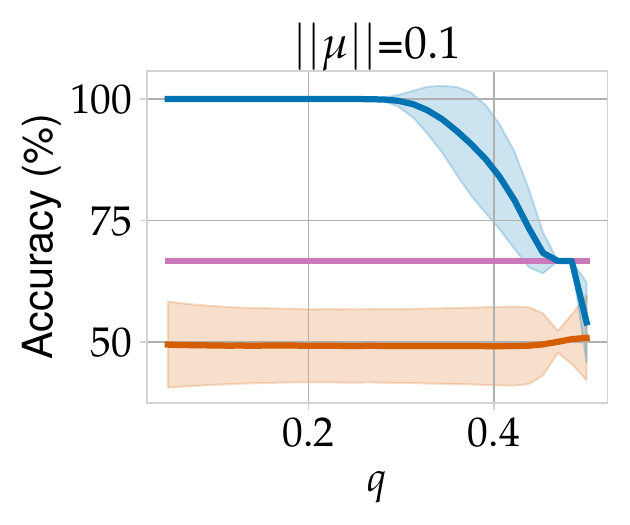}
	\end{subfigure}%
	\begin{subfigure}{.24\textwidth}  %
		\centering
		\includegraphics[width=.99\linewidth]{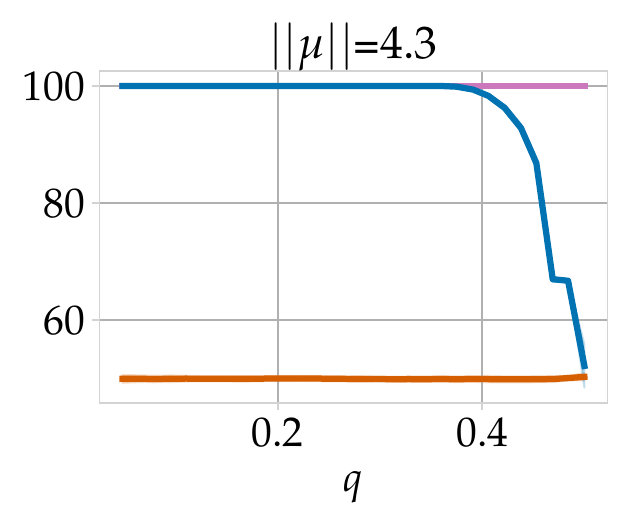}
	\end{subfigure}%
	\begin{subfigure}{.24\textwidth}  %
		\centering
		\includegraphics[width=.99\linewidth]{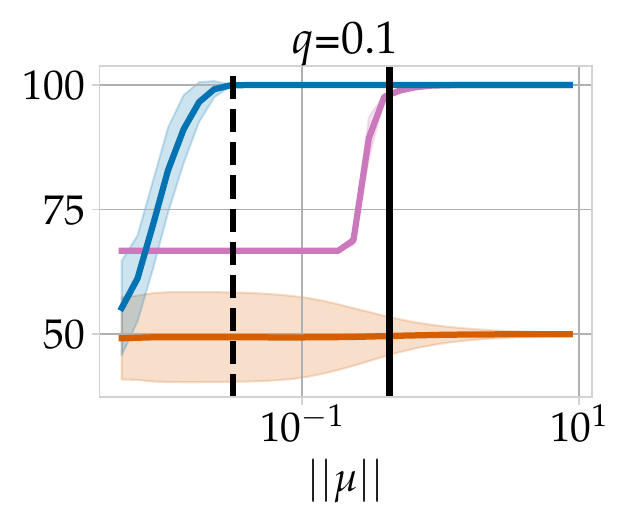}
	\end{subfigure}%
	\begin{subfigure}{.24\textwidth}  %
		\centering
		\includegraphics[width=.99\linewidth]{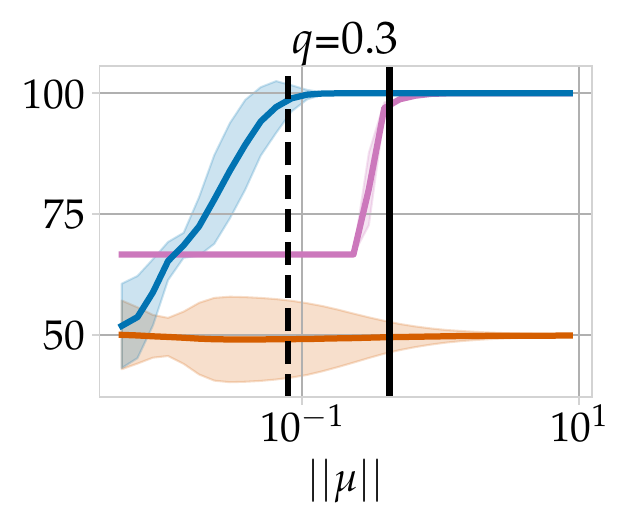}
	\end{subfigure}%
\vspace{-0.3cm}
	\caption{Synthetic data results. The left-most plots show accuracy as we vary the noise level $q$ for $\|\bmu\|=0.1$ and  $\|\bmu\|=4.3$. The right-most plots show the accuracy as we change the norm of the means $\|\bmu\|$ for $q=0.1$ and $q=0.3$. We use two vertical lines to present the classification threshold stated in \cref{prop:gat} (solid line) and \cref{prop:cat} (dashed line).}
	\label{fig:toy_v2_ansatz}
 \vspace{-0.2cm}
\end{figure}
\raggedbottom

\msg{\todopablo \todoamit \todokimon Experiments on the stochastic block model experiments, where we reproduce the notebooks from the previous papers, trying to understand the behaviour of the model.}

First, we empirically validate our theoretical results (\cref{prop:gat} and \cref{prop:cat}). 
We aim to better understand the behavior of each layer as the properties of the data change, \ie, the noise level $q$ (proportion of inter-edges) and the distance between the means of consecutive classes $\|\bmu\|$.  
We provide extra results and additional experiments in \cref{app:toy_v2}.

\paragraph{Experimental setup.} As data model, we use the proposed $\CSBM$ (see \cref{sec:convolved-GAT}) with $n=10000$, $p=0.5$, $\sigma=0.1$, and $d=n/\left(5\log^2(n)\right)$.
All results are averaged over $50$ runs, %
and parameters are set as described in \cref{app:theory}.
We conduct two experiments to assess the sensitivity to structural noise.
First, we vary the noise level $q$ between \num{0} and \num{0.5}, leaving the mean vector $\bmu$ fixed. 
We test two values of~$\|\bmu\|$: 
the first corresponds to the \textit{easy} regime ($\|\bmu\| = 10\sigma\sqrt{2\log n}$) where classes are far apart;
and the second correspond to the \textit{hard} regime ($\|\bmu\| = \sigma$) where %
clusters are close. 
In the second experiment we instead sweep $\|\bmu\|$ in the range $\left[\sigma/20, 20\sigma\sqrt{2\log n}\right]$, covering the transition from hard (small $\|\bmu\|$) to easy (large $\|\bmu\|$) settings. Here, we fix $q$ to \num{0.1} (low noise) and \num{0.3} (high noise).
In both cases, we compare the behavior of 1-layer GAT and \ours, and include GCN as the baseline.

\paragraph{Results.} 
The two left-most plots of \cref{fig:toy_v2_ansatz} show \replaced1{the }node classification performance for the hard and easy regimes, respectively, as we vary the noise level $q$. 
In the hard regime, we observe that GAT is unable to achieve separation for any value of $q$, whereas \ours achieves perfect classification when $q$ is small enough.
This exemplifies the advantage of \ours over GAT as stated in \cref{prop:cat}. When the distance between the means is large enough, we see that GAT achieves perfect results independently of $q$, as stated in \cref{prop:gat}. 
In contrast, when \ours fails to satisfy the condition in \cref{prop:cat} (as we increase $q$), it achieves inferior performance.

The right-most part of \cref{fig:toy_v2_ansatz} shows the results when we fix $q$ and sweep $\|\bmu\|$.
In these two plots, we can appreciate the transition in the accuracy of both GAT and \ours as a function of $\|\bmu\|$. We observe that GAT achieves perfect accuracy when the distance between the means satisfies the condition in \cref{prop:gat} (solid vertical line in \cref{fig:toy_v2_ansatz}). Moreover, we can see the improvement \ours obtains over GAT. Indeed, when $\|\bmu\|$ satisfies the conditions of \cref{prop:cat} (dashed vertical line in \cref{fig:toy_v2_ansatz}), the classification threshold is improved. As we increase $q$ we see that the gap between the two vertical lines decreases, which means that the improvement of \ours over GAT decreases as $q$ increments, exactly as stated in \cref{prop:cat}.

\msg{\todopablo \todoamit \todokimon Experiments similar as the previous one but with the new example, where our model is the only one capable of performing well.}

\subsection{Real data\pages{1}} \label{sec:experiments-graphgym}

\msg{Intro paragraph}
We study now the performance of the proposed models in a comprehensive set of real-world experiments, in order to gain further insights of the settings in which they excel.
Specifically, we found \ours and \ours* to outperform their baselines as the average node degree increases.
For a detailed description of the datasets and additional results, refer to \cref{app:dataset,app:extra-results-node,app:extra-results-baselines}.

\msg{Dataset and setup. Emphasize freezing all layers}

\paragraph{Models.} We consider as baselines a simple GCN layer~\citep{kipf2016semi}, the original GAT layer~\citep{velivckovic2017graph} and its recent extension, GATv2~\citep{brody2021attentive}. 
Based on the two attention models, we consider their \ours and \ours* extensions. 
To ensure fair comparisons, all layers use the same number of parameters and implementation.

\paragraph{Datasets.} We consider six node classification datasets. %
The \spname{Facebook}\-/\spname{GitHub}\-/\spname{TwitchEN} datasets involve social-network graphs~\citep{rozemberczki2021multi}, whose nodes represent verified pages/developers/streamers; and where the task is to predict the topic/expertise/explicit-language-use of the node.  
The \spname{Coauthor Physics} dataset~\citep{shchur2018pitfalls} represents a co-authorship network whose nodes represent authors, and the task is to infer their main research field. 
The \spname{Amazon} datasets represent two product-similarity graphs~\citep{shchur2018pitfalls}, where each node is a product, and the task is to infer its category.

\paragraph{Experimental setup.} 
To ensure the best results, we cross-validate all optimization-related hyperparameters for each model using GraphGym~\citep{you2020design}.
All models use four GNN layers with hidden size of~\num{32}, and thus have an equal number of parameters.
For evaluation, we take the best-validation configuration during training, and report test-set performance.
For further details, refer to \cref{app:extra-results-node}.

\begin{table*}[!t]
    \vspace{-0.5cm}
	\centering
	\caption{Test accuracy (\%) of the considered models for different datasets (sorted by their average node degree), and averaged over ten runs. Bold numbers are statistically different to their baseline model ($\alpha = 0.05$). Best average performance is underlined.}
	\label{tab:performance_node_small}
	\vspace{-0.1cm}
	\resizebox{\textwidth}{!}{
		\begin{tabular}{lcccccc}
			\toprule
			Dataset & \thead{\spname{Amazon}\\\spname{Computers}} & \thead{\spname{Amazon}\\\spname{Photo}} & \spname{GitHub} & \thead{\spname{Facebook}\\\spname{PagePage}} & \thead{\spname{Coauthor}\\\spname{Physics}} & \spname{TwitchEN} \\
			Avg. Deg. & 35.76 & 31.13 & 15.33 & 15.22 & 14.38 & 10.91 \\ \midrule
			GCN & \best{90.59 $\pm$ 0.36} & \best{95.13 $\pm$ 0.57} & 84.13 $\pm$ 0.44 & 94.76 $\pm$ 0.19 & 96.36 $\pm$ 0.10 & 57.83 $\pm$ 1.13 \\
            \midrule
            GAT & 89.59 $\pm$ 0.61 & 94.02 $\pm$ 0.66 & 83.31 $\pm$ 0.18 & 94.16 $\pm$ 0.48 & 96.36 $\pm$ 0.10 & 57.59 $\pm$ 1.20 \\
            \ours & \better{90.58 $\pm$ 0.40} & \better{94.77 $\pm$ 0.47} & \better{84.11 $\pm$ 0.66} & \better{94.71 $\pm$ 0.30} & \best{96.40 $\pm$ 0.10} & \best{58.09 $\pm$ 1.61} \\
            \ours* & \better{90.34 $\pm$ 0.47} & \better{94.93 $\pm$ 0.37} & \better{84.05 $\pm$ 0.70} & \best{\better{94.81 $\pm$ 0.25}} & 96.35 $\pm$ 0.10 & 57.88 $\pm$ 2.07 \\
            \midrule
            GATv2 & 89.49 $\pm$ 0.53 & 93.47 $\pm$ 0.62 & 82.92 $\pm$ 0.45 & 93.44 $\pm$ 0.30 & 96.24 $\pm$ 0.19 & 57.70 $\pm$ 1.17 \\
            \ours[v2] & \better{90.44 $\pm$ 0.46} & \better{94.81 $\pm$ 0.55} & \better{84.10 $\pm$ 0.88} & \better{94.27 $\pm$ 0.31} & 96.34 $\pm$ 0.12 & 57.99 $\pm$ 2.02 \\
            \ours*[v2] & \better{90.33 $\pm$ 0.44} & \better{94.79 $\pm$ 0.61} & \best{\better{84.31 $\pm$ 0.59}} & \better{94.44 $\pm$ 0.39} & 96.29 $\pm$ 0.13 & 57.89 $\pm$ 1.53 \\
            \bottomrule
		\end{tabular}
	}
\end{table*}
\raggedbottom

\paragraph{Results}
 are presented in \cref{tab:performance_node_small}. %
In contrast with \cref{sec:exps-synhetic-data}, we here find GCN to be a strong contender, reinforcing its viability in real-world data despite its simplicity.
We observe both \ours and \ours* not only holding up the performance with respect to their baselines models for all datasets, but in most cases also improving the test accuracy in a statistically significant manner.
These results validate the effectiveness of \ours as a GNN layer, and show the viability of \emph{\ours* as a drop-in replacement}, achieving good results on all datasets.

\msg{Analysis: correlation with average degree}

\begin{wrapfigure}[9]{r}{.3\textwidth}
	\centering
	\vspace{-1.5em}
	\includegraphics[width=\linewidth, keepaspectratio]{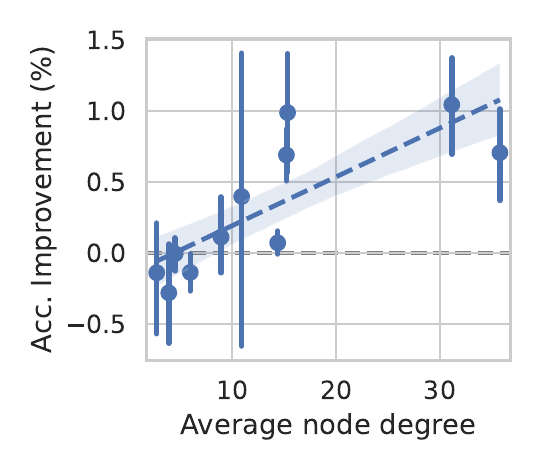}
\end{wrapfigure}
As explained in \cref{sec:convolved-GAT}, \ours differs from a usual GAT in that the score is computed with respect to the convolved features.
Intuitively, this means that \ours should excel in those settings where nodes are better connected, allowing \ours to extract more information from their neighborhoods. 
Indeed, in the inset figure we can observe the improvement in accuracy of \ours with respect to its baseline model, as a function of the average node degree of the dataset, and the linear regression fit of these results (dashed line).
This plot includes all datasets (from the manuscript and \cref{app:extra-results-node}), and shows a positive trend between node connectivity and improved performance achieved by \ours.

\subsection{Open Graph Benchmark\pages{1}} \label{sec:ogb-experiments}

\msg{Intro, why}

In this section, we assess the robustness of the proposed models, %
in order to fully understand their benefits.
For further details and additional results, refer to \cref{app:ogb}.

\msg{Dataset and experimental setup}

\paragraph{Datasets.} We consider four datasets from the OGB suite~\citep{hu2020open}: \spname{proteins}, \spname{products}, \spname{arxiv}, and \spname{mag}. Note that these datasets are significantly larger than those from \cref{sec:experiments-graphgym} and correspond to more difficult tasks, \eg,  \spname{arxiv} is a 40-class classification problem  (see \cref{tab:datasets-statistics} in \cref{app:dataset} for details). This makes them more suitable for the proposed analysis. 

\paragraph{Experimental setup.} We adopt the same experimental setup as~\citet{brody2021attentive} for the \spname{proteins}, \spname{products}, and \spname{mag} datasets. For the \spname{arxiv} dataset, we use instead the example code from OGB~\citep{hu2020open}, as it yields better performance than that of \citet{brody2021attentive}.
Just as in \cref{sec:experiments-graphgym}, we compare with GCN~\citep{kipf2016semi}, GAT~\citep{velivckovic2017graph}, GATv2~\citep{brody2021attentive}, and their \ours and \ours* counterparts.
We cross-validate the number of heads (\num{1} and \num{8}), %
and select the best-validation models during training. 
All models are identical except for their $\lambda$ values.

\msg{Table 2. Overall node accuracy results.}

\begin{table*}[!t]
	\centering
	\caption{Test performance of the considered models on four OGB datasets, averaged over five runs. Bold numbers are statistically different to their baseline model ($\alpha = 0.05$). Best average performance is underlined. Left table: accuracy (\%); right table: AUC-ROC (\%).}
	\label{tab:performance_ogb}
	
	\begin{subtable}{.7\linewidth}
		\centering
		{
			\begin{tabular}{lccc}
				\toprule
				Dataset & \spname{arxiv} & \spname{products} & \spname{mag} \\
				\midrule
				GCN & 71.58 $\pm$ 0.20 & 74.12 $\pm$ 1.20 & \best{32.77 $\pm$ 0.36} \\
				\midrule
				GAT & 71.58 $\pm$ 0.16 & \best{78.53 $\pm$ 0.91} & 32.15 $\pm$ 0.31 \\
				\ours & \best{\better{72.14 $\pm$ 0.21}} & \worse{77.38 $\pm$ 0.36} & 31.98 $\pm$ 0.46 \\
				\ours* & \better{71.99 $\pm$ 0.08} & 77.19 $\pm$ 1.11 & 32.47 $\pm$ 0.38 \\
				\midrule
				GATv2 & 71.73 $\pm$ 0.24 & 76.40 $\pm$ 0.71 & 32.76 $\pm$ 0.18 \\
				\ours[v2] & \better{72.03 $\pm$ 0.09} & 74.81 $\pm$ 1.12 & \worse{32.43 $\pm$ 0.22}  \\
				\ours*[v2] & 71.97 $\pm$ 0.22 & \worse{76.37 $\pm$ 0.92} & 32.68 $\pm$ 0.50 \\
				\bottomrule
			\end{tabular}
		}
	\end{subtable}%
	\begin{subtable}{.2\linewidth}
		\centering
		{
			\begin{tabular}{c}
				\toprule
				\spname{proteins} \\
				\midrule
				\best{80.10 $\pm$ 0.55} \\
				\midrule
				79.08 $\pm$ 1.47 \\
				73.26 $\pm$ 1.65 \\
				79.63 $\pm$ 0.71 \\
				\midrule
				78.65 $\pm$ 1.44 \\
				74.33 $\pm$ 0.94 \\
				\better{79.07 $\pm$ 0.98} \\
				\bottomrule
			\end{tabular}
		}
	\end{subtable}
\end{table*}

\paragraph{Results} are summarized in \cref{tab:performance_ogb}. %
Here we do not observe a clear preferred baseline: %
GCN performs really well in \spname{proteins} and \spname{mag}; GAT excels in \spname{products}; and GATv2 does well in \spname{arxiv} and \spname{mag}.
While \ours obtains the best results on \spname{arxiv}, its performance on \spname{proteins} and \spname{products} is significantly worse than the baseline model. Presumably, an excessive amount of inter-edges could explain why convolving the features prior to computing the score is harmful, as seen in \cref{sec:exps-synhetic-data}. 
As we explore in \cref{sec:exp-robustness-init}, however, \ours improves over its baseline for most \spname{proteins} scenarios, specially with a single head.
In stark contrast, \ours* performs remarkably well, improving the baseline models in all datasets but \spname{products}---even on those in which \ours fails---demonstrating the adaptability of \ours* to different scenarios.

\msg{\todoadri Plot the evolution of the lambdas during training.}

To better understand the training dynamics of the models, we plot in \cref{fig:arxiv-test} the test accuracy of GCN and the GATv2 models during training on the \spname{arxiv} dataset. 
Interestingly, despite all models obtaining similar final results, \emph{\ours[v2] and \ours*[v2] drastically improved their convergence speed and stability with respect to GATv2}, matching that of GCN. 
To understand the behavior of \ours*[v2], \cref{fig:arxiv-lambdas} shows the evolution of the $\lambda$ parameters.
We observe that, to achieve these results, \ours*[v2] converged to a GNN network that combines three types of layers: the first layer is a \ours[v2] layer, taking advantage of the neighboring information; the second layer is a quasi-GCN layer, in which scores are almost uniform and some neighboring information is still used in the score computation; and the third layer is a pure GCN layer, in which all scores are uniformly distributed.
It is important to remark that these dynamics are fairly consistent, as \ours*[v2] reached the same $\lambda$ values over all five runs.

\begin{figure}
	\centering
	\begin{subfigure}{.4\linewidth}
		\centering
		\includegraphics[width=\textwidth, keepaspectratio]{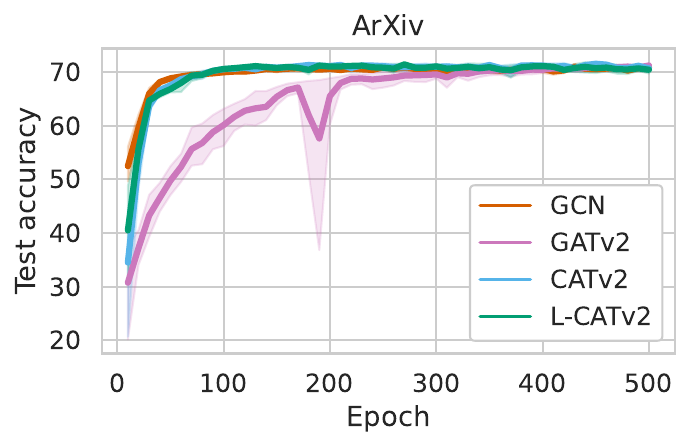}
		\caption{Test accuracy.} \label{fig:arxiv-test}
	\end{subfigure} \hspace{.05\linewidth}
	\begin{subfigure}{.4\linewidth}
		\centering
		\includegraphics[width=\textwidth, keepaspectratio]{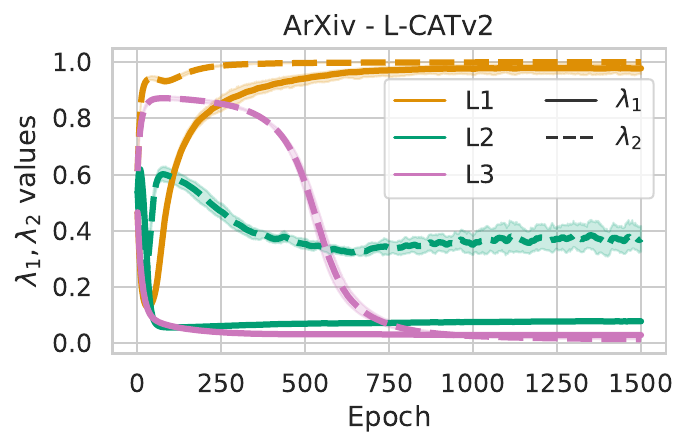}
		\caption{Evolution of $\lambda_1,\lambda_2$.} \label{fig:arxiv-lambdas}
	\end{subfigure}
	\caption{Behavior of GCN and GATv2-based models during training on the \spname{arxiv} dataset. \captiona~\ours and \ours* converge quicker and more stably than their baseline model. 
	\captionb~\ours* consistently converges to the same architecture: a \ours$\rightarrow$quasi-GCN$\rightarrow$GCN network.}
	\label{fig:arxiv-training}
\end{figure}
\raggedbottom

\subsubsection{Robustness to noise} \label{subsec:exp-noise}

\msg{Intro}

One intrinsic aspect of real world data is the existence of noise. In this section, we explore the robustness of the proposed models to different levels of noise,
\ie, we attempt to simulate scenarios where there exist measurement inaccuracies in the input features and edges. %

\msg{Setup}

\begin{wrapfigure}[20]{r}{.33\linewidth}
	\centering
	\vspace{-\baselineskip}
	\includegraphics[width=\linewidth, keepaspectratio]{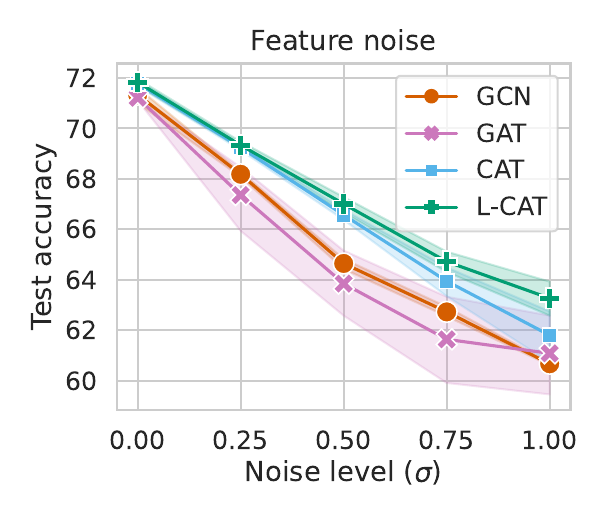} %
	\includegraphics[width=\linewidth, keepaspectratio]{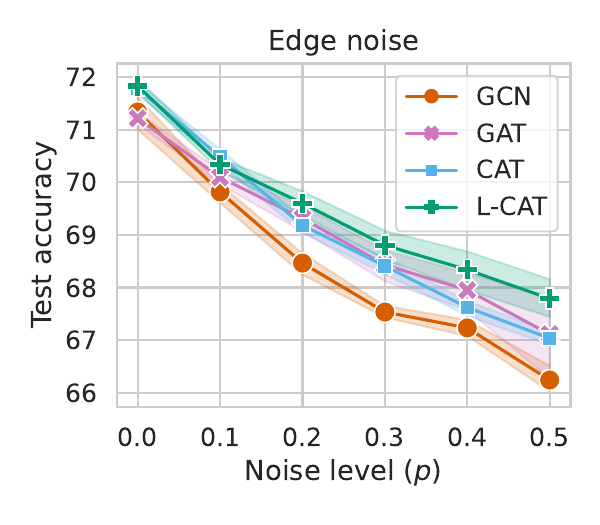}
\end{wrapfigure}
\paragraph{Experimental setup.} 
We consider the \spname{arxiv} dataset, and the same experimental setup as in \cref{sec:ogb-experiments}. 
We conduct two experiments. 
First, we introduce to the node features additive noise of the form $\Normal(\bm{0}, \bm{1}\sigma)$, and consider different levels of noise, $\sigma\in\{0, 0.25, 0.5, 0.75, 1\}$.
Then, as in \citep{brody2021attentive}, we simulate edge noise by adding fake edges with probability $Bern(p)$ for $p\in\{0, 0.1, 0.2, 0.3, 0.4, 0.5\}$.

\msg{\todoadri Show results on the robustness to homoscedastic noise on arxiv.}

\paragraph{Results} 
are shown in the inset figures for feature (top) and edge noise (bottom),
summarizing the performance of all models over five runs and two numbers of heads (\num{1} and \num{8}). 
Baseline attention models are quite sensitive to feature noise, but are more robust to edge noise, as they can drop inter-class edges (see \cref{subsec:limitations}).
GCNs, as expected, are instead more robust to feature noise, but suffer more in the presence of edge noise.
In concordance with the synthetic experiments (see \cref{sec:convolved-GAT,sec:exps-synhetic-data}), \ours is able to leverage convolutions as a variance-reduction technique, reducing the variance and improving its robustness to feature noise.
Remarkably, \ours* proves to be the most robust for both types of noise: by adapting the amount of attention used in each layer, it outperforms existing methods and reduces the variance.

\subsubsection{Robustness to network initialization} \label{sec:exp-robustness-init}

\msg{Intro}

Another important aspect for real-world applications is robustness to network initialization, \ie, the ability to obtain satisfying performance independently of the initial parameters. 
Otherwise, a practitioner can waste lots of resources trying initilizations or, even worse, give up on a model just because they did not try the initial parameters that yield great results.

\msg{Experimental setup}

\paragraph{Experimental setup.} We follow once again the same setup for \spname{proteins} as in \cref{sec:ogb-experiments}. We consider two different network initializations. 
The first one, \spname{uniform}, uses uniform Glorot initilization~\citep{pmlr-v9-glorot10a} with a gain of \num{1}, which is the standard initialization used throughout this work.
The second one, \spname{normal}, uses instead normal Glorot initialization~\citep{pmlr-v9-glorot10a} with a gain of $\sqrt{2}$. This is the initialization employed on the original GATv2 paper~\citep{brody2021attentive} exclusively for the \spname{proteins} dataset.

\begin{table*}[!t]
\centering
\caption{Test AUC-ROC (\%) on the \spname{proteins} dataset for attention models with two different network initializations (see \cref{sec:exp-robustness-init}), using \num{1} head (top) and \num{8} heads (bottom).}
\label{tab:results-proteins-init}

\resizebox{\textwidth}{!}
{
\begin{tabular}{llcccccc}
\toprule
& & GAT &  \ours & \ours* & GATv2 & \ours[v2] & \ours*[v2] \\
\midrule
\multirow{2}{*}{\rotatebox[origin=c]{90}{1h}} & \spname{uniform} & 59.73 $\pm$ 3.61 & \better{64.32 $\pm$ 2.33} & \better{77.77 $\pm$ 1.28} & 59.85 $\pm$ 2.73 & \better{64.32 $\pm$ 2.33} & \best{\better{79.08 $\pm$ 0.95}} \\
& \spname{normal} & 66.38 $\pm$ 6.94 & 73.26 $\pm$ 1.65 & \better{78.06 $\pm$ 1.25} & 69.13 $\pm$ 8.48 & 74.33 $\pm$ 0.94 & \best{\better{79.07 $\pm$ 0.98}} \\ \midrule
\multirow{2}{*}{\rotatebox[origin=c]{90}{8h}} & \spname{uniform} & 72.23 $\pm$ 2.86 & 73.60 $\pm$ 1.14 & \best{\better{78.85 $\pm$ 1.57}} & 75.21 $\pm$ 1.61 & 74.16 $\pm$ 1.30 & \better{78.77 $\pm$ 0.97} \\
& \spname{normal}  & 79.08 $\pm$ 1.47 & \worse{74.67 $\pm$ 1.15} & \best{79.63 $\pm$ 0.71} & 78.65 $\pm$ 1.44  & \worse{73.40 $\pm$ 0.56} & 79.30 $\pm$ 0.49 \\ \midrule \midrule
\multicolumn{2}{r}{average} & 69.36 $\pm$ 8.52 & 73.93 $\pm$ 1.35 & 78.58 $\pm$ 1.48 & 70.71 $\pm$ 8.70 & 71.55 $\pm$ 4.54 & \best{79.05 $\pm$ 0.91} \\
\bottomrule
\end{tabular}
}
\end{table*}
\raggedbottom

\msg{\todoadri Show performance results of the models for two different parametrizations.}

\paragraph{Results}---segregated by number of heads---are shown in \cref{tab:results-proteins-init}, while the results for GCN appear in the inset table. %
These results show that the baseline models perform poorly on the \spname{uniform} initialization. %
However, this is somewhat alleviated when using \num{8} heads in the attention models. %
Moreover, all baselines significantly improve with \spname{normal} initialization, being GCN the best model, and attention models obtaining \SI{79}{\percent} accuracy on average with \num{8}~heads.
\begin{wraptable}[6]{r}{.26\linewidth}
	\vspace{-\baselineskip}
	\centering
	\resizebox{\linewidth}{!}
	{
		\begin{tabular}{lc}
			\toprule
			& GCN \\
			\midrule
			\spname{uniform} & 61.08 $\pm$ 2.56 \\
			\spname{normal} & 80.10 $\pm$ 0.55 \\ \midrule \midrule
			average & 70.59 $\pm$ 10.21 \\
			\bottomrule
		\end{tabular}
	}
\end{wraptable}
Compared to the baselines, \ours does a good job and improves the performance in all cases except for \spname{normal} with \num{8} heads.
Remarkably, \ours* consistently obtains high accuracy in all scenarios and runs.
To emphasize consistency, bottom row shows the average accuracy across runs, 
showing that \ours* is clearly more robust to parameter initialization than competing models.

\begin{figure}
    \centering
    \includegraphics[width=\linewidth, keepaspectratio]{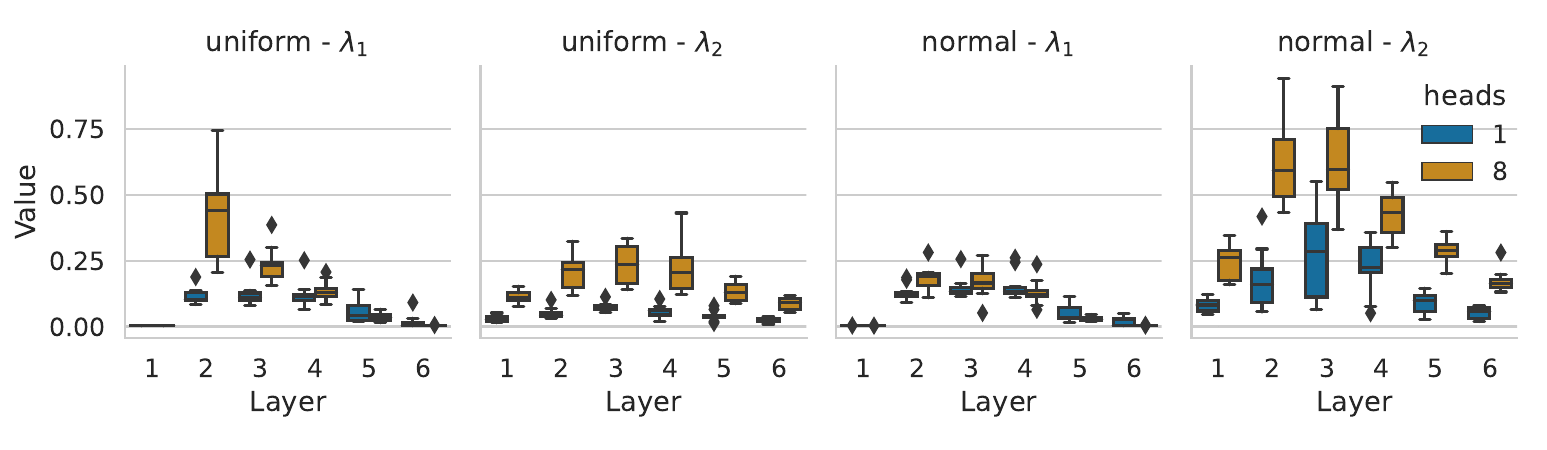}
    \caption{Distribution of $\lambda_1,\lambda_2$ on \spname{proteins} dataset for \ours* across initializations.}
    \label{fig:lambdas-proteins}
\end{figure}
\raggedbottom

To understand this performance, we inspect the distribution of $\lambda_1,\lambda_2$ for \ours* in \cref{fig:lambdas-proteins}. 
Here, we can spot a few interesting patterns. 
Consistently, the first and last layers are always GCNs, while the inner layers progressively admit less attention.
Second, the number of heads affects the amount of attention allowed in the network; the more heads, the more expressive the layer tends to be, and  more attention  is permitted. 
Third, \ours* adapts to the initialization used: in \spname{uniform}, it allows more attention in the second layer; in \spname{normal}, it allows more attention in the score inputs.
These results consolidate the flexibility of \ours*.

\section{Conclusions and future work\pages{0.25}} \label{sec:conclusions}

\msg{Talk about conclusions, what we've shown}

In this work, we studied how to combine the strengths of convolution and attention layers in GNNs.
We proposed \ours, which computes attention with respect to the convolved features, and analyzed its benefits and limitations on a new synthetic dataset. 
This analysis revealed different regimes where one model is preferred over the others, reinforcing the idea that selecting between GCNs, GATs, and now CATs, is a difficult task. %
For this reason, we proposed \ours*, a model which interpolates between the three via two learnable parameters.
Extensive experimental results demonstrated the effectiveness of \ours*, yielding great results while being more robust than other methods. 
As a result, \ours* proved to be a viable drop-in replacement that removes the need to cross-validate the layer type.

\msg{Societal impact}

\msg{Talk about possible future work.}

We strongly believe learnable interpolation can get us a long way, and we hope \ours* to motivate new and exciting work.
Specially, we are eager to see \ours* in real applications, and thus finding out what combining different GNN layers across a model (without the hurdle of cross-validating all layer combinations) can lead to in the real-world.

\section*{Ethic statement}

Given the nature of this work, we do not see any direct ethical concerns.
On the contrary, \ours* eases the application of GNNs to the practitioner, and removes the need of cross-validating the layer type, which can potentially benefit other areas and applications, as GNNs have already proven.

\section*{Reproducibility statement}

For the theoretical results, we describe the data model used in \cref{sec:convolved-GAT}, and provide all the detailed proofs in \cref{app:theory}.
For the experimental results, we include in the supplementary material the necessary code and scripts required to reproduce our experiments, and all required datasets are freely available. Complete details about the experimental setup can be found in \cref{app:toy_v2,app:extra-results-node,app:ogb}, and we report in our results the mean and standard deviation computed using five trials or more. In addition, we highlight in bold the results that are statistically significant. Moreover, we include details of computational resources used for the all sets of experiments in \cref{app:toy_v2}, \cref{app:extra-results-node} and \cref{app:ogb}.

\section*{Acknowledgements}

We would like to thank Batuhan Koyuncu, Jonas Klesen, Miriam Rateike, and Maryam Meghdadi for helpful
feedback and discussions.
Pablo S\'anchez Mart\'in thanks the German Research Foundation through the Cluster of Excellence “Machine Learning – New Perspectives for Science”, EXC 2064/1, project number 390727645 for generous funding support. 
The authors thank the International Max Planck Research School for Intelligent Systems (IMPRS-IS) for supporting Pablo S\'anchez Mart\'in.

	\bibliographystyle{iclr2023_conference}
	\bibliography{references}

\begin{thebibliography}{50}
\providecommand{\natexlab}[1]{#1}
\providecommand{\url}[1]{\texttt{#1}}
\expandafter\ifx\csname urlstyle\endcsname\relax
  \providecommand{\doi}[1]{doi: #1}\else
  \providecommand{\doi}{doi: \begingroup \urlstyle{rm}\Url}\fi

\bibitem[Anderson(2003)]{anderson1962introduction}
T.W. Anderson.
\newblock \emph{An introduction to multivariate statistical analysis}.
\newblock John Wiley \& Sons, 2003.

\bibitem[Ba et~al.(2016)Ba, Kiros, and Hinton]{ba2016layernorm}
Jimmy~Lei Ba, Jamie~Ryan Kiros, and Geoffrey~E Hinton.
\newblock Layer normalization.
\newblock \emph{arXiv preprint arXiv:1607.06450}, 2016.
\newblock URL \url{https://arxiv.org/abs/1607.06450}.

\bibitem[Baranwal et~al.(2021)Baranwal, Fountoulakis, and Jagannath]{kimongcn}
Aseem Baranwal, Kimon Fountoulakis, and Aukosh Jagannath.
\newblock Graph convolution for semi-supervised classification: Improved linear
  separability and out-of-distribution generalization.
\newblock In \emph{International Conference on Machine Learning (ICML)}. PMLR,
  2021.

\bibitem[Baranwal et~al.(2022)Baranwal, Fountoulakis, and
  Jagannath]{baranwal2022effects}
Aseem Baranwal, Kimon Fountoulakis, and Aukosh Jagannath.
\newblock Effects of graph convolutions in deep networks.
\newblock \emph{arXiv preprint arXiv:2204.09297}, 2022.
\newblock URL \url{https://arxiv.org/abs/2204.09297}.

\bibitem[Battaglia et~al.(2018)Battaglia, Hamrick, Bapst, Sanchez-Gonzalez,
  Zambaldi, Malinowski, Tacchetti, Raposo, Santoro, Faulkner,
  et~al.]{battaglia2018relational}
Peter~W Battaglia, Jessica~B Hamrick, Victor Bapst, Alvaro Sanchez-Gonzalez,
  Vinicius Zambaldi, Mateusz Malinowski, Andrea Tacchetti, David Raposo, Adam
  Santoro, Ryan Faulkner, et~al.
\newblock Relational inductive biases, deep learning, and graph networks.
\newblock \emph{arXiv preprint arXiv:1806.01261}, 2018.
\newblock URL \url{https://arxiv.org/abs/1806.01261}.

\bibitem[Bhatia et~al.(2016)Bhatia, Dahiya, Jain, Kar, Mittal, Prabhu, and
  Varma]{Bhatia16}
K.~Bhatia, K.~Dahiya, H.~Jain, P.~Kar, A.~Mittal, Y.~Prabhu, and M.~Varma.
\newblock The extreme classification repository: Multi-label datasets and code,
  2016.
\newblock URL \url{http://manikvarma.org/downloads/XC/XMLRepository.html}.

\bibitem[Bojchevski \& G{\"{u}}nnemann(2018)Bojchevski and
  G{\"{u}}nnemann]{bojchevski2017deep}
Aleksandar Bojchevski and Stephan G{\"{u}}nnemann.
\newblock Deep gaussian embedding of graphs: Unsupervised inductive learning
  via ranking.
\newblock In \emph{6th International Conference on Learning Representations,
  {ICLR} 2018, Vancouver, BC, Canada, April 30 - May 3, 2018, Conference Track
  Proceedings}. OpenReview.net, 2018.
\newblock URL \url{https://openreview.net/forum?id=r1ZdKJ-0W}.

\bibitem[Brody et~al.(2022)Brody, Alon, and Yahav]{brody2021attentive}
Shaked Brody, Uri Alon, and Eran Yahav.
\newblock How attentive are graph attention networks?
\newblock In \emph{International Conference on Learning Representations
  (ICLR)}, 2022.

\bibitem[Busbridge et~al.(2019)Busbridge, Sherburn, Cavallo, and
  Hammerla]{busbridge2019relational}
Dan Busbridge, Dane Sherburn, Pietro Cavallo, and Nils~Y Hammerla.
\newblock Relational graph attention networks.
\newblock \emph{arXiv preprint arXiv:1904.05811}, 2019.
\newblock URL \url{https://arxiv.org/abs/1904.05811}.

\bibitem[Chen et~al.(2020)Chen, Wei, Huang, Ding, and Li]{chen2020simple}
Ming Chen, Zhewei Wei, Zengfeng Huang, Bolin Ding, and Yaliang Li.
\newblock Simple and deep graph convolutional networks.
\newblock In \emph{Proceedings of the 37th International Conference on Machine
  Learning, {ICML} 2020, 13-18 July 2020, Virtual Event}, volume 119 of
  \emph{Proceedings of Machine Learning Research}, pp.\  1725--1735. {PMLR},
  2020.
\newblock URL \url{http://proceedings.mlr.press/v119/chen20v.html}.

\bibitem[Corso et~al.(2020)Corso, Cavalleri, Beaini, Li{\`{o}}, and
  Velickovic]{corso2020principal}
Gabriele Corso, Luca Cavalleri, Dominique Beaini, Pietro Li{\`{o}}, and Petar
  Velickovic.
\newblock Principal neighbourhood aggregation for graph nets.
\newblock In Hugo Larochelle, Marc'Aurelio Ranzato, Raia Hadsell,
  Maria{-}Florina Balcan, and Hsuan{-}Tien Lin (eds.), \emph{Advances in Neural
  Information Processing Systems 33: Annual Conference on Neural Information
  Processing Systems 2020, NeurIPS 2020, December 6-12, 2020, virtual}, 2020.
\newblock URL
  \url{https://proceedings.neurips.cc/paper/2020/hash/99cad265a1768cc2dd013f0e740300ae-Abstract.html}.

\bibitem[Deshpande et~al.(2018)Deshpande, Sen, Montanari, and
  Mossel]{deshpande2018contextual}
Yash Deshpande, Subhabrata Sen, Andrea Montanari, and Elchanan Mossel.
\newblock Contextual stochastic block models.
\newblock In Samy Bengio, Hanna~M. Wallach, Hugo Larochelle, Kristen Grauman,
  Nicol{\`{o}} Cesa{-}Bianchi, and Roman Garnett (eds.), \emph{Advances in
  Neural Information Processing Systems 31: Annual Conference on Neural
  Information Processing Systems 2018, NeurIPS 2018, December 3-8, 2018,
  Montr{\'{e}}al, Canada}, pp.\  8590--8602, 2018.
\newblock URL
  \url{https://proceedings.neurips.cc/paper/2018/hash/08fc80de8121419136e443a70489c123-Abstract.html}.

\bibitem[Dwivedi \& Bresson(2020)Dwivedi and
  Bresson]{dwivedi2020generalization}
Vijay~Prakash Dwivedi and Xavier Bresson.
\newblock A generalization of transformer networks to graphs.
\newblock \emph{arXiv preprint arXiv:2012.09699}, 2020.
\newblock URL \url{https://arxiv.org/abs/2012.09699}.

\bibitem[Fey \& Lenssen(2019)Fey and Lenssen]{fey-pytorch-geometric}
Matthias Fey and Jan~E. Lenssen.
\newblock Fast graph representation learning with {PyTorch Geometric}.
\newblock In \emph{ICLR Workshop on Representation Learning on Graphs and
  Manifolds}, 2019.

\bibitem[Fountoulakis et~al.(2022)Fountoulakis, Levi, Yang, Baranwal, and
  Jagannath]{kimongat}
Kimon Fountoulakis, Amit Levi, Shenghao Yang, Aseem Baranwal, and Aukosh
  Jagannath.
\newblock Graph attention retrospective.
\newblock \emph{arXiv preprint arXiv:2202.13060}, 2022.
\newblock URL \url{https://arxiv.org/abs/2202.13060}.

\bibitem[Gilmer et~al.(2017)Gilmer, Schoenholz, Riley, Vinyals, and
  Dahl]{gilmer2017neural}
Justin Gilmer, Samuel~S. Schoenholz, Patrick~F. Riley, Oriol Vinyals, and
  George~E. Dahl.
\newblock Neural message passing for quantum chemistry.
\newblock In Doina Precup and Yee~Whye Teh (eds.), \emph{Proceedings of the
  34th International Conference on Machine Learning, {ICML} 2017, Sydney, NSW,
  Australia, 6-11 August 2017}, volume~70 of \emph{Proceedings of Machine
  Learning Research}, pp.\  1263--1272. {PMLR}, 2017.
\newblock URL \url{http://proceedings.mlr.press/v70/gilmer17a.html}.

\bibitem[Glorot \& Bengio(2010)Glorot and Bengio]{pmlr-v9-glorot10a}
Xavier Glorot and Yoshua Bengio.
\newblock Understanding the difficulty of training deep feedforward neural
  networks.
\newblock In \emph{International Conference on Artificial Intelligence and
  Statistics (AISTATS)}, 2010.

\bibitem[Hamilton et~al.(2017{\natexlab{a}})Hamilton, Ying, and
  Leskovec]{hamilton2017representation}
William~L. Hamilton, Rex Ying, and Jure Leskovec.
\newblock Representation learning on graphs: Methods and applications.
\newblock \emph{{IEEE} Data Eng. Bull.}, 40, 2017{\natexlab{a}}.

\bibitem[Hamilton et~al.(2017{\natexlab{b}})Hamilton, Ying, and
  Leskovec]{hamilton2017inductive}
William~L. Hamilton, Zhitao Ying, and Jure Leskovec.
\newblock Inductive representation learning on large graphs.
\newblock In Isabelle Guyon, Ulrike von Luxburg, Samy Bengio, Hanna~M. Wallach,
  Rob Fergus, S.~V.~N. Vishwanathan, and Roman Garnett (eds.), \emph{Advances
  in Neural Information Processing Systems 30: Annual Conference on Neural
  Information Processing Systems 2017, December 4-9, 2017, Long Beach, CA,
  {USA}}, pp.\  1024--1034, 2017{\natexlab{b}}.
\newblock URL
  \url{https://proceedings.neurips.cc/paper/2017/hash/5dd9db5e033da9c6fb5ba83c7a7ebea9-Abstract.html}.

\bibitem[He et~al.(2015)He, Zhang, Ren, and Sun]{he2015delving}
Kaiming He, Xiangyu Zhang, Shaoqing Ren, and Jian Sun.
\newblock Delving deep into rectifiers: Surpassing human-level performance on
  imagenet classification.
\newblock In \emph{2015 {IEEE} International Conference on Computer Vision,
  {ICCV} 2015, Santiago, Chile, December 7-13, 2015}, pp.\  1026--1034. {IEEE}
  Computer Society, 2015.
\newblock \doi{10.1109/ICCV.2015.123}.
\newblock URL \url{https://doi.org/10.1109/ICCV.2015.123}.

\bibitem[Hu et~al.(2020{\natexlab{a}})Hu, Fey, Zitnik, Dong, Ren, Liu, Catasta,
  and Leskovec]{hu2020open}
Weihua Hu, Matthias Fey, Marinka Zitnik, Yuxiao Dong, Hongyu Ren, Bowen Liu,
  Michele Catasta, and Jure Leskovec.
\newblock Open graph benchmark: Datasets for machine learning on graphs.
\newblock In Hugo Larochelle, Marc'Aurelio Ranzato, Raia Hadsell,
  Maria{-}Florina Balcan, and Hsuan{-}Tien Lin (eds.), \emph{Advances in Neural
  Information Processing Systems 33: Annual Conference on Neural Information
  Processing Systems 2020, NeurIPS 2020, December 6-12, 2020, virtual},
  2020{\natexlab{a}}.
\newblock URL
  \url{https://proceedings.neurips.cc/paper/2020/hash/fb60d411a5c5b72b2e7d3527cfc84fd0-Abstract.html}.

\bibitem[Hu et~al.(2020{\natexlab{b}})Hu, Dong, Wang, and
  Sun]{hu2020heterogeneous}
Ziniu Hu, Yuxiao Dong, Kuansan Wang, and Yizhou Sun.
\newblock Heterogeneous graph transformer.
\newblock In Yennun Huang, Irwin King, Tie{-}Yan Liu, and Maarten van Steen
  (eds.), \emph{{WWW} '20: The Web Conference 2020, Taipei, Taiwan, April
  20-24, 2020}, pp.\  2704--2710. {ACM} / {IW3C2}, 2020{\natexlab{b}}.
\newblock \doi{10.1145/3366423.3380027}.
\newblock URL \url{https://doi.org/10.1145/3366423.3380027}.

\bibitem[Ioffe \& Szegedy(2015)Ioffe and Szegedy]{ioffe2015batch}
Sergey Ioffe and Christian Szegedy.
\newblock Batch normalization: Accelerating deep network training by reducing
  internal covariate shift.
\newblock In Francis~R. Bach and David~M. Blei (eds.), \emph{Proceedings of the
  32nd International Conference on Machine Learning, {ICML} 2015, Lille,
  France, 6-11 July 2015}, volume~37 of \emph{{JMLR} Workshop and Conference
  Proceedings}, pp.\  448--456. JMLR.org, 2015.
\newblock URL \url{http://proceedings.mlr.press/v37/ioffe15.html}.

\bibitem[Kim \& Oh(2021)Kim and Oh]{kim2022find}
Dongkwan Kim and Alice Oh.
\newblock How to find your friendly neighborhood: Graph attention design with
  self-supervision.
\newblock In \emph{International Conference on Learning Representations
  (ICLR)}, 2021.

\bibitem[Kingma \& Ba(2015)Kingma and Ba]{kingma-adam}
Diederik~P. Kingma and Jimmy Ba.
\newblock Adam: {A} method for stochastic optimization.
\newblock In Yoshua Bengio and Yann LeCun (eds.), \emph{3rd International
  Conference on Learning Representations, {ICLR} 2015, San Diego, CA, USA, May
  7-9, 2015, Conference Track Proceedings}, 2015.
\newblock URL \url{http://arxiv.org/abs/1412.6980}.

\bibitem[Kipf \& Welling(2017)Kipf and Welling]{kipf2016semi}
Thomas~N. Kipf and Max Welling.
\newblock Semi-supervised classification with graph convolutional networks.
\newblock In \emph{5th International Conference on Learning Representations,
  {ICLR} 2017, Toulon, France, April 24-26, 2017, Conference Track
  Proceedings}. OpenReview.net, 2017.
\newblock URL \url{https://openreview.net/forum?id=SJU4ayYgl}.

\bibitem[Knyazev et~al.(2019)Knyazev, Taylor, and
  Amer]{knyazev2019understanding}
Boris Knyazev, Graham~W. Taylor, and Mohamed~R. Amer.
\newblock Understanding attention and generalization in graph neural networks.
\newblock In Hanna~M. Wallach, Hugo Larochelle, Alina Beygelzimer, Florence
  d'Alch{\'{e}}{-}Buc, Emily~B. Fox, and Roman Garnett (eds.), \emph{Advances
  in Neural Information Processing Systems 32: Annual Conference on Neural
  Information Processing Systems 2019, NeurIPS 2019, December 8-14, 2019,
  Vancouver, BC, Canada}, pp.\  4204--4214, 2019.
\newblock URL
  \url{https://proceedings.neurips.cc/paper/2019/hash/4c5bcfec8584af0d967f1ab10179ca4b-Abstract.html}.

\bibitem[Kreuzer et~al.(2021)Kreuzer, Beaini, Hamilton, L{\'e}tourneau, and
  Tossou]{kreuzer2021rethinking}
Devin Kreuzer, Dominique Beaini, Will Hamilton, Vincent L{\'e}tourneau, and
  Prudencio Tossou.
\newblock Rethinking graph transformers with spectral attention.
\newblock \emph{Advances in Neural Information Processing Systems (NeurIPS)},
  34, 2021.

\bibitem[Lee et~al.(2019)Lee, Rossi, Kim, Ahmed, and Koh]{lee2019attention}
John~Boaz Lee, Ryan~A Rossi, Sungchul Kim, Nesreen~K Ahmed, and Eunyee Koh.
\newblock Attention models in graphs: A survey.
\newblock \emph{ACM Transactions on Knowledge Discovery from Data (TKDD)}, 13,
  2019.

\bibitem[Morris et~al.(2019)Morris, Ritzert, Fey, Hamilton, Lenssen, Rattan,
  and Grohe]{morris2019weisfeiler}
Christopher Morris, Martin Ritzert, Matthias Fey, William~L. Hamilton, Jan~Eric
  Lenssen, Gaurav Rattan, and Martin Grohe.
\newblock Weisfeiler and leman go neural: Higher-order graph neural networks.
\newblock In \emph{The Thirty-Third {AAAI} Conference on Artificial
  Intelligence, {AAAI} 2019, The Thirty-First Innovative Applications of
  Artificial Intelligence Conference, {IAAI} 2019, The Ninth {AAAI} Symposium
  on Educational Advances in Artificial Intelligence, {EAAI} 2019, Honolulu,
  Hawaii, USA, January 27 - February 1, 2019}, pp.\  4602--4609. {AAAI} Press,
  2019.
\newblock \doi{10.1609/aaai.v33i01.33014602}.
\newblock URL \url{https://doi.org/10.1609/aaai.v33i01.33014602}.

\bibitem[Rigollet \& H{\"u}tter(2015)Rigollet and H{\"u}tter]{rigollet2015high}
P.~Rigollet and J.-C. H{\"u}tter.
\newblock High dimensional statistics.
\newblock \emph{Lecture notes for course 18S997}, 813:\penalty0 814, 2015.

\bibitem[Rozemberczki et~al.(2021)Rozemberczki, Allen, and
  Sarkar]{rozemberczki2021multi}
Benedek Rozemberczki, Carl Allen, and Rik Sarkar.
\newblock Multi-scale attributed node embedding.
\newblock \emph{Journal of Complex Networks}, 2021.

\bibitem[Scarselli et~al.(2008)Scarselli, Gori, Tsoi, Hagenbuchner, and
  Monfardini]{scarselli2008gnns}
Franco Scarselli, Marco Gori, Ah~Chung Tsoi, Markus Hagenbuchner, and Gabriele
  Monfardini.
\newblock The graph neural network model.
\newblock \emph{IEEE transactions on neural networks}, 20, 2008.

\bibitem[Sen et~al.(2008)Sen, Namata, Bilgic, Getoor, Galligher, and
  Eliassi-Rad]{sen2008collective}
Prithviraj Sen, Galileo Namata, Mustafa Bilgic, Lise Getoor, Brian Galligher,
  and Tina Eliassi-Rad.
\newblock Collective classification in network data.
\newblock \emph{AI magazine}, 29\penalty0 (3), 2008.

\bibitem[Shchur et~al.(2018)Shchur, Mumme, Bojchevski, and
  G{\"u}nnemann]{shchur2018pitfalls}
Oleksandr Shchur, Maximilian Mumme, Aleksandar Bojchevski, and Stephan
  G{\"u}nnemann.
\newblock Pitfalls of graph neural network evaluation.
\newblock \emph{arXiv preprint arXiv:1811.05868}, 2018.
\newblock URL \url{https://arxiv.org/abs/1811.05868}.

\bibitem[Srivastava et~al.(2014)Srivastava, Hinton, Krizhevsky, Sutskever, and
  Salakhutdinov]{srivastava14a-dropout}
Nitish Srivastava, Geoffrey Hinton, Alex Krizhevsky, Ilya Sutskever, and Ruslan
  Salakhutdinov.
\newblock Dropout: A simple way to prevent neural networks from overfitting.
\newblock \emph{Journal of Machine Learning Research}, 15\penalty0
  (56):\penalty0 1929--1958, 2014.

\bibitem[Thekumparampil et~al.(2018)Thekumparampil, Wang, Oh, and
  Li]{thekumparampil2018attention}
Kiran~K Thekumparampil, Chong Wang, Sewoong Oh, and Li-Jia Li.
\newblock Attention-based graph neural network for semi-supervised learning.
\newblock \emph{arXiv preprint arXiv:1803.03735}, 2018.
\newblock URL \url{https://arxiv.org/abs/1803.03735}.

\bibitem[Vaswani et~al.(2017)Vaswani, Shazeer, Parmar, Uszkoreit, Jones, Gomez,
  Kaiser, and Polosukhin]{vaswani2017attention}
Ashish Vaswani, Noam Shazeer, Niki Parmar, Jakob Uszkoreit, Llion Jones,
  Aidan~N. Gomez, Lukasz Kaiser, and Illia Polosukhin.
\newblock Attention is all you need.
\newblock In Isabelle Guyon, Ulrike von Luxburg, Samy Bengio, Hanna~M. Wallach,
  Rob Fergus, S.~V.~N. Vishwanathan, and Roman Garnett (eds.), \emph{Advances
  in Neural Information Processing Systems 30: Annual Conference on Neural
  Information Processing Systems 2017, December 4-9, 2017, Long Beach, CA,
  {USA}}, pp.\  5998--6008, 2017.
\newblock URL
  \url{https://proceedings.neurips.cc/paper/2017/hash/3f5ee243547dee91fbd053c1c4a845aa-Abstract.html}.

\bibitem[Velickovic et~al.(2018)Velickovic, Cucurull, Casanova, Romero,
  Li{\`{o}}, and Bengio]{velivckovic2017graph}
Petar Velickovic, Guillem Cucurull, Arantxa Casanova, Adriana Romero, Pietro
  Li{\`{o}}, and Yoshua Bengio.
\newblock Graph attention networks.
\newblock In \emph{6th International Conference on Learning Representations,
  {ICLR} 2018, Vancouver, BC, Canada, April 30 - May 3, 2018, Conference Track
  Proceedings}. OpenReview.net, 2018.
\newblock URL \url{https://openreview.net/forum?id=rJXMpikCZ}.

\bibitem[Wang et~al.(2021{\natexlab{a}})Wang, Ying, Huang, and
  Leskovec]{wang2020multi}
Guangtao Wang, Rex Ying, Jing Huang, and Jure Leskovec.
\newblock Multi-hop attention graph neural networks.
\newblock In \emph{International Joint Conference on Artificial Intelligence
  (IJCAI)}, 2021{\natexlab{a}}.

\bibitem[Wang et~al.(2020)Wang, Shen, Huang, Wu, Dong, and
  Kanakia]{wang2020microsoft}
Kuansan Wang, Zhihong Shen, Chiyuan Huang, Chieh-Han Wu, Yuxiao Dong, and
  Anshul Kanakia.
\newblock Microsoft academic graph: When experts are not enough.
\newblock \emph{Quantitative Science Studies}, 2020.

\bibitem[Wang et~al.(2019{\natexlab{a}})Wang, Zheng, Ye, Gan, Li, Song, Zhou,
  Ma, Yu, Gai, Xiao, He, Karypis, Li, and Zhang]{wang2019dgl}
Minjie Wang, Da~Zheng, Zihao Ye, Quan Gan, Mufei Li, Xiang Song, Jinjing Zhou,
  Chao Ma, Lingfan Yu, Yu~Gai, Tianjun Xiao, Tong He, George Karypis, Jinyang
  Li, and Zheng Zhang.
\newblock Deep graph library: A graph-centric, highly-performant package for
  graph neural networks.
\newblock \emph{arXiv preprint arXiv:1909.01315}, 2019{\natexlab{a}}.
\newblock URL \url{https://arxiv.org/abs/1909.01315}.

\bibitem[Wang et~al.(2019{\natexlab{b}})Wang, Ji, Shi, Wang, Ye, Cui, and
  Yu]{wang2019heterogeneous}
Xiao Wang, Houye Ji, Chuan Shi, Bai Wang, Yanfang Ye, Peng Cui, and Philip~S.
  Yu.
\newblock Heterogeneous graph attention network.
\newblock In Ling Liu, Ryen~W. White, Amin Mantrach, Fabrizio Silvestri,
  Julian~J. McAuley, Ricardo Baeza{-}Yates, and Leila Zia (eds.), \emph{The
  World Wide Web Conference, {WWW} 2019, San Francisco, CA, USA, May 13-17,
  2019}, pp.\  2022--2032. {ACM}, 2019{\natexlab{b}}.
\newblock \doi{10.1145/3308558.3313562}.
\newblock URL \url{https://doi.org/10.1145/3308558.3313562}.

\bibitem[Wang et~al.(2021{\natexlab{b}})Wang, Jin, Zhang, Yu, Zhang, and
  Wipf]{wang2021bag}
Yangkun Wang, Jiarui Jin, Weinan Zhang, Yong Yu, Zheng Zhang, and David Wipf.
\newblock Bag of tricks for node classification with graph neural networks.
\newblock \emph{arXiv preprint arXiv:2103.13355}, 2021{\natexlab{b}}.
\newblock URL \url{https://arxiv.org/abs/2103.13355}.

\bibitem[Wu et~al.(2020{\natexlab{a}})Wu, Sun, Zhang, Xie, and
  Cui]{wu2020graph}
Shiwen Wu, Fei Sun, Wentao Zhang, Xu~Xie, and Bin Cui.
\newblock Graph neural networks in recommender systems: a survey.
\newblock \emph{ACM Computing Surveys (CSUR)}, 2020{\natexlab{a}}.

\bibitem[Wu et~al.(2020{\natexlab{b}})Wu, Pan, Chen, Long, Zhang, and
  Philip]{wu2020comprehensive}
Zonghan Wu, Shirui Pan, Fengwen Chen, Guodong Long, Chengqi Zhang, and S~Yu
  Philip.
\newblock A comprehensive survey on graph neural networks.
\newblock \emph{IEEE transactions on neural networks and learning systems}, 32,
  2020{\natexlab{b}}.

\bibitem[Xu et~al.(2019)Xu, Hu, Leskovec, and Jegelka]{xu2018powerful}
Keyulu Xu, Weihua Hu, Jure Leskovec, and Stefanie Jegelka.
\newblock How powerful are graph neural networks?
\newblock In \emph{7th International Conference on Learning Representations,
  {ICLR} 2019, New Orleans, LA, USA, May 6-9, 2019}. OpenReview.net, 2019.
\newblock URL \url{https://openreview.net/forum?id=ryGs6iA5Km}.

\bibitem[Yang et~al.(2016)Yang, Cohen, and Salakhutdinov]{yang2016revisiting}
Zhilin Yang, William~W. Cohen, and Ruslan Salakhutdinov.
\newblock Revisiting semi-supervised learning with graph embeddings.
\newblock In Maria{-}Florina Balcan and Kilian~Q. Weinberger (eds.),
  \emph{Proceedings of the 33nd International Conference on Machine Learning,
  {ICML} 2016, New York City, NY, USA, June 19-24, 2016}, volume~48 of
  \emph{{JMLR} Workshop and Conference Proceedings}, pp.\  40--48. JMLR.org,
  2016.
\newblock URL \url{http://proceedings.mlr.press/v48/yanga16.html}.

\bibitem[You et~al.(2020)You, Ying, and Leskovec]{you2020design}
Jiaxuan You, Zhitao Ying, and Jure Leskovec.
\newblock Design space for graph neural networks.
\newblock \emph{Advances in Neural Information Processing Systems (NeurIPS)},
  33, 2020.

\bibitem[Yun et~al.(2019)Yun, Jeong, Kim, Kang, and Kim]{yun2019graph}
Seongjun Yun, Minbyul Jeong, Raehyun Kim, Jaewoo Kang, and Hyunwoo~J. Kim.
\newblock Graph transformer networks.
\newblock In Hanna~M. Wallach, Hugo Larochelle, Alina Beygelzimer, Florence
  d'Alch{\'{e}}{-}Buc, Emily~B. Fox, and Roman Garnett (eds.), \emph{Advances
  in Neural Information Processing Systems 32: Annual Conference on Neural
  Information Processing Systems 2019, NeurIPS 2019, December 8-14, 2019,
  Vancouver, BC, Canada}, pp.\  11960--11970, 2019.
\newblock URL
  \url{https://proceedings.neurips.cc/paper/2019/hash/9d63484abb477c97640154d40595a3bb-Abstract.html}.

\end{thebibliography}

	\newpage
	\appendix  %
	\renewcommand{\partname}{}  %
	\part{Appendix} %

	\parttoc %

	\clearpage
	\newpage

\section{Theoretical results} \label{app:theory}

\subsection{A hard example for GCN}
In this subsection, we present a dataset and classification task for which GCN performs poorly.
Note that we follow the similar techniques and notation as~\citep{kimongat}, as described in the main paper.

We recall our data model. Fix $n,d\in \N$ and let $\eps_1,\ldots,\eps_n$ be i.i.d uniformly sampled from $\{-1,0,1\}$. Let $C_k=\{j\in [n]\mid \eps_j=k\}$ for $k\in \{-1,0,1\}$. For each index $i\in [n]$, we set the feature vector $\bX_i\in \R^d$ as $\bX_i\sim \Normal(\eps_i\cdot\bmu, \bI\cdot \sigma^2)$, where $\bmu\in \R^d$, $\sigma\in \R$ and $\bI\in \zo^{d\times d}$ is the identity matrix. For a given pair $p,q\in [0,1]$ we consider the stochastic adjacency matrix $\bA\in\zo^{n\times n}$ defined as follows. For $i,j\in [n]$ in the same class, we set $a_{ij}\sim \Ber(p)$, and if $i,j$ are in different classes, we set $a_{ij}\sim\Ber(q)$. We let $\bD\in \R^{n\times n}$ be a diagonal matrix containing the degrees of the vertices. We denote by $(\bX,\bA)\sim \CSBM(n,p,q,\bmu,\sigma^2)$ a sample obtained according to the above random process.

The task we wish to solve is classifying $C_0$ vs $C_{-1}\cup C_1$. Namely, we want our model $\phi$ to satisfy $\phi(\bX_i)<0$ if and only if $i\in C_0$. Moreover, note that the posed problem \emph{is not linearly classifiable}.

To this end, we start by stating an assumption on the choice of parameters. This assumption is necessary to achieve degree concentration in the graph.
\begin{assumption} \label{ass:p_q}$p,q=\Omega(\log^2n/n)$ . \end{assumption}

We now show the distribution of the convolved features. The following lemma can be easily obtained using the techniques in~\citet{kimongcn}. 
\begin{lemma}\label{lem:Dist_after_conv}Fix $p,q$ satisfying \cref{ass:p_q}.
With probability at least $1-o(1)$ over $\bA$ and $\{\eps_i\}_i$,
\[(\bD^{-1}\bA\bX)_i\sim \Normal\left(\eps_i\cdot\frac{p-q}{p+2q}\bmu, \frac{\sigma^2}{n(p+2q)}\right),\qquad\forall i \in [n].\] 
\end{lemma}

To prove the above lemma, we need the following definition of our high probability event.

\begin{definition}\label{def:Event-E}
We define the even $\calbE$ as the intersection of the following events over $\bA$ and $\{\eps_i\}_i$:
\begin{enumerate}
    \item $\calbE_1$ is the event that $|C_0|=\frac{n}{3}\pm O(\sqrt{n\log n})$, $|C_1|=\frac{n}{3}\pm O(\sqrt{n\log n})$ \\ and {$|C_{-1}|=\frac{n}{3}\pm O(\sqrt{n\log n})$}.
    \item $\calbE_2$ is the event that for each $i\in [n]$, $\bD_{ii}=\frac{n(p+2q)}{3}\left(1\pm \frac{10}{\sqrt{\log n}}\right)$.
    \item $\calbE_3$ is the event that for each $i \in [n]$ and $k\in \{-1,0,1\}$,
    \begin{align*}
        |N_i\cap C_k|=\begin{cases}
        \bD_{ii}\cdot \frac{p}{p+2q}\cdot\left(1\pm \frac{10}{\sqrt{\log n}}\right) &\text{if }\; i\in C_k\\
        \bD_{ii}\cdot \frac{q}{p+2q}\cdot\left(1\pm \frac{10}{\sqrt{\log n}}\right) &\text{if }\; i\notin C_k
        \end{cases}.
    \end{align*}
\end{enumerate}
\end{definition}

The following lemma is a direct application of Chernoff bound and a union bound.
\begin{lemma}\label{lem:E} With probability at least $1-1/\poly(n)$ the event $\calbE$ holds.
\end{lemma}

\begin{proof}[Proof of \cref{lem:Dist_after_conv}] By applying \cref{lem:E}, and conditioned on $\calbE$, for any $i\in [n]$
\[(\bD^{-1}\bA\bX)_i=\frac{1}{\bD_{ii}}\sum_{j\in N_i}\bX_j=\frac{1}{\bD_{ii}}\left(\sum_{j\in N_i\cap C_{-1}}\bX_j+\sum_{j\in N_i\cap C_0}\bX_j+\sum_{j\in N_i\cap C_1}\bX_j\right).\]
Using the definition of $\calbE$ and properties of Gaussian distributions the lemma follows.
\end{proof}

\cref{lem:Dist_after_conv} shows that essentially, the convolution reduced the variance and moved the means closer, but the structure of the problem stayed exactly the same. Therefore, one layer of GCN cannot separate $C_0$ from $C_{-1}\cup C_1$ with high probability.

\subsection{A solution for GAT and CAT}

In what follows, we show that GAT is able to handle the above classification task easily when the distance between the means is large enough. Then, we show how the additional convolution on the inputs to the score function improves the regime of perfect classification when the graph is not too noisy.
Our main technical lemma considers a specific attention architecture and characterize the attention scores for our data model. 

\begin{lemma}\label{lem:edge_separation} Suppose that $p,q$ satisfy \cref{ass:p_q}, $\|\bmu\|\ge \omega\left(\sigma \sqrt{\log n}\right)$, fix the $\LeakyRelu$ constant $\beta\in (0,1)$ and $R\in \R$. Then, there exists a choice of attention architecture $\Psi$ such that with probability at least $1-o_n(1)$ over the data $(\bX,\bA)\sim\CSBM(n,p,q,\bmu,\sigma^2)$ the following holds.
\begin{align*}
   \Psi(\bX_i,\bX_j)=
    \begin{cases}
    10R\beta\|\bmu\|(1\pm o(1))&\text{if } i,j\in C_1^2\\
    -2R\|\bmu\|(1+2\beta)(1\pm o(1)) &\text{if } i,j\in C_{-1}^2\\
    -2R\|\bmu\|(1+5\beta)(1\pm o(1)) &\text{if } i\in C_1,\; j\in C_{-1}\\
    10R\beta\|\bmu\|(1\pm o(1)) &\text{if } i\in C_{-1},\; j\in C_1\\
    -\frac{R}{2}\|\bmu\|(1-21\beta)(1\pm o(1)) &\text{if } i\in C_{0},\; j\in C_1\\
    -\frac{R}{2}\|\bmu\|(1-11\beta)(1\pm o(1)) &\text{if } i\in C_{0},\; j\in C_{-1}\\
    -\frac{R}{2}\|\bmu\|(1-5\beta)(1\pm o(1)) &\text{if } i\in C_{1},\; j\in C_0\\
    -\frac{R}{2}\|\bmu\|(1-5\beta)(1\pm o(1)) &\text{if } i\in C_{-1},\; j\in C_0\\
    2R\beta\|\bmu\|(1\pm o(1)) &\text{if } i,j\in C_0^2
    \end{cases}.
\end{align*}
\end{lemma}

\begin{proof}We consider as an ansatz the following two layer architecture $\Psi$.
\[\tilde\bw\eqdef\frac{\bmu}{\|\bmu\|},
\qquad 
\bS\eqdef  
\begin{bmatrix}  
1 & 1  \\
-1 & -1 \\
1 & -1\\
-1 & 1\\
0 & 1\\
1 & 0\\
0 &-1 \\
-1 & 0
\end{bmatrix},
\qquad \bb\eqdef\begin{bmatrix}
-3/2\\
-3/2\\
-3/2\\
-3/2\\
-1/2\\
-1/2\\
-1/2\\
-1/2
\end{bmatrix}\cdot\|\bmu\|, \qquad
\boldr\eqdef R \cdot \begin{bmatrix} 2 \\ -2 \\ -2 \\ 2 \\ -1 \\-1 \\-1\\-1 \end{bmatrix},
\]
where $R>0$ is an arbitrary scaling parameter. The output of the attention model is defined as 
\[\Psi(\bX_i,\bX_j)\eqdef\boldr^T\cdot \LeakyRelu\left(\bS \cdot\begin{bmatrix} \tilde \bw^T \bX_i\\
\tilde \bw^T \bX_j\end{bmatrix}+\bb \right).\]

Let $\bDelta_{ij}\eqdef\bS \cdot\begin{bmatrix} \tilde \bw^T \bX_i\\
\tilde \bw^T \bX_j\end{bmatrix}+\bb\in \R^{8}$, and note that for each element $t\in [8]$ of $\bDelta_{ij}$, we have that $(\bDelta_{ij})_t= \bS_{t,1}\tilde\bw^T\bX_i+\bS_{t,2}\tilde \bw^T\bX_j+\bb_t$. Note that the random variable $(\bDelta_{ij})_t$ is distributed as follows: 

\begin{align*}
    (\bDelta_{ij})_t \sim 
    \begin{cases}
    \Normal\left((\bS_{t,1}+\bS_{t,2})\tilde\bw^T\bmu+\bb_t,\; \|\bS_{t,\ast}\|^2{\sigma^2}\right) &\text{if } i,j\in C_1^2\\
    \Normal\left(-(\bS_{t,1}+\bS_{t,2})\tilde\bw^T\bmu+\bb_t,\; \|\bS_{t,\ast}\|^2{\sigma^2}\right) &\text{if } i,j\in C_{-1}^2\\
    \Normal\left((\bS_{t,1}-\bS_{t,2})\tilde\bw^T\bmu+\bb_t,\; \|\bS_{t,\ast}\|^2{\sigma^2}\right) &\text{if } i\in C_1,\; j\in C_{-1}\\
    \Normal\left(-(\bS_{t,1}-\bS_{t,2})\tilde\bw^T\bmu+\bb_t,\;  \|\bS_{t,\ast}\|^2{\sigma^2}\right) &\text{if } i\in C_{-1},\; j\in C_1\\
    \Normal\left(\bS_{t,2}\tilde\bw^T\bmu+\bb_t,\;  \|\bS_{t,\ast}\|^2{\sigma^2}\right) &\text{if } i\in C_{0},\; j\in C_1\\
    \Normal\left(-\bS_{t,2}\tilde\bw^T\bmu+\bb_t,\;  \|\bS_{t,\ast}\|^2{\sigma^2}\right) &\text{if } i\in C_{0},\; j\in C_{-1}\\
    \Normal\left(\bS_{t,1}\tilde\bw^T\bmu+\bb_t,\;  \|\bS_{t,\ast}\|^2{\sigma^2}\right) &\text{if } i\in C_{1},\; j\in C_0\\
    \Normal\left(-\bS_{t,1}\tilde\bw^T\bmu+\bb_t,\;  \|\bS_{t,\ast}\|^2{\sigma^2}\right) &\text{if } i\in C_{-1},\; j\in C_0\\
     \Normal\left(\bb_t,\; \|\bS_{t,\ast}\|^2{\sigma^2}\right) &\text{if } i,j\in C_0^2
    \end{cases}.
\end{align*}
Therefore, for a fixed $i,j\in [n]^2$ we have that the entries of $\bDelta_{ij}$ are distributed as follows (where we use $\Normal^y_{x}$ as abbreviation for the Gaussian $\Normal(x,y)$)
\begin{align*}
  \begin{bmatrix}
  \Normal_{\frac{\|\bmu\|}{2}}^{4\sigma^2} & \Normal_{\frac{-7\|\bmu\|}{2}}^{4\sigma^2} & \Normal_{-\frac{3\|\bmu\|}{2}}^{4\sigma^2} & \Normal_{-\frac{3\|\bmu\|}{2}}^{4\sigma^2} & \Normal_{\frac{\|\bmu\|}{2}}^{\sigma^2} & \Normal_{\frac{\|\bmu\|}{2}}^{\sigma^2} & \Normal_{-\frac{3\|\bmu\|}{2}}^{\sigma^2} & \Normal_{-\frac{3\|\bmu\|}{2}}^{\sigma^2} \end{bmatrix} &\qquad\text{for } i,j\in C_1^2,\\
  \begin{bmatrix}
  \Normal_{-\frac{7\|\bmu\|}{2}}^{4\sigma^2} & \Normal_{\frac{\|\bmu\|}{2}}^{4\sigma^2} & \Normal_{-\frac{3\|\bmu\|}{2}}^{4\sigma^2} & \Normal_{-\frac{3\|\bmu\|}{2}}^{4\sigma^2} & \Normal_{-\frac{3\|\bmu\|}{2}}^{\sigma^2} & \Normal_{-\frac{3\|\bmu\|}{2}}^{\sigma^2} & \Normal_{\frac{\|\bmu\|}{2}}^{\sigma^2} & \Normal_{\frac{\|\bmu\|}{2}}^{\sigma^2} \end{bmatrix} &\qquad\text{for } i,j\in C_{-1}^2,\\
  \begin{bmatrix}
  \Normal_{-\frac{3\|\bmu\|}{2}}^{4\sigma^2} & \Normal_{-\frac{3\|\bmu\|}{2}}^{4\sigma^2} & \Normal_{\frac{\|\bmu\|}{2}}^{4\sigma^2} & \Normal_{-\frac{7\|\bmu\|}{2}}^{4\sigma^2} & \Normal_{-\frac{3\|\bmu\|}{2}}^{\sigma^2} & \Normal_{\frac{\|\bmu\|}{2}}^{\sigma^2} & \Normal_{\frac{\|\bmu\|}{2}}^{\sigma^2} & \Normal_{-\frac{3\|\bmu\|}{2}}^{\sigma^2} \end{bmatrix} &\qquad\text{for } i,j\in C_1\times C_{-1},\\
  \begin{bmatrix}
  \Normal_{-\frac{3\|\bmu\|}{2}}^{4\sigma^2} & \Normal_{-\frac{3\|\bmu\|}{2}}^{4\sigma^2} & \Normal_{-\frac{7\|\bmu\|}{2}}^{4\sigma^2} & \Normal_{\frac{\|\bmu\|}{2}}^{4\sigma^2} & \Normal_{\frac{\|\bmu\|}{2}}^{\sigma^2} & \Normal_{-\frac{3\|\bmu\|}{2}}^{\sigma^2} & \Normal_{-\frac{3\|\bmu\|}{2}}^{\sigma^2} & \Normal_{\frac{\|\bmu\|}{2}}^{\sigma^2} \end{bmatrix} &\qquad\text{for } i,j\in C_{-1}\times C_{1},\\
  \begin{bmatrix}
  \Normal_{-\frac{\|\bmu\|}{2}}^{4\sigma^2} & \Normal_{-\frac{5\|\bmu\|}{2}}^{4\sigma^2} & \Normal_{-\frac{5\|\bmu\|}{2}}^{4\sigma^2} & \Normal_{-\frac{\|\bmu\|}{2}}^{4\sigma^2} & \Normal_{\frac{\|\bmu\|}{2}}^{\sigma^2} & \Normal_{-\frac{\|\bmu\|}{2}}^{\sigma^2} & \Normal_{-\frac{3\|\bmu\|}{2}}^{\sigma^2} & \Normal_{-\frac{\|\bmu\|}{2}}^{\sigma^2} \end{bmatrix} &\qquad\text{for } i,j\in C_{0}\times C_{1},\\
  \begin{bmatrix}
  \Normal_{-\frac{5\|\bmu\|}{2}}^{4\sigma^2} & \Normal_{-\frac{\|\bmu\|}{2}}^{4\sigma^2} & \Normal_{-\frac{\|\bmu\|}{2}}^{4\sigma^2} & \Normal_{-\frac{5\|\bmu\|}{2}}^{4\sigma^2} & \Normal_{-\frac{3\|\bmu\|}{2}}^{\sigma^2} & \Normal_{-\frac{\|\bmu\|}{2}}^{\sigma^2} & \Normal_{\frac{\|\bmu\|}{2}}^{\sigma^2} & \Normal_{-\frac{\|\bmu\|}{2}}^{\sigma^2} \end{bmatrix} &\qquad\text{for } i,j\in C_{0}\times C_{-1},\\
  \begin{bmatrix}
  \Normal_{-\frac{\|\bmu\|}{2}}^{4\sigma^2} & \Normal_{-\frac{5\|\bmu\|}{2}}^{4\sigma^2} & \Normal_{-\frac{\|\bmu\|}{2}}^{4\sigma^2} & \Normal_{-\frac{5\|\bmu\|}{2}}^{4\sigma^2} & \Normal_{-\frac{\|\bmu\|}{2}}^{\sigma^2} & \Normal_{\frac{\|\bmu\|}{2}}^{\sigma^2} & \Normal_{-\frac{\|\bmu\|}{2}}^{\sigma^2} & \Normal_{-\frac{3\|\bmu\|}{2}}^{\sigma^2} \end{bmatrix} &\qquad\text{for } i,j\in C_{1}\times C_{0},\\
  \begin{bmatrix}
  \Normal_{-\frac{5\|\bmu\|}{2}}^{4\sigma^2} & \Normal_{-\frac{\|\bmu\|}{2}}^{4\sigma^2} & \Normal_{-\frac{5\|\bmu\|}{2}}^{4\sigma^2} & \Normal_{-\frac{\|\bmu\|}{2}}^{4\sigma^2} & \Normal_{-\frac{\|\bmu\|}{2}}^{\sigma^2} & \Normal_{-\frac{3\|\bmu\|}{2}}^{\sigma^2} & \Normal_{-\frac{\|\bmu\|}{2}}^{\sigma^2} & \Normal_{\frac{\|\bmu\|}{2}}^{\sigma^2} \end{bmatrix} &\qquad\text{for } i,j\in C_{-1}\times C_{0},\\
  \begin{bmatrix}
  \Normal_{-\frac{3\|\bmu\|}{2}}^{4\sigma^2} & \Normal_{-\frac{3\|\bmu\|}{2}}^{4\sigma^2} & \Normal_{-\frac{3\|\bmu\|}{2}}^{4\sigma^2} & \Normal_{-\frac{3\|\bmu\|}{2}}^{4\sigma^2} & \Normal_{-\frac{\|\bmu\|}{2}}^{\sigma^2} & \Normal_{-\frac{\|\bmu\|}{2}}^{\sigma^2} & \Normal_{-\frac{\|\bmu\|}{2}}^{\sigma^2} & \Normal_{-\frac{\|\bmu\|}{2}}^{\sigma^2} \end{bmatrix} &\qquad\text{for } i,j\in C_{0}^2,\\
\end{align*}
 Next, we will use the following lemma regarding $\LeakyRelu$ concentration.
 \begin{lemma}[Lemma A.6 in~\citet{kimongat}]
  \label{lem:LeakyReluConcentration_gen} Fix $s\in \N$, and let $z_1,\ldots,z_{s}$ be jointly Gaussian random variables with marginals $\bz_i\sim \Normal(\mu_i,\sigma_i^2)$.
 There exists an absolute constant $C>0$ such that with probability at least $1-o_s(1)$, we have 
    \begin{align*}
        \LeakyRelu(z_{i})&=\LeakyRelu\left(\mu_i\right)\pm C\sigma_i\sqrt{\log s},\quad  \mbox{for all }  i\in[s].
    \end{align*}
 \end{lemma}
 \medskip

Using \cref{lem:LeakyReluConcentration_gen} with the assumption on $\|\bmu\|$ and a union bound, we have that with probability at least $1-o_n(1)$, $\LeakyRelu(\bDelta_{ij})$ is  (up to $1\pm o(1)$)
\begin{align*}
    \begin{bmatrix}
  \frac{\|\bmu\|}{2} & \frac{-7\beta\|\bmu\|}{2}& -\frac{3\beta\|\bmu\|}{2} & -\frac{3\beta\|\bmu\|}{2} & \frac{\|\bmu\|}{2} & \frac{\|\bmu\|}{2} & -\frac{3\beta\|\bmu\|}{2} & -\frac{3\beta\|\bmu\|}{2} \end{bmatrix} &\qquad\text{for } i,j\in C_1^2,\\
  \begin{bmatrix}
  -\frac{7\beta\|\bmu\|}{2}& \frac{\|\bmu\|}{2} & -\frac{3\beta\|\bmu\|}{2} & -\frac{3\beta\|\bmu\|}{2} & -\frac{3\beta\|\bmu\|}{2} & -\frac{3\beta\|\bmu\|}{2} & \frac{\|\bmu\|}{2} & \frac{\|\bmu\|}{2} \end{bmatrix} &\qquad\text{for } i,j\in C_{-1}^2,\\
  \begin{bmatrix}
  -\frac{3\beta\|\bmu\|}{2}& -\frac{3\beta\|\bmu\|}{2} &\frac{\|\bmu\|}{2} & -\frac{7\beta\|\bmu\|}{2}& -\frac{3\beta\|\bmu\|}{2} & \frac{\|\bmu\|}{2} &\frac{\|\bmu\|}{2} & -\frac{3\beta\|\bmu\|}{2} \end{bmatrix} &\qquad\text{for } i,j\in C_1\times C_{-1},\\
  \begin{bmatrix}
  -\frac{3\beta\|\bmu\|}{2} & -\frac{3\beta\|\bmu\|}{2} & -\frac{7\beta\|\bmu\|}{2} & \frac{\|\bmu\|}{2} & \frac{\|\bmu\|}{2} &-\frac{3\beta\|\bmu\|}{2}& -\frac{3\beta\|\bmu\|}{2} & \frac{\|\bmu\|}{2} \end{bmatrix} &\qquad\text{for } i,j\in C_{-1}\times C_{1},\\
  \begin{bmatrix}
  -\frac{\beta\|\bmu\|}{2} & -\frac{5\beta\|\bmu\|}{2} & -\frac{5\beta\|\bmu\|}{2} & -\frac{\beta\|\bmu\|}{2} & \frac{\|\bmu\|}{2} & -\frac{\|\bmu\|}{2} & -\frac{3\beta\|\bmu\|}{2}&-\frac{\|\beta\bmu\|}{2} \end{bmatrix} &\qquad\text{for } i,j\in C_{0}\times C_{1},\\
  \begin{bmatrix}
  -\frac{5\beta\|\bmu\|}{2} & -\frac{\beta\|\bmu\|}{2} & -\frac{\beta\|\bmu\|}{2} & -\frac{5\beta\|\bmu\|}{2} & -\frac{3\beta\|\bmu\|}{2}& -\frac{\beta\|\bmu\|}{2} & \frac{\|\bmu\|}{2} & -\frac{\beta\|\bmu\|}{2} \end{bmatrix} &\qquad\text{for } i,j\in C_{0}\times C_{-1},\\
  \begin{bmatrix}
  -\frac{\beta\|\bmu\|}{2}& -\frac{5\beta\|\bmu\|}{2} & -\frac{\beta\|\bmu\|}{2} &-\frac{5\beta\|\bmu\|}{2} & -\frac{\beta\|\bmu\|}{2} & \frac{\|\bmu\|}{2} & -\frac{\beta\|\bmu\|}{2} & -\frac{3\beta\|\bmu\|}{2} \end{bmatrix} &\qquad\text{for } i,j\in C_{1}\times C_{0},\\
  \begin{bmatrix}
  -\frac{5\beta\|\bmu\|}{2}& -\frac{\beta\|\bmu\|}{2} & -\frac{5\beta\|\bmu\|}{2} & -\frac{\beta\|\bmu\|}{2} & -\frac{\beta\|\bmu\|}{2} & -\frac{3\beta\|\bmu\|}{2} & -\frac{\|\beta\bmu\|}{2} & \frac{\|\bmu\|}{2} \end{bmatrix} &\qquad\text{for } i,j\in C_{-1}\times C_{0},\\
  \begin{bmatrix}
  -\frac{3\beta\|\bmu\|}{2} & -\frac{3\beta\|\bmu\|}{2} & -\frac{3\beta\|\bmu\|}{2} & -\frac{3\beta\|\bmu\|}{2} & -\frac{\beta\|\bmu\|}{2} & -\frac{\beta\|\bmu\|}{2} & -\frac{\beta\|\bmu\|}{2} & -\frac{\beta\|\bmu\|}{2} \end{bmatrix} &\qquad\text{for } i,j\in C_{0}^2.\\
\end{align*}
Then, 
\begin{align*}
    \boldr^T\cdot\LeakyRelu(\bDelta_{ij})=
    \begin{cases}
    10R\beta\|\bmu\|(1\pm o(1))&\text{if } i,j\in C_1^2\\
    -2R\|\bmu\|(1+2\beta)(1\pm o(1)) &\text{if } i,j\in C_{-1}^2\\
    -2R\|\bmu\|(1+5\beta)(1\pm o(1)) &\text{if } i\in C_1,\; j\in C_{-1}\\
    10R\beta\|\bmu\|(1\pm o(1)) &\text{if } i\in C_{-1},\; j\in C_1\\
    -\frac{R}{2}\|\bmu\|(1-21\beta)(1\pm o(1)) &\text{if } i\in C_{0},\; j\in C_1\\
    -\frac{R}{2}\|\bmu\|(1-11\beta)(1\pm o(1)) &\text{if } i\in C_{0},\; j\in C_{-1}\\
    -\frac{R}{2}\|\bmu\|(1-5\beta)(1\pm o(1)) &\text{if } i\in C_{1},\; j\in C_0\\
    -\frac{R}{2}\|\bmu\|(1-5\beta)(1\pm o(1)) &\text{if } i\in C_{-1},\; j\in C_0\\
    2R\beta\|\bmu\|(1\pm o(1)) &\text{if } i,j\in C_0^2
    \end{cases},
\end{align*}
and the proof is complete.
\end{proof}

Next we will define our high probability event.
\begin{definition}\label{def:Event-E'}
 $\calbE'\eqdef\calbE\cap \calbE^*$, where $\calbE^*$ is the event that for a fixed $\bw\in \R^d$, all $i \in [n]$ satisfy $| \bw^T\bX_i-\Ex[\bw^T\bX_i]|\le 10\sigma\|\bw\|_2\sqrt{\log n}$.

\end{definition}
The following lemma is obtained by using \cref{lem:E} with standard Gaussian concentration and a union bound.

\begin{lemma}\label{lem:E'} With probability at least $1-1/\poly(n)$ event $\calbE'$ holds.
\end{lemma}

\begin{corollary}\label{cor:gammas}
Suppose that $p,q$ satisfy \cref{ass:p_q}, $\|\bmu\|=\omega(\sigma\sqrt{\log n})$ and fix $R\in \R$. Then, there exists a choice of attention architecture $\Psi$ such that with probability $1-o_n(1)$ over $(\bA,\bX)\sim\CSBM(n,p,q,\bmu,\sigma^2)$ it holds that
\begin{align*}
    \gamma_{ij}=
    \begin{cases}
    \frac{3}{np}(1\pm o(1)) & \text{if}\; i,j\in C_0^2\cup C_1^2\\
    \frac{3}{nq}(1\pm o(1)) & \text{if}\; i,j\in C_{-1}\times C_1\\
    \frac{3}{nq}\exp(-\Theta(R\|\bmu\|)) & \text{if}\; i,j\in C_{-1}\times C_{-1}\cup C_0 \\
    \frac{3}{np}\exp(-\Theta(R\|\bmu\|)) & \text{otherwise}
    \end{cases},
\end{align*}
where $R$ is a parameter of the architecture.
\end{corollary}
\begin{proof}The proof is immediate. First applying the ansatz from \cref{lem:edge_separation} with $\beta<1/25$, Lemma~\ref{lem:E'} and a union bound.  Using the definition of $\gamma_{ij}$ concludes the proof. 
\end{proof}

Next, we prove \cref{prop:gat}  that the model distinguish nodes from $C_0$ for any choice of $p,q$ satisfying \cref{ass:p_q}. We restate the theorem for convince.
\begin{theorem}[Formal restatement of \cref{prop:gat}]\label{thm:main_GAT} Suppose that $p,q$ satisfy \cref{ass:p_q} and $\|\bmu\|_2=\omega(\sigma \sqrt{\log n})$. Then, there exists a choice of attention architecture $\Psi$ such that with probability at least $1-o_n(1)$ over the data $(\bX,\bA) \sim \CSBM(n,p,q,\bmu,\sigma^2)$, the estimator
\[\hat{x}_i\eqdef\sum_{j\in N_i}\gamma_{ij}\tilde{\bw}^T\bX_j+b\;\;\text{where}\;\tilde\bw=\bmu/\|\bmu\|,\;b=-\|\bmu\|/2\]
satisfies $\hat{x}_i<0$ if and only if $i\in C_0$.
\end{theorem}
\begin{proof} Let $\Psi$ be the architecture from \cref{cor:gammas} and let $R$ satisfy $R\|\bmu\|_2 = \omega(1)$. We will compute the mean and variance of the estimator $\hat{x}_i$ conditioned on $\calbE'$. Suppose that $i\in C_0$. By using \cref{cor:gammas}, \cref{def:Event-E'} and our assumption on $\|\bmu\|$ and $R$, we have
\[\max\left\{\frac{3}{np}\exp(-\Theta(R\|\bmu\|)),\frac{3}{nq}\exp(-\Theta(R\|\bmu\|))\right\}=o\left(\frac{1}{n(p+2q)}\right),\]
and therefore
\begin{align*}
    \Ex\left[\hat x_i\mid \calbE'\right]&=\Ex\left[\sum_{k\in\{-1,0,1\}}\sum_{j\in N_i\cap C_k}\gamma_{ij}\tilde \bw^T\bX_j\;\mid\calbE'\right]-\frac{\|\bmu\|}{2}\\
    &=\Ex[|C_0\cap N_i|\mid\calbE']\left(\pm\frac{3}{np}(1\pm o(1))\cdot 10\sigma\sqrt{\log n}\right)\\
    &\qquad+\Ex[|C_1\cap N_i|\mid\calbE']\left( o\left(\frac{1}{n(p+2q)}\right)\cdot(\|\bmu\|\pm 10\sigma\sqrt{\log n})\right)\\
    &\qquad +\Ex[|C_{-1}\cap N_i|\mid \calbE']\left( o\left(\frac{1}{n(p+2q)}\right)\cdot(-\|\bmu\|\pm 10\sigma\sqrt{\log n})\right)-\frac{\|\bmu\|}{2}\\
    &=-\frac{\|\bmu\|}{2}(1\pm o(1)).
\end{align*}
By similar reasoning we have that for $i\in C_{-1}\cup C_1$, $\Ex\left[\hat x_i\mid\calbE'\right]=\frac{\|\bmu\|}{2}(1\pm o(1))$.

Next, we claim that for each $i \in [n]$ the random variable $\hat{x}_i$ given $\calbE'$ is sub-Gaussian with a small sub-Gaussian constant compared to the above expectation. The following lemma is a straightforward adaptation of Lemma A.11 in~\citet{kimongat}, and we provide its proof for completeness.

\begin{lemma}\label{lem:hat_x_sub_g} Conditioned on $\calbE'$, the random variables $\{\hat{\bx}_i\}_{i}$ are sub-Gaussian with parameter $\tilde\sigma_i^2=O\left(\frac{\sigma^2}{np}\right)$ if $i\in C_0\cup C_1$ and $\tilde\sigma_i^2=O\left(\frac{\sigma^2}{nq}\right)$ otherwise.
\end{lemma}
\begin{proof}
Fix $i \in [n]$, and write $\bX_i = \eps_i\bmu + \sigma \bg_i$ where $\bg_i \sim \Normal(0, \bI_d)$, and $\eps_i$ denotes the class membership. Consider $\hat{x}_i$ as a function of $\bg = [\bg_1 \circ \bg_2 \circ\cdots\circ \bg_n] \in \R^{nd}$, where $\circ$ denotes vertical concatenation. Namely, consider the function
\[
	\hat{x}_i = f_i(\bg) \eqdef \sum_{j \in N_i} \gamma_{ij}(\bg) \, \tilde\bw^T(\eps_j\bmu + \sigma \bg_j)-\|\bmu\|/2, \quad i \in [n].
\]
Since $\bg \sim \Normal(0, \bI_{nd})$, proving that $\hat{x}_i$ given $\calbE'$ is sub-Gaussian for each $i \in [n]$, reduces to showing that the function $f_i : \R^{nd} \rightarrow \R$ is Lipschitz over  $E \subseteq \R^{nd}$ defined by $\calbE'$ and the relation $\bX_i = \eps_i\bmu + \sigma \bg_i$. That is, $E \eqdef \left\{\bg \in \R^{nd} \ \big| \ |\tilde\bw^T \bg_i| \le 10\sqrt{\log n}, \forall i \in [n] \right\}$.
Specifically, we show that conditioning on the event $\calbE'$ (which restricts  $\bg \in E$), the Lipschitz constant $L_{f_i}$ of $f_i$ satisfies $L_{f_i} = O\left(\frac{\sigma}{\sqrt{np}}\right)$ for $i\in C_0\cup C_1$ and $L_{f_i}=O\left(\frac{\sigma}{nq}\right)$ otherwise, and hence proving the claim.

First note that event $\calbE'$ induce a transformation which transforms isotropic Gaussians to truncated Gaussians vectors. Similarly to \citet{kimongat}, we can show that this transformation can be obtained by a push-forward mapping whose Lipschitz constant is $1$.
\begin{equation}
\label{eq:pushforward}
	\bar{\bv} = M(\bv) \eqdef [\tau(\bv_1), \tau(\bv_2), \ldots, \tau(\bv_n)]^T
\end{equation}
where $\tau(x) \eqdef \Phi^{-1}((1-2c) \Phi(x) + c)$ for $c = \Phi(-10\sqrt{\log n})$.

To compute the Lipschitz constant of $f_i(\bg)$ for $i \in [n]$, let us denote $\bX = [\bX_1 \circ \bX_2 \circ \cdots \circ \bX_n]$ and consider the function
\[
	\tilde{f}_i(\bX) \eqdef \sum_{j \in N_i} \gamma_{ij}(\bX) \, \tilde\bw^T\bX_j, \quad i \in [n]
\]
Let us assume without loss of generality that $i \in C_0$ (the cases for $i\in C_1$ and $i\in C_{-1}$ are  obtained identically). Conditioning on the event $\calbE'$, which imposes the restriction that $\bX \in \tilde{E}$ where $$\tilde{E} \eqdef\left\{\bX \in \R^{nd} \ \big| \ |\bX_i - \eps_i\bmu| \le 10\sigma\sqrt{\log n}, \forall i \in [n] \right\}.$$ Conditioning on $\calbE'$ (which restricts $\bX, \bX' \in \tilde{E}$), using \cref{cor:gammas} and recalling that $R$ satisfies $R\|\bmu\|_2 = \omega(1)$, we get\footnote{We drop the $(1\pm o(1))$ in the first line of the computation for compactness and use $\simeq$ as notation.}
\begin{align*}
	&\left|\tilde{f}_i(\bX) - \tilde{f}_i(\bX')\right|\\
	&\simeq \left|\sum_{j \in N_i \cap C_0} \frac{3}{np} \tilde\bw^T(\bX_j - \bX_j') + \sum_{j \in N_i \cap C_1} \frac{3}{np}\cdot e^{-\Theta(R\|\bmu\|_2)} \tilde\bw^T(\bX_j - \bX_j') +\sum_{j \in N_i \cap C_{-1}} \frac{3}{np}\cdot e^{-\Theta(R\|\bmu\|_2)} \tilde\bw^T(\bX_j - \bX_j')\right|\\
	&= \left|\begin{bmatrix*}[l] \frac{3}{np}(1 \pm o(1)) \tilde\bw & \mbox{if} \ j \in N_i \cap C_0 \\  \frac{3}{np}\exp(-\Theta(R\|\bmu\|_2))(1 \pm o(1)) \tilde\bw & \mbox{if} \ j \in N_i \cap C_1 \\\frac{3}{np}\exp(-\Theta(R\|\bmu\|_2))(1 \pm o(1)) \tilde\bw & \mbox{if} \ j \in N_i \cap C_{-1} \\
	0 & \mbox{if} \ j \notin N_i \end{bmatrix*}^T_{j \in [n]} \left(\bX-\bX'\right)\right|\\
	&\le \left\|\begin{bmatrix*}[l] \frac{3}{np}(1 \pm o(1)) \tilde\bw & \mbox{if} \ j \in N_i \cap C_0 \\  \frac{3}{np}\exp(-\Theta(R\|\bmu\|_2))(1 \pm o(1)) \tilde\bw & \mbox{if} \ j \in N_i \cap C_1 \\ \frac{3}{np}\exp(-\Theta(R\|\bmu\|_2))(1 \pm o(1)) \tilde\bw & \mbox{if} \ j \in N_i \cap C_{-1} \\
	0 & \mbox{if} \ j \notin N_i \end{bmatrix*}_{j \in [n]} \right\|_2 \left\|\bX - \bX'\right\|_2\\
	&\le \sqrt{\frac{3}{np}}(1 + o(1))\|\tilde\bw\|_2 \left\|\bX - \bX'\right\|_2\\
	&= \sqrt{\frac{3}{np}}(1 + o(1))\left\|\bX - \bX'\right\|_2.
\end{align*}
This shows the Lipschitz constant of $\tilde{f}_i(\bX)$ over $\tilde{E}$ satisfies $L_{\tilde{f}_i} = O\left(\frac{1}{\sqrt{np}}\right)$. On the other hand, by viewing $\bX$ as a function of $\bg$, it is straightforward to see that the function $h(\bg) : \R^{nd} \rightarrow \R^{nd}$ defined by $h(\bg) \eqdef \bX(\bg)$ has Lipschitz constant $L_h = \sigma$, as
\[
	\|h(\bg) -  h(\bg')\|_2 = \left\| \beps\bmu + \sigma \bg  -  (\beps \bmu + \sigma \bg') \right\|_2 = \sigma \|\bg - \bg'\|_2.
\]
Therefore, since $f_i(\bg) = \tilde{f}_i(h(\bg))$ and $\bg \in E$ if and only if $\bX \in \tilde{E}$, we have that, conditioning on $\calbE'$, the function $\hat{x}_i = f_i(\bg)$ is Lipschitz continuous with Lipschitz constant $L_{f_i} = L_{\tilde{f}_i}L_h = O\left(\frac{\sigma}{\sqrt{np}}\right)$. Since $\bg \sim \Normal(0, \bI_{nd})$, we know that $\hat{x}_i$ is sub-Gaussian with sub-Gaussian constant $\tilde\sigma^2 = L_{f_i}^2 = O\left(\frac{\sigma^2}{np}\right)$. 
\end{proof}

The following lemma will be used for bounding the misclassification probability.
\begin{lemma}[\citet{rigollet2015high}]\label{lem:sub-g-concentration} Let $x_1,\ldots,x_n$ be sub-Gaussian random variables with the same mean and sub-Gaussian parameter $\tilde\sigma^2$. Then,
\[\Ex\left[\max_{i\in [n]}\left(x_i- \Ex[x_i]\right)\right]\le \tilde\sigma\sqrt{2\log n}.\]
Moreover, for any $t>0$
\[\Prx\left[\max_{i\in [n]}\left(x_i-\Ex[x_i]\right)>t\right]\le 2n\exp\left(-\frac{t^2}{2\tilde\sigma^2}\right).\]
\end{lemma}

We bound the probability of misclassification
\[\Prx\left[\max_{i\in C_0}\hat x_i\ge 0\right]\le \Prx\left[\max_{i\in C_0}\hat x_i>t+\Ex[\hat x_i]\right],\] for $t<|\Ex[\hat x_i]|=\frac{\|\bmu\|_2}{2}(1\pm o(1))$. By \cref{lem:hat_x_sub_g}, picking $t=\Theta\left(\sigma\sqrt{\log |C_0|}\right)$ and applying \cref{lem:sub-g-concentration} implies that the above probability is $1/\poly(n)$.

Similarly for class $C_1\cup C_{-1}$ we have that the misclassification probability is 
\begin{align*}
    \Prx\left[\min_{i\in C_1\cup C_{-1}}\hat x_i\le  0\right]
    &=\Prx\left[-\max_{i\in C_1\cup C_{-1}} (-\hat x_i)\le 0\right]
    =\Prx\left[\max_{i\in C_1\cup C_{-1}}(-\hat{ x}_i)\ge 0\right]
    \\
    &\le \Prx\left[\max_{i\in C_1\cup C_{-1}}-\hat{ x_i}>t-\Ex[\hat x_i]\right],
\end{align*}
for $t<\Ex[\hat x_i]$. Picking $t=\Theta\left(\sigma\sqrt{\log |C_1\cup C_{-1}|}\right)$ and applying \cref{lem:sub-g-concentration} and a union bound over the misclassification probabilities of both classes conclude the proof of the corollary.
\end{proof}

Combining \cref{thm:main_GAT} with \cref{lem:Dist_after_conv}, we immediately get \cref{prop:cat} which we restate below.
\begin{corollary} \label{prop:cat_2} Suppose $p,q=\Omega(\log^2 n/n)$ and  $\|\bmu\|\ge \omega\left(\sigma \sqrt{\frac{(p+2q)\log n}{n(p-q)^2}}\right)$. Then, there is a choice of attention architecture $\Psi$ such that, with probability at least $1-o(1)$ over the data $(\bold{X},\bA)\sim\CSBM(n,p,q,\bmu,\sigma^2)$, \ours separates nodes $C_0$ from $C_1\cup C_{-1}$.
\end{corollary}
    \clearpage

\section{Synthetic experiments} \label{app:toy_v2}

In this section, we present the complete results for the synthetic data experiments of \cref{sec:exps-synhetic-data}. First, we describe the parameterization we use for the 1-layer GCN, GAT, and \ours models; then, we verify the behavior of the normalized score function ($\gamma_{ij}$) matches that of the theory presented in \cref{cor:gammas}. In particular, we visualize the average of the following three groups of gammas (\cref{figapp:toy_v2_ansatz}):
\begin{itemize}
    \item Gammas $\gamma_{ij}$ included in $ i,j\in C_0^2\cup C_1^2$. Solid lines.
    \item Gammas $\gamma_{ij}$ included in $i,j\in C_{-1}\times C_1$. Dashed lines.
    \item The rest of gammas. Dotted lines.
\end{itemize}

For completeness, we also include the empirical results that validate \cref{prop:gat} and \cref{prop:cat}, which were discussed already in \cref{sec:exps-synhetic-data}.

\paragraph{Experimental setup.} We assume the following parametrization for the  1-layer GCN, GAT, and \ours:

\begin{equation}
    \vh_i^\prime =   \left(\sum_{j \in N_i} \gamma_{ij} \tilde\bw^T \bX_j \right) - C \cdot \|\bmu\|/2 \;,
    \label{eqapp:ansatz_layer}
\end{equation}
where $N_i$ are the set of neighbors of node $i$, $\bX_j$ are the features of node $j$---obtained from the $\CSBM$ described in \cref{sec:convolved-GAT}, and  $\vh_i^\prime$ are the logits of the prediction of node $i$. Note that for GCN we have $\gamma_{ij}= \frac{1}{|N^*_i|}$. Otherwise, we consider the following parameterization of the score function $\Psi$:
\begin{align}
    &\gamma_{ij} = \frac{\exp{(\Psi(   \vh_i,    \vh_j)) }}{\sum_{k  \in N_i^*} \exp({\Psi( \vh_i,   \vh_k))} } \quad\text{where}\quad \\
    &\Psi(\bh_i,\bh_j) \eqdef\boldr^T\cdot \LeakyRelu\left(\bS \cdot\begin{bmatrix} \tilde \bw^T \bh_i\\
\tilde \bw^T \bh_j\end{bmatrix}+\bb \right)\;.
\label{eqapp:ansatz_score}
\end{align}

For these experiments, we define the parameters $\tilde\bw$, $\bS$, $\bb$ and $\boldr$ as in the proofs in \cref{app:theory}:

\begin{align}
\tilde\bw\eqdef\frac{\bmu}{\|\bmu\|},
\qquad 
\bS\eqdef  
\begin{bmatrix}  
1 & 1  \\
-1 & -1 \\
1 & -1\\
-1 & 1\\
0 & 1\\
1 & 0\\
0 &-1 \\
-1 & 0
\end{bmatrix},
\qquad \bb\eqdef\begin{bmatrix}
-3/2\\
-3/2\\
-3/2\\
-3/2\\
-1/2\\
-1/2\\
-1/2\\
-1/2
\end{bmatrix}\cdot\|\bmu\| \cdot C, \qquad
\boldr\eqdef R \cdot \begin{bmatrix} 2 \\ -2 \\ -2 \\ 2 \\ -1 \\-1 \\-1\\-1 \end{bmatrix},
\label{eqapp:ansatz_params}
\end{align}
where  $R>0$ and  $C>0$ are arbitrary scaling parameters. Both $C$ and $R$  and input to the score function are set different for each of the models, as indicated in \cref{tabapp:ansatz-params}. In particular, we set $R=\frac{7}{\|\bmu\|}$ for both GAT and \ours such that: i) all $\gamma_{ij}$ are distinguishable as we decrease $\|\bmu\|$; and ii) we avoid numerical instabilities in the implementation when computing the exponential of $R \times \|\bmu\| $ in order to obtain $\gamma_{ij}$ (see \cref{cor:gammas}), as the exponential of small or large values leads to under/overflow issues. As for $C$, we set $C=1$ for GAT and $C={(p-q)}/{(p+2q)}$ for \ours such that we counteract the fact that the distance between classes shrink as we increase $q$, see \cref{lem:Dist_after_conv}:

\begin{table*}[!ht]
\centering
\caption{Parameters for the synthetic experiments.}
\label{tabapp:ansatz-params}
\begin{tabular}{lccc} 
\toprule
    Model & $C$ & $R$  & $ \vh_i$ \\ \midrule
    GCN & 0  &  $-$ & $-$ \\
    GAT & 1 & $\frac{7}{\|\bmu\|}$ & $\bX_i$\\ 
    \ours & $\frac{p-q}{p+2q}$ & $\frac{7}{\|\bmu\|}$ &$ \frac{1}{|N^*_i|} \sum_{k \in N_i^*}\bX_k$\\
\bottomrule
\end{tabular}
\end{table*}

Regarding the data model, we set (as described in \cref{sec:exps-synhetic-data}) $n=10000$, $p=0.5$, $\sigma=0.1$, and $d=n/\left(5\log^2(n)\right)$. We set the slope of the LeakyReLU activation to $\beta=1/5$ for the GAT and $\beta=0.01$ for \ours,  such that the proof of \cref{cor:gammas} is valid. 
As described in the main paper, to assess the sensitivity to structural noise, we present the complete results for two sets of experiments. First, we vary the noise level $q$ between \num{0} and \num{0.5}, fixing the mean vector $\bmu$. 
We test two values of~$\|\bmu\|$: the first corresponds to the \textit{easy} regime ($\|\bmu\| = 10\sigma\sqrt{2\log n} \approx 4.3$) where classes are far apart; and the second correspond to the \textit{hard} regime ($\|\bmu\| = \sigma = 0.1$) where the distance between the clusters is small. 
In the second experiment we modify instead the distance between the means, sweeping $\|\bmu\|$ in the range $\left[\sigma/20, 20\sigma\sqrt{2\log n}\right]$ which corresponds to $\left[0.005, 8.58\right]$, and thus covering the transition from the hard setting (small $\|\bmu\|$) to the easy one (large $\|\bmu\|$). In these experiments, we fix $q$ to \num{0.1} (low noise) and \num{0.3} (high noise).

\begin{figure}[!t]
\centering
\begin{subfigure}{.33\textwidth} %
  \centering
  \includegraphics[width=.99\linewidth]{images/toyv2_vary_q_True_0_300_q_R2_legend.pdf}
\end{subfigure}\\
\hfill
\begin{subfigure}{.24\textwidth}  %
  \centering
  \includegraphics[width=.99\linewidth]{images/toyv2_vary_mu_True_0_100_mu_R2_y_0.pdf}
\end{subfigure}
\begin{subfigure}{.24\textwidth} 
  \centering
  \includegraphics[width=.99\linewidth]{images/toyv2_vary_mu_True_4_300_mu_R2_0.pdf}
\end{subfigure}
\begin{subfigure}{.24\textwidth}  
  \centering
  \includegraphics[width=.99\linewidth]{images/toyv2_vary_q_True_0_100_q_R2_0.pdf}
\end{subfigure}
\begin{subfigure}{.24\textwidth} 
  \centering
  \includegraphics[width=.99\linewidth]{images/toyv2_vary_q_True_0_300_q_R2_0.pdf}
\end{subfigure}
\hfill
\begin{subfigure}{.55\textwidth} %
  \centering
  \includegraphics[width=.99\linewidth]{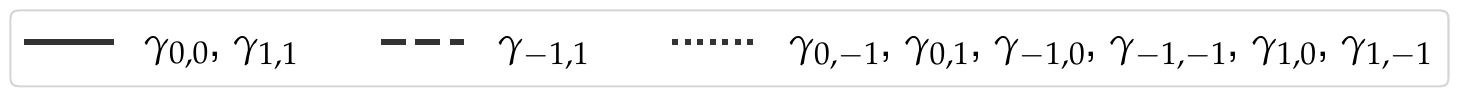}
\end{subfigure}\\
\hfill
\begin{subfigure}{.24\textwidth}  %
  \centering
  \includegraphics[width=.99\linewidth]{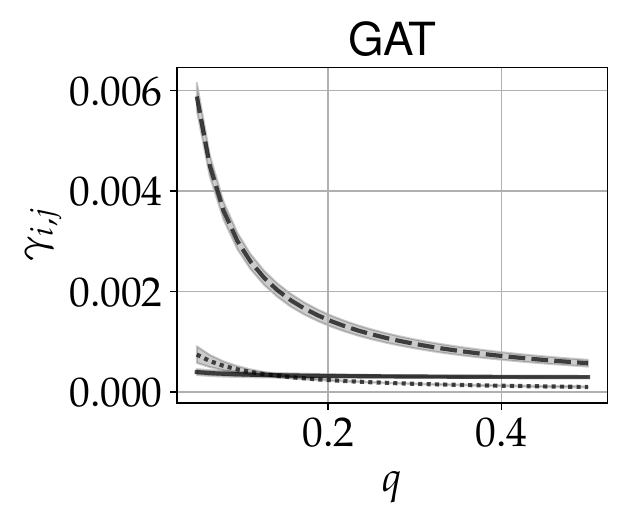}
\end{subfigure}
\begin{subfigure}{.24\textwidth} 
  \centering
  \includegraphics[width=.99\linewidth]{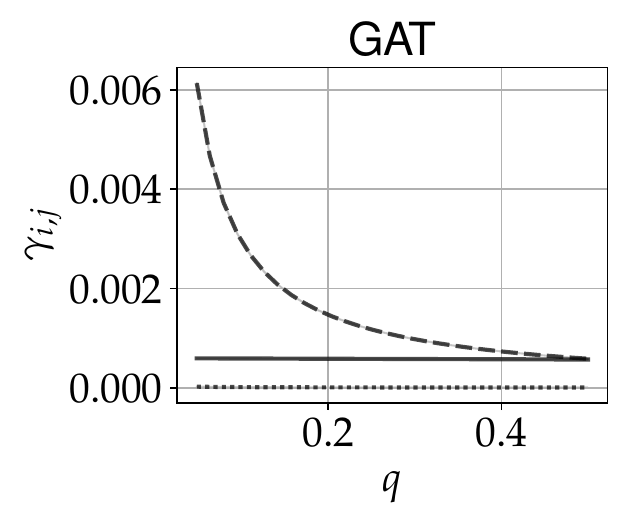}
\end{subfigure}
\begin{subfigure}{.24\textwidth}  
  \centering
  \includegraphics[width=.99\linewidth]{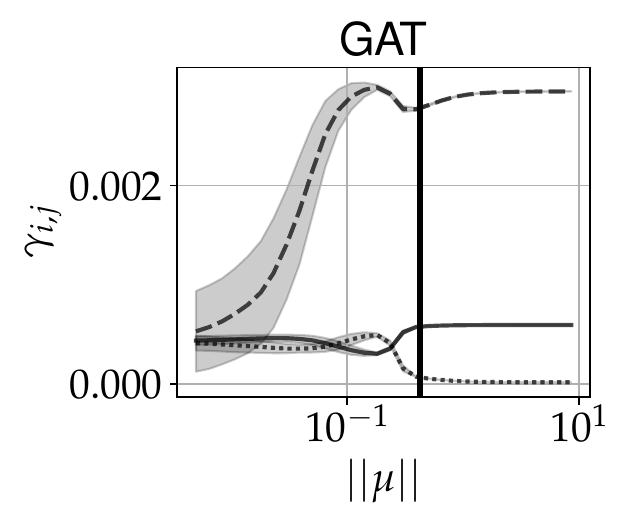}
\end{subfigure}
\begin{subfigure}{.24\textwidth} 
  \centering
  \includegraphics[width=.99\linewidth]{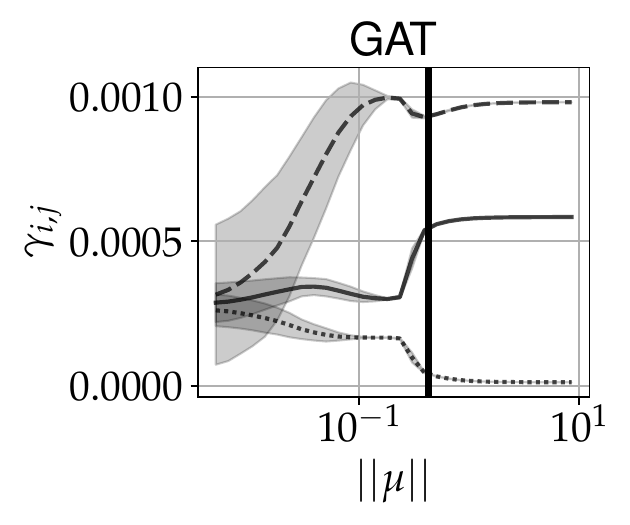}
\end{subfigure}
\hfill
\begin{subfigure}{.24\textwidth}  %
  \centering
  \includegraphics[width=.99\linewidth]{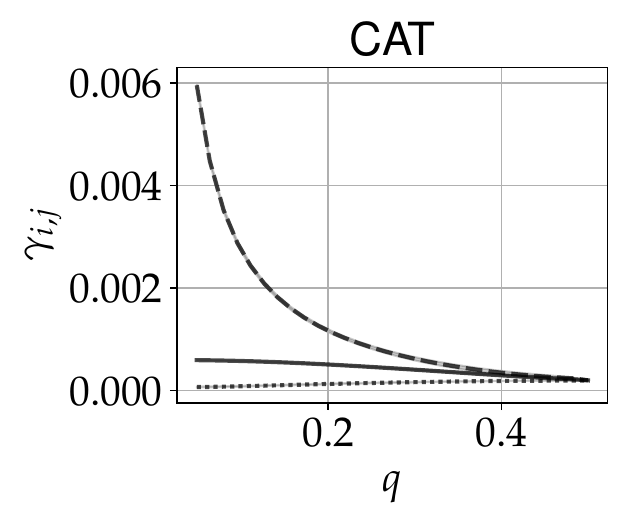}
\end{subfigure}
\begin{subfigure}{.24\textwidth} 
  \centering
  \includegraphics[width=.99\linewidth]{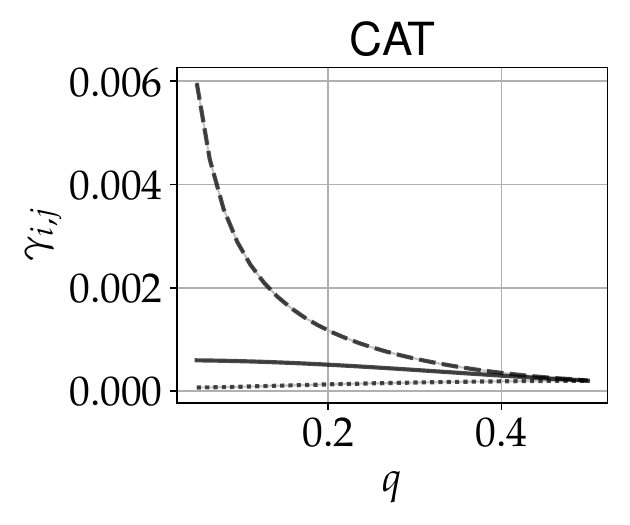}
\end{subfigure}
\begin{subfigure}{.24\textwidth}  
  \centering
  \includegraphics[width=.99\linewidth]{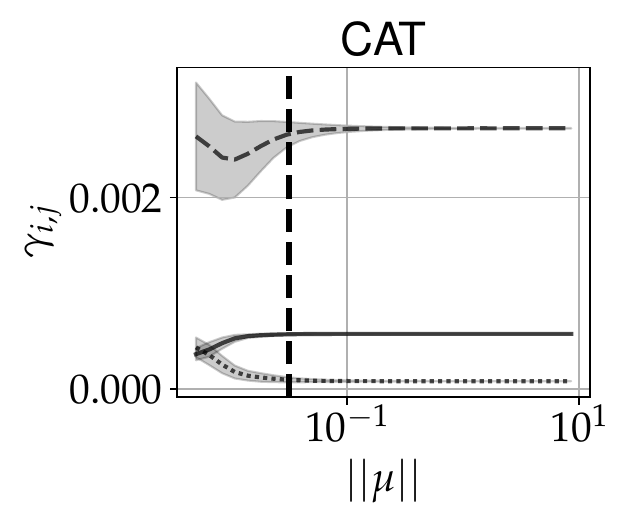}
\end{subfigure}
\begin{subfigure}{.24\textwidth} 
  \centering
  \includegraphics[width=.99\linewidth]{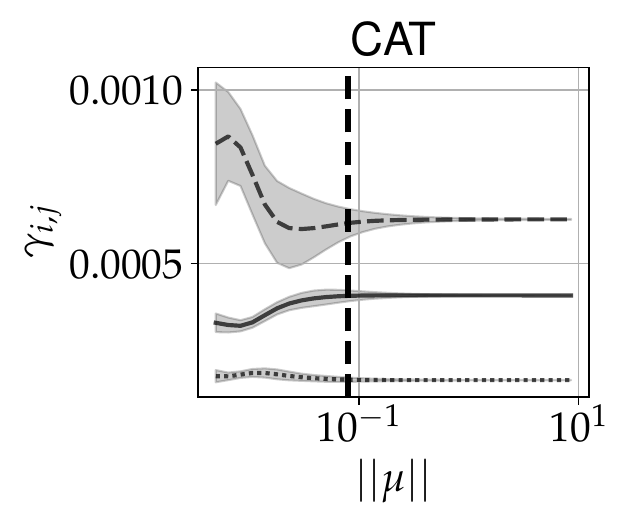}
\end{subfigure}
\hfill
\caption{Synthetic data results. On the top row,  we show the node classification, and in the following two rows we show the $\gamma_{ij}$ values for GAT and \ours respectively. In the two left-most figures, we show how the results vary with the noise level $q$ for $\|\bmu\|=0.1$ and  $\|\bmu\|=4.3$. In the two right-most figures, we show how the results vary with the norm of the means $\|\bmu\|$  for $q=0.1$ and $q=0.3$. We use two vertical lines to present the classification threshold stated in \cref{prop:gat} (solid line) and \cref{prop:cat} (dashed line).}
\label{figapp:toy_v2_ansatz}
\end{figure}

\paragraph{Results} are summarized in \cref{figapp:toy_v2_ansatz}. The top row contains the node classification performance  for each of the models (i.e., \cref{fig:toy_v2_ansatz}), the next two rows contain the $\gamma_{ij}$ values for GAT and \ours respectively.
The two left-most columns of \cref{figapp:toy_v2_ansatz} show the results for the hard and easy regimes, respectively, as we vary the noise level $q$. 
In the hard regime, we observe that GAT is unable to achieve separation for any value of $q$, whereas \ours achieves perfect classification when $q$ is small enough.
The gamma plots help shed some light on this question. For GAT, we observe that the gammas represented with the dotted and solid lines collapse for any value of $q$ (see middle plot), while this does not happen for \ours when the noise level is low (see bottom plot). This exemplifies the advantage of \ours over GAT as stated in \cref{prop:cat}.
When the distance between the means is large enough, we see that GAT achieves perfect results independently of $q$, as stated in \cref{prop:gat}. We also observe that, in this case, the gammas represented with the dotted and solid lines do not collapse for any value of $q$.
In contrast, as we increase $q$, \ours fails to satisfy the condition in \cref{prop:cat}, and therefore achieves inferior performance. 
We  note that the low performance is due to the fact that all  gammas collapse to the same value for large noise levels.

For the  second set of experiments (two right-most columns of \cref{fig:toy_v2_ansatz}), where we fix $q$ and sweep $\|\bmu\|$, 
we observe that, for both values of $q$, there exists a transition in the accuracy of both GAT and \ours as a function of $\|\bmu\|$. 
As shown in the main manuscript,  GAT achieves perfect accuracy when the distance between the means satisfies the condition in \cref{prop:gat} (solid vertical line in \cref{fig:toy_v2_ansatz}). Moreover, we can see the improvement \ours obtains over GAT. Indeed, when $\|\bmu\|$ satisfies the conditions of \cref{prop:cat} (dashed vertical line in \cref{fig:toy_v2_ansatz}), the classification threshold is improved. As we increase $q$, we see that the gap between the two vertical lines decreases, which means that the improvement decreases as $q$ increments, exactly as stated in \cref{prop:cat}. This transition from the hard regime to the easy regime is also observed in the gamma plots: we observe the largest difference in value between the different groups of lambdas for values of $\|\bmu\|$ that satisfy the condition in \cref{prop:gat} (that is to the right of the vertical lines).

\subsection{Other experiments}

\begin{figure}[!t]
\centering
\begin{subfigure}{.5\textwidth} %
  \centering
  \includegraphics[width=.99\linewidth]{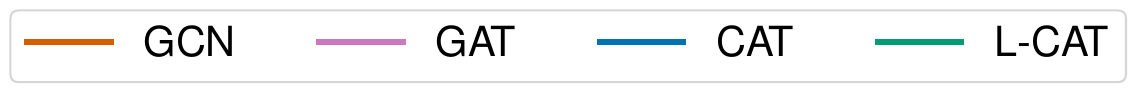}
\end{subfigure}\\
\hfill
\begin{subfigure}{.24\textwidth}  %
  \centering
  \includegraphics[width=.99\linewidth]{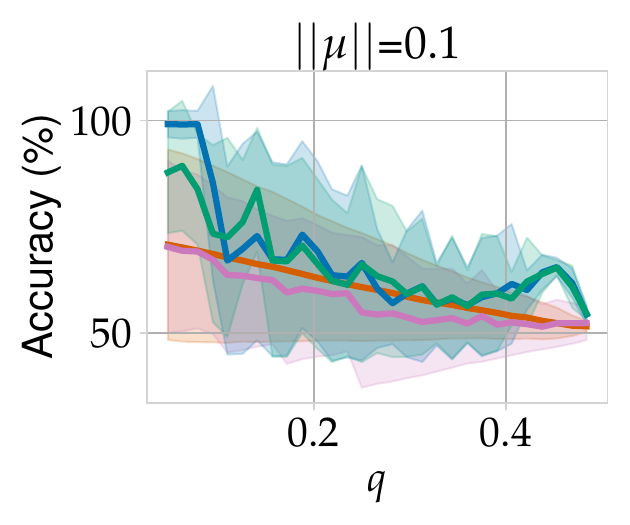}
\end{subfigure}
\begin{subfigure}{.24\textwidth} 
  \centering
  \includegraphics[width=.99\linewidth]{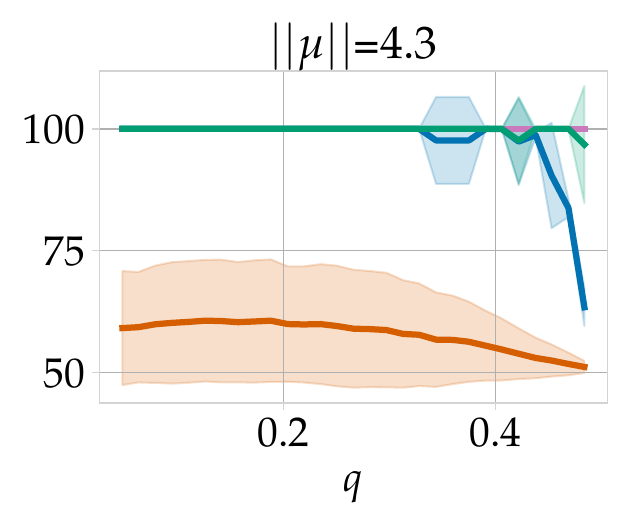}
\end{subfigure}
\begin{subfigure}{.24\textwidth}  
  \centering
  \includegraphics[width=.99\linewidth]{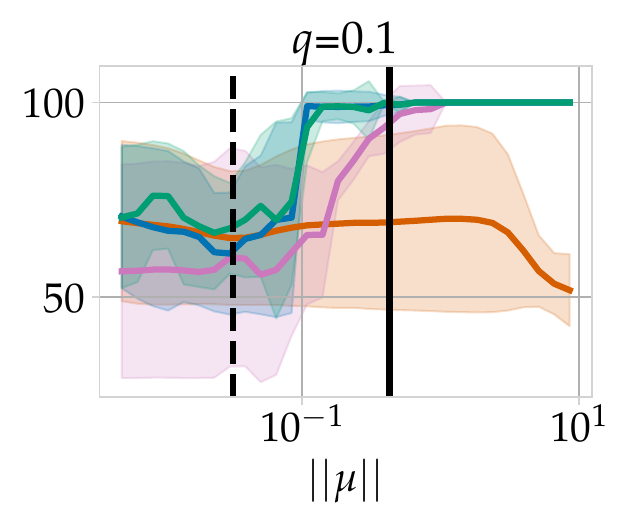}
\end{subfigure}
\begin{subfigure}{.24\textwidth} 
  \centering
  \includegraphics[width=.99\linewidth]{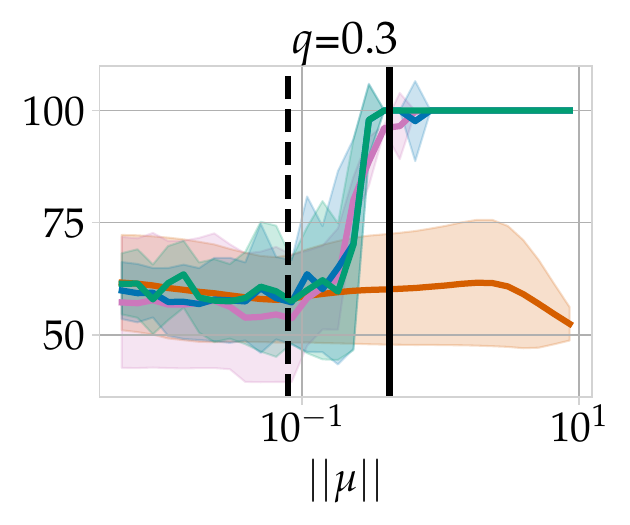}
\end{subfigure}
\hfill
\begin{subfigure}{.2\textwidth} %
  \centering
  \includegraphics[width=.99\linewidth]{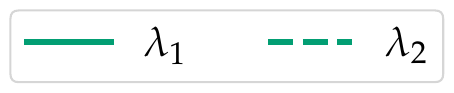}
\end{subfigure}\\
\hfill
\begin{subfigure}{.24\textwidth}  %
  \centering
  \includegraphics[width=.99\linewidth]{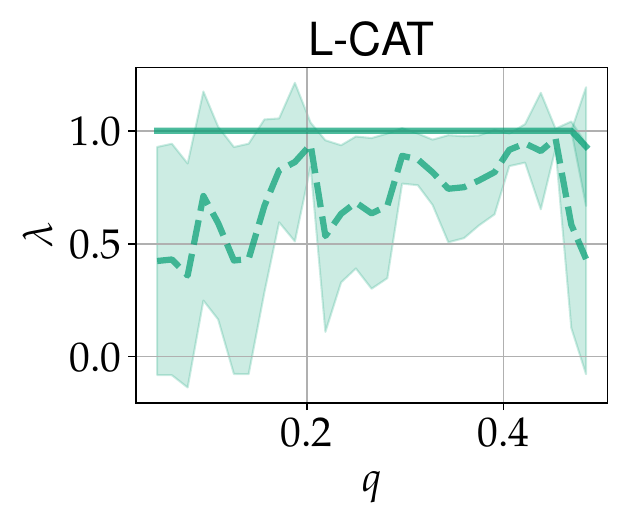}
\end{subfigure}
\begin{subfigure}{.24\textwidth} 
  \centering
  \includegraphics[width=.99\linewidth]{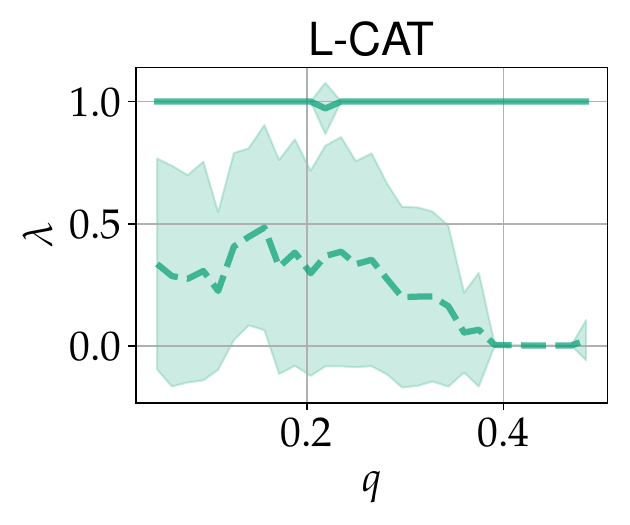}
\end{subfigure}
\begin{subfigure}{.24\textwidth}  
  \centering
  \includegraphics[width=.99\linewidth]{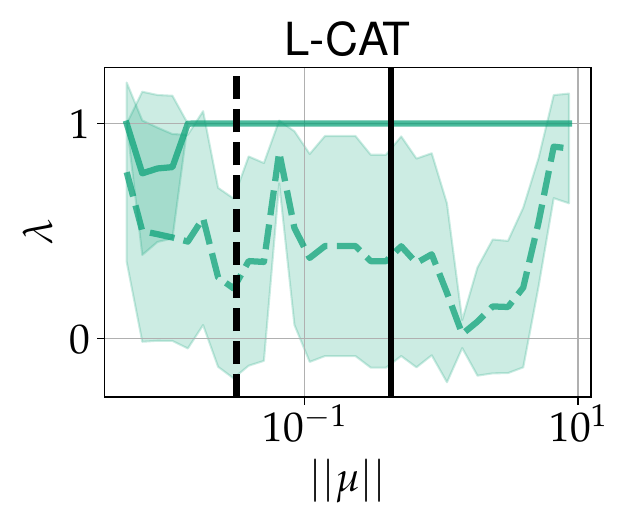}
\end{subfigure}
\begin{subfigure}{.24\textwidth} 
  \centering
  \includegraphics[width=.99\linewidth]{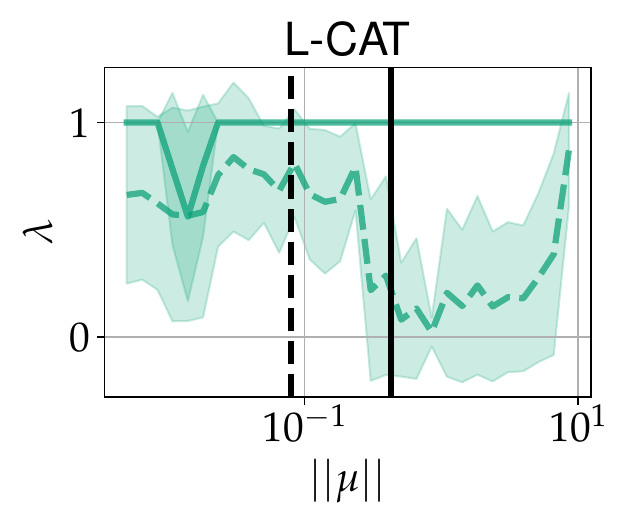}
\end{subfigure}
\hfill
\caption{Synthetic data results learning $C$, $\lambda_1$ and $\lambda_2$. On the top row,  we show the node classification accuracy, and in the bottom row we show the learned values of $\lambda_1$ and $\lambda_2$ for \ours*. In the two left-most figures, we show how the results vary with the noise level $q$ for $\|\bmu\|=0.1$ and  $\|\bmu\|=4.3$. In the two right-most figures, we show how the results vary with the norm of the means $\|\bmu\|$  for $q=0.1$ and $q=0.3$. We use two vertical lines to present the classification threshold stated in \cref{prop:gat} (solid line) and \cref{prop:cat} (dashed line).}
\label{figapp:toy_v2_learn}
\end{figure}

In the following, we extend the results for the synthetic data presented above. In particular, we aim to evaluate if \ours* is able to achieve top performance regardless of the scenario. That is, we want to evaluate if \ours* consistently performs at least as good as the best-performing model. We change the fixed-parameter setting of the previous section and, instead, we evaluate the performance of GCN, GAT, \ours and \ours* when we learn the  model-dependent parameters. 

\paragraph{Experimental setup.} We assume the same parametrization for the  1-layer GCN, GAT and \ours described in \cref{eqapp:ansatz_layer} and  \cref{eqapp:ansatz_score}. For \ours*, we add the parameters $\lambda_1$ and $\lambda_2$, as indicated in \cref{eq:ours-score-lmbda}. 
We fix the parameters shared among the models, that is, $\tilde\bw$, $\bS$, $\bb$, $\boldr$, and $R$, with the values indicated in \cref{eqapp:ansatz_params}. Different from previous experiments,  we now learn $C$ and, for \ours*, we also learn  $\lambda_1$ and $\lambda_2$.
We %
choose to fix part of the parameters (instead of learning them all) %
to keep the problem as similar as possible to the theoretical analysis we provided in \cref{sec:convolved-GAT} and \cref{app:theory}. 
If we instead learn all the parameters, it takes a single dimension of the features to be (close to) linearly separable %
to find a solution that achieves a similar performance regardless of the model, which hinders the analysis. 
This is a consequence of the probabilistic nature of the features. One way of solving this issue would be to make $n$ big enough. Instead, we opt to have a fixed $n$ and reduce the degrees of freedom of the models by fixing the parameters shared across all models.
The rest of the experimental setup matches the one %
from \cref{app:toy_v2}. 
Additionally, we use the Adam optimizer~\citep{kingma-adam} with a learning rate of $0.05$, and we train for $100$.

\paragraph{Results} are summarized in \cref{figapp:toy_v2_learn}. The top row contains the node classification performance for every model, while the bottom row contains the learned values of $\lambda_1$ (solid line) and  $\lambda_2$ (dashed line) with \ours*.
The two left-most columns of \cref{figapp:toy_v2_learn} show the results for the hard and easy regimes, respectively, as we vary the noise level $q$. 
In the hard regime, we see rather noisy results. Still, the behaviour is similar to that of \cref{figapp:toy_v2_ansatz}:  the performance of \ours degrades as we increase $q$. We also observe that, on average, \ours outperforms GAT. In this case, we observe that \ours* achieves similar performance as \ours, which can be explained by inspecting the learned values of lambda in the bottom row. We observe that $\lambda_1=1$ and $\lambda_2 \geq 0.5$ on average for all values of $q$. This indicates that \ours* is closer to \ours than to GAT.
When the distance between the means is large enough (\ie, $\|\bmu\|=4.3$), we see that GAT achieves perfect results independently of $q$ while the performance of \ours deteriorates with large values of $q$, the same trend as in \cref{figapp:toy_v2_ansatz}.
Remarkably, we observe that \ours* also achieves perfect results independently of $q$. If we inspect the lambda values, we first see that $\lambda=1$ for all $q$, thus the interpolation happens between \ours and GAT. Looking at the values of $\lambda_2$, we observe that, for small values of $q$, $\lambda_2$ is pretty noisy, which is expected since any solution achieves perfect performance. Interestingly, we have that $\lambda_2=0$ for large values of $q$, with negligible variance. This indicates that \ours* learns that it must behave like GAT in order to perform well.

For the  second set of experiments (two right-most columns of \cref{figapp:toy_v2_learn}),  we fix $q$ and sweep $\|\bmu\|$ like we did in \cref{figapp:toy_v2_ansatz}. Here, we observe a similar trend: for both values of $q$, there exists a transition in the accuracy of both GAT and \ours as a function of $\|\bmu\|$. Yet once again, we observe that \ours* consistently achieves a similar performance to the best-performing model in every situation.

\section{Dataset description} \label{app:dataset}

We present further details about the datasets used in our experiments, summarized in \cref{tab:datasets-statistics}. All datasets are undirected (or transformed to undirected otherwise) and transductive. 

The upper rows of the table refer to datasets used in \cref{sec:experiments-graphgym} taken from the PyTorch Geometric framework.\footnote{\url{https://pytorch-geometric.readthedocs.io/en/latest/modules/datasets.html}}
The following paragraphs present a short description of such datasets.

\paragraph{Amazon Computers \& Photos} are datasets taken from \citet{shchur2018pitfalls}, in which nodes represent products, and edges indicate that the products are usually bought together. The node features are a Bag of Words (BoW) representation of the product reviews. The task is to predict the category of the products.

\paragraph{GitHub} is a dataset introduced in \citet{rozemberczki2021multi}, in which nodes correspond to  developers, and edges indicate mutual follow relationship. Node features are embeddings extracted from the developer's starred repositories and profile information (e.g., location or employer). The task is to infer whether a node relates to web or machine learning development.

\paragraph{FacebookPagePage} is a dataset introduced in \citet{rozemberczki2021multi}, where nodes are Facebook pages, and edges imply mutual likes between the pages. Nodes features are text embeddings extracted from the pages' description. The task consist on identifying the page's category.

\paragraph{TwitchEN} is a dataset introduced in \citet{rozemberczki2021multi}. Here, nodes correspond to Twitch gamers, and links reveal mutual friendship. Node features are an embedding of {games liked, location, and streaming habits}. The task is to infer if a gamer uses explicit content.

\paragraph{Coauthor Physics \& CS} are datasets introduced in \citet{shchur2018pitfalls}. In this case, nodes represent authors which are connected with an edge if they have co-authored a paper. Node features are BoW representations of the keywords of the author's papers. The task consist on mapping each author to their corresponding field of study.

\paragraph{DBLP} is a dataset introduced in \citet{bojchevski2017deep} that represents a citation network. In this dataset, nodes represent papers and edges correspond to citations.  Node features are BoW representations of the keywords of the papers. The task is to predict the research area of the papers.

\paragraph{PubMed, Cora \& CiteSeer} are citation networks introduced in \citet{yang2016revisiting}. Nodes represent documents, and edges refer to citations between the documents. Node features are BoW representations of the documents. The task is to infer the topic of the documents.

The bottom rows of \cref{tab:datasets-statistics} refer to the datasets from  Open Graph Benchmark (OGB) \citep{hu2020open} \footnote{\url{https://ogb.stanford.edu/docs/nodeprop}} used in \cref{sec:ogb-experiments}. We include a short description of them in the paragraphs below.

\paragraph{ogbn-arxiv} is a citation network of computer science papers in arXiv~\citep{wang2020microsoft}. Nodes represent papers, and directed edges refer to citations among them. Node features are embeddings of the title and abstract of the papers. The task is to predict the research area of the nodes.

\paragraph{ogbn-products} contains a co-purchasing network~\citep{Bhatia16}. Nodes represent products, and links are present whenever two products are bought together. Node features are embeddings of a BoW representation of the product description. The task is to infer the category of the products.

\paragraph{ogbn-mag} is a heterogeneous network formed from a subgraph of the Microsoft Academic Graph (MAG)~\citep{wang2020microsoft}. Nodes can belong to one of these four types: authors, papers, institutions and fields of study. Moreover, directed edges belong to one of the following categories: ``author is affiliated with an institution,'' ``author has written a paper,'' ``paper cites a paper,'' and ``paper belongs to a research area.''
Only nodes that are papers contain node features, which are a text embedding of the document content. The task is to predict the venue of the nodes that are papers.

\paragraph{ogbn-proteins} is a network whose nodes represent proteins and edges indicate different types of associations among them. This dataset does not contain node features. The tasks are to predict multiple protein functions, each of them being a binary classification problem.

\begin{table*}[!htbp]
\centering
\caption{Dataset statistics. On the top part of the table, we show the datasets used in \cref{sec:experiments-graphgym}. On the bottom part of the table, we show the datasets used in \cref{sec:ogb-experiments}.}
\resizebox{\textwidth}{!}{
\begin{tabular}{lcccccccc}
\toprule
Name  &  \#Nodes  & \#Edges  & \thead{Avg.\\degree} & \thead{\#Node\\feats.}  & \thead{\#Edge\\feats.}  & \#Tasks  & Task Type \\ \midrule
AmazonComp. & 13,752  & 491,722  & 35.76  & 767  & -  & 1  & 10-class clf. \\
AmazonPhoto & 7,650  & 238,162  & 31.13  & 745  & -  & 1  & 8-class clf. \\
GitHub & 37,700  & 578,006  & 15.33  & 128  & -  & 1  & Binary clf. \\
FacebookP. & 22,470  & 342,004  & 15.22  & 128  & -  & 1  & 4-class clf. \\
CoauthorPh. & 34,493  & 495,924  & 14.38  & 8415  & -  & 1  & 5-class clf. \\
TwitchEN & 7,126  & 77,774  & 10.91  & 128  & -  & 1  & Binary clf. \\
CoauthorCS & 18,333  & 163,788  & 8.93  & 6805  & -  & 1  & 15-class clf. \\
DBLP & 17,716  & 105,734  & 5.97  & 1639  & -  & 1  & 4-class clf. \\
PubMed & 19,717  & 88,648  & 4.50  & 500  & -  & 1  & 3-class clf. \\
Cora & 2,708  & 10,556  & 3.90  & 1433  & -  & 1  & 7-class clf. \\
CiteSeer & 3,327  & 9,104  & 2.74  & 3703  & -  & 1  & 6-class clf. \\
  \midrule 
 ogbn-arxiv & 169,343  & 1,166,243  & 6.89  & 128  & -  & 1  & 40-class clf. \\
 ogbn-products & 2,449,029  & 123,718,280  & 50.52  & 100  & -  & 1  & 47-class clf. \\
ogbn-mag & 1,939,743  & 21,111,007  & 18.61 & 128  & 4  & 1  & 349-class clf. \\
 ogbn-proteins & 132,534  & 79,122,504  & 597.00  & -  & 8  & 112  & Multi-task \\
\bottomrule
\end{tabular}
\label{tab:datasets-statistics}
}
\end{table*}

\section{Real data experiments} \label{app:extra-results-node}  %

\subsection{Experimental details} \label{app:extra-results-node-details}

\paragraph{Computational resources.} We used CPU cores to run this set of experiments. In particular, for each trial, we used 2 CPU cores and up to \SI{16}{\giga\byte} of memory. We ran the experiments in parallel using a shared cluster with $10000$ CPU cores approximately.

\paragraph{General experimental setup.} 
As mentioned in \cref{sec:experiments-graphgym}, we repeat all experiments $10$ times, which correspond to $10$ different random initialization of the parameters of the GNNs. 
In all cases, we choose the model parameters with the best validation performance during training. In order to run the experiments and collect the results, we used the GraphGym framework~\citep{you2020design}, which includes the data processing and loading of the datasets, as well as the evaluation and collection of the results. We split the datasets in \SI{70}{\percent} training, \SI{15}{\percent} validation, and \SI{15}{\percent} test.

We cross-validate the number of message-passing layers in the network ($2,3,4$), as well as the learning rate ($[0.01, 0.005]$). Then, we report the results of the best validation error among the $4$ possible combinations. However, in practice we found the best performance always to use $4$ message-passing layers, and thus the only difference in configuration lies in the learning rate.

We use residual connections between the GNN layers, $4$ heads in the attention models, and the Parametric ReLU (PReLU)~\citep{he2015delving} as the nonlinear activation function. We do not use batch normalization~\citep{ioffe2015batch}, nor dropout~\citep{srivastava14a-dropout}. We use the Adam optimizer~\citep{kingma-adam} with $\beta = (0.9, 0.999)$, and an exponential learning-rate scheduler with $\gamma=0.998$. We train all the models for $2500$ epochs. Importantly, we do not use weight decay, since this will bias the solution towards $\lambda_1=0$ and $\lambda_2=1$.

We use the Pytorch Geometric~\citep{fey-pytorch-geometric} implementation of \ours* for all experiments, switching between models by properly by setting $\lambda_1$ and $\lambda_2$. 
We parametrize $\lambda_1$ and $\lambda_2$ as free-parameters in log-space that pass through a sigmoid function---\ie, $\texttt{sigmoid}(10^x)$---such that they are constrained to the unit interval, and they are learned quickly.

\subsection{Additional results}\label{app:extra-results-node-extra-results}

\Cref{tabapp:performance_node_small1} shows the results presented in the main paper (with the addition of a dense feed-forward network), while  \cref{tabapp:performance_node_small2} presents the results for the remaining datasets, with smaller average degree. 

If we focus on \cref{tabapp:performance_node_small2}, we observe that all models perform equally well, yet in a few cases \ours and \ours* are significantly better than the baselines---\eg, \ours*[v2] in \spname{CoauthorCS}, or \ours* in \spname{Cora}. 
Following a similar discussion as the one presented in the main paper, these results indicates that \ours* achieves similar or better performance than baseline models and thus, should be the preferred architecture.
 
 \paragraph{Competitive performance without the graph.} We also include in \cref{tabapp:performance_node_small1,tabapp:performance_node_small2} the performance of a feed-forward network, referred to as Dense (first row). 
 Note that the only data available to this model are the  node features, and thus no graph information is provided. 
 Therefore,  we should expect a significant drop in performance, which indeed  happens for some datasets such as \spname{Amazon Computers} ($\approx7\%$ drop), \spname{FacebookPagePage}  ($\approx 20\%$ drop), \spname{DBLP}  ($\approx 9\%$ drop) and \spname{Cora}  ($\approx 14\%$ drop). Still, we found that for other commonly used datasets the performance is similar, \eg, \spname{Coauthor Physics} and \spname{PubMed}; or \emph{it is even better} \spname{CoauthorCS}. 
These results manifest the importance of a proper benchmarking, and of carefully considering the datasets used to evaluate GNN models.

\begin{table*}[!ht]
\centering
\caption{Test accuracy (\%) of the considered convolution and attention models for different datasets (sorted by their average node degree), and averaged over ten runs. Bold numbers are statistically different to their baseline model ($\alpha = 0.05$). Best average performance is underlined.}
\label{tabapp:performance_node_small1}
\resizebox{\textwidth}{!}{
\begin{tabular}{lcccccc}
    \toprule
    Dataset & \thead{\spname{Amazon}\\\spname{Computers}} & \thead{\spname{Amazon}\\\spname{Photo}} & \spname{GitHub} & \thead{\spname{Facebook}\\\spname{PagePage}} & \thead{\spname{Coauthor}\\\spname{Physics}} & \spname{TwitchEN} \\
    Avg. Deg. & 35.76 & 31.13 & 15.33 & 15.22 & 14.38 & 10.91 \\ 
    \midrule
    Dense & 83.73 $\pm$ 0.34 & 91.74 $\pm$ 0.46 & 81.21 $\pm$ 0.30 & 75.89 $\pm$ 0.66 & 95.41 $\pm$ 0.14 & 56.26 $\pm$ 1.74 \\
    \midrule
    GCN & \best{90.59 $\pm$ 0.36} & \best{95.13 $\pm$ 0.57} & 84.13 $\pm$ 0.44 & 94.76 $\pm$ 0.19 & 96.36 $\pm$ 0.10 & 57.83 $\pm$ 1.13 \\
    \midrule
    GAT & 89.59 $\pm$ 0.61 & 94.02 $\pm$ 0.66 & 83.31 $\pm$ 0.18 & 94.16 $\pm$ 0.48 & 96.36 $\pm$ 0.10 & 57.59 $\pm$ 1.20 \\
    \ours & \better{90.58 $\pm$ 0.40} & \better{94.77 $\pm$ 0.47} & \better{84.11 $\pm$ 0.66} & \better{94.71 $\pm$ 0.30} & \best{96.40 $\pm$ 0.10} & \best{58.09 $\pm$ 1.61} \\
    \ours* & \better{90.34 $\pm$ 0.47} & \better{94.93 $\pm$ 0.37} & \better{84.05 $\pm$ 0.70} & \best{\better{94.81 $\pm$ 0.25}} & 96.35 $\pm$ 0.10 & 57.88 $\pm$ 2.07 \\
    \midrule
    GATv2 & 89.49 $\pm$ 0.53 & 93.47 $\pm$ 0.62 & 82.92 $\pm$ 0.45 & 93.44 $\pm$ 0.30 & 96.24 $\pm$ 0.19 & 57.70 $\pm$ 1.17 \\
    \ours[v2] & \better{90.44 $\pm$ 0.46} & \better{94.81 $\pm$ 0.55} & \better{84.10 $\pm$ 0.88} & \better{94.27 $\pm$ 0.31} & 96.34 $\pm$ 0.12 & 57.99 $\pm$ 2.02 \\
    \ours*[v2] & \better{90.33 $\pm$ 0.44} & \better{94.79 $\pm$ 0.61} & \best{\better{84.31 $\pm$ 0.59}} & \better{94.44 $\pm$ 0.39} & 96.29 $\pm$ 0.13 & 57.89 $\pm$ 1.53 \\
    \bottomrule
\end{tabular}
}
\end{table*}

\begin{table*}[!ht]
\centering
\caption{Test accuracy (\%) of the considered convolution and attention models for different datasets (sorted by their average node degree), and averaged over ten runs. Bold numbers are statistically different to their baseline model ($\alpha = 0.05$). Best average performance is underlined.}
\label{tabapp:performance_node_small2}
\resizebox{\textwidth}{!}{
\begin{tabular}{lccccccccccc}
\toprule
Dataset & \spname{CoauthorCS} & \spname{DBLP} & \spname{PubMed} & \spname{Cora} & \spname{CiteSeer}  \\
Avg. Deg. & 8.93 &  5.97 &  4.5 & 3.9 &  2.74 \\ \midrule
Dense & \best{94.88 $\pm$ 0.21} & 75.46 $\pm$ 0.27 & 88.13 $\pm$ 0.33 & 72.75 $\pm$ 1.72 & 73.02 $\pm$ 1.01 \\
\midrule
GCN & 93.85 $\pm$ 0.23 & 84.18 $\pm$ 0.40 & 88.50 $\pm$ 0.18 & \best{86.68 $\pm$ 0.78} & \best{75.76 $\pm$ 1.09} \\
\midrule
GAT & 93.80 $\pm$ 0.38 & 84.15 $\pm$ 0.39 & \best{88.62 $\pm$ 0.18} & 85.95 $\pm$ 0.95 & 75.40 $\pm$ 1.43 \\
\ours & 93.70 $\pm$ 0.31 & 84.10 $\pm$ 0.29 & 88.58 $\pm$ 0.25 & 85.85 $\pm$ 0.79 & 75.64 $\pm$ 0.91 \\
\ours* & 93.65 $\pm$ 0.23 & 84.13 $\pm$ 0.26 & 88.45 $\pm$ 0.32 & \better{86.66 $\pm$ 0.87} & 75.04 $\pm$ 1.12 \\
\midrule
GATv2 & 93.19 $\pm$ 0.64 & \best{84.33 $\pm$ 0.18} & 88.52 $\pm$ 0.27 & 85.65 $\pm$ 1.01 & 75.14 $\pm$ 1.20 \\
\ours[v2] & 93.51 $\pm$ 0.34 & \worse{84.15 $\pm$ 0.41} & 88.54 $\pm$ 0.29 & 85.50 $\pm$ 0.94 & 74.68 $\pm$ 1.30 \\
\ours*[v2]   & \better{93.65 $\pm$ 0.20} & 84.31 $\pm$ 0.31 & 88.48 $\pm$ 0.24 & 85.75 $\pm$ 0.72 & 75.04 $\pm$ 1.30 \\
\bottomrule
\end{tabular}
}
\end{table*}

\section{Open Graph Benchmark experiments} \label{app:ogb}

\subsection{Experimental details} \label{app:ogb-details}

\paragraph{Computational resources.} For this set of experiments, we had at our disposal a set of $16$ Tesla V100-SXM GPUs with $160$ CPU cores, shared among the rest of the department.

\paragraph{Statistical significance.} For each \ours and \ours* model, we highlight significant improvements according to a two-sided paired t-test ($\alpha = 0.05$), with respect to its corresponding baseline model. For example, for \ours*[v2] with $8$ heads we perform the test with respect to GATv2 with 8 heads.

\paragraph{General experimental setup.} 
As mentioned in \cref{sec:ogb-experiments}, we repeat all experiments with OGB datasets $5$ times. 
In all cases, we choose the model parameters with the best validation performance during training. Moreover, when we show the results without specifying the number of heads, we take the model with the best validation error among the two models with $1$ and $8$ heads.

We use the same implementation of \ours* for all experiments, switching between models by properly setting $\lambda_1$ and $\lambda_2$. 
Experiments on \spname{arxiv}, \spname{mag}, \spname{products} use a version of \ours* implemented in Pytorch Geometric~\citep{fey-pytorch-geometric}. Experiments on \spname{proteins} use a version of \ours* implemented in DGL~\citep{wang2019dgl}.
We parametrize $\lambda_1$ and $\lambda_2$ as free-parameters in log-space that pass through a sigmoid function---\ie, $\texttt{sigmoid}(10^x)$---such that they are constrained to the unit interval, and they are learned quickly.

\paragraph{ArXiv.} 
As described in \cref{sec:ogb-experiments}, we use the example code from the OGB framework~\citep{hu2020open}.
The network is composed of $3$ GNN layers with a hidden size of $128$. We use batch normalization~\citep{ioffe2015batch} and a dropout~\citep{srivastava14a-dropout} of $0.5$ between the GNN layers, and Adam~\citep{kingma-adam} with a learning rate of $0.01$. We use the ReLU as activation function.
For the initial experiments, we train for \num{1500} epochs, while we train for \num{500} epochs for the noise experiments in~\cref{subsec:exp-noise}.
This is justified given the convergence plots in \cref{fig:arxiv-training}.

\paragraph{MAG.}
We adapted the official code from~\citep{brody2021attentive}. 
The network is composed of $2$ layers with $128$ hidden channels.
This time, we use layer normalization~\citep{ba2016layernorm} and a dropout of $0.5$ between the layers. Again, we use ReLU as the activation function, and add residual connections to the network.
As with \spname{arxiv}, we use Adam~\citep{kingma-adam} with learning rate $0.01$. We set a batch size of $20000$ and train for $100$ epochs.

\paragraph{Products.} We use the same setup as \citep{brody2021attentive}, with a network of $3$ GNN layers and $128$ hidden dimensions. We apply residual connections once again, with a dropout~\citep{srivastava14a-dropout} of $0.5$ between layers. This time, we use ELU as the activation function. The batch size is set to $256$. Adam~\citep{kingma-adam} is again the optimizer in use, this time with a learning rate of $0.001$. We train for $100$ epochs, although we apply early stopping whenever the validation accuracy stops increasing for more than $10$ epochs. Note the training split of this dataset only contains \SI{8}{\percent} of the data.

\paragraph{Proteins.} We follow once more the setup of \citep{brody2021attentive}. The network we use has $6$ GNN layers of hidden size $64$. Dropout~\citep{srivastava14a-dropout} is set to $0.25$ between layers, with an input dropout of $0.1$.
At the beginning of the network, we place a linear layer followed by a ReLU activation to encode the nodes, and a linear layer at the end of the network to predict the class. Moreover, we use batch normalization~\citep{ioffe2015batch} between layers and ReLU as the activation function.
We train the model for $1200$ epochs at most, with early stopping after not improving for $10$ epochs.

\subsection{Additional results}\label{app:ogb-extra-results}

We show in \cref{tab:results-arxiv-full,tab:results-mag-full,tab:results-products-full} the results of the main paper for the \spname{arxiv}, \spname{mag}, \spname{products} datasets, respectively, without selecting the best configuration for each type of model.
That is, we show the results for both number of heads.
Note that we already show the full table of results for the \spname{protein} datasets in the main paper (\cref{tab:results-proteins-init}).
All the trends discussed in the main paper hold.

\begin{table*}[!hbtp]
\centering
\caption{Test accuracy on the \spname{arxiv} dataset for attention models using \num{1} head and \num{8} heads.}
\label{tab:results-arxiv-full}

\resizebox{\textwidth}{!}
{
\begin{tabular}{lccccccc}
\toprule
& GCN & GAT &  \ours & \ours* & GATv2 & \ours[v2] & \ours*[v2] \\
\midrule
1h & 71.58 $\pm$ 0.19 & 71.58 $\pm$ 0.15 & \best{\better{72.04 $\pm$ 0.20}} & \better{72.00 $\pm$ 0.11} & 71.70 $\pm$ 0.14 & \better{72.02 $\pm$ 0.08} & 71.96 $\pm$ 0.21 \\
8h & $\minusmark$ & 71.63 $\pm$ 0.11 & \best{\better{72.14 $\pm$ 0.20}} & \better{71.98 $\pm$ 0.08} & 71.72 $\pm$ 0.24 & 71.76 $\pm$ 0.14 & 71.91 $\pm$ 0.16 \\
\bottomrule
\end{tabular}
}
\end{table*}

\begin{table*}[!hbtp]
\centering
\caption{Test accuracy on the \spname{mag} dataset for attention models using \num{1} head and \num{8} heads.}
\label{tab:results-mag-full}

\resizebox{\textwidth}{!}
{
\begin{tabular}{lccccccc}
\toprule
& GCN & GAT &  \ours & \ours* & GATv2 & \ours[v2] & \ours*[v2] \\
\midrule
1h & \best{32.77 $\pm$ 0.36} & 32.35 $\pm$ 0.24 & 31.98 $\pm$ 0.46 & 32.47 $\pm$ 0.38 & 32.76 $\pm$ 0.18 & \worse{ 32.43 $\pm$ 0.22 } & 32.68 $\pm$ 0.50 \\
8h & $\minusmark$ & 32.15 $\pm$ 0.31 & \worse{ 31.58 $\pm$ 0.22 } & 32.49 $\pm$ 0.21 & \best{32.85 $\pm$ 0.21} & \worse{ 32.34 $\pm$ 0.18 } & \worse{ 32.38 $\pm$ 0.28 } \\
\bottomrule
\end{tabular}
}
\end{table*}

\begin{table*}[!hbtp]
\centering
\caption{Test accuracy on the \spname{products} dataset for attention models using \num{1} head and \num{8} heads.}
\label{tab:results-products-full}

\resizebox{\textwidth}{!}
{
\begin{tabular}{lccccccc}
\toprule
& GCN & GAT &  \ours & \ours* & GATv2 & \ours[v2] & \ours*[v2] \\
\midrule
1h & 74.12 $\pm$ 1.20 & \best{78.53 $\pm$ 0.91} & \worse{77.38 $\pm$ 0.36} & 77.19 $\pm$ 1.11 & 73.81 $\pm$ 0.39 & 74.81 $\pm$ 1.12 & \better{76.37 $\pm$ 0.92} \\
8h & $\minusmark$ & \best{78.23 $\pm$ 0.25} & \worse{76.63 $\pm$ 1.15} & \worse{76.56 $\pm$ 0.45} & 76.40 $\pm$ 0.71 & 75.20 $\pm$ 0.92 & \worse{74.70 $\pm$ 0.28} \\
\bottomrule
\end{tabular}
}
\end{table*}

\paragraph{Extrapolation ablation study.}
Due to page constraints, these results were not added to the main paper.
Here, we study two questions. First, how important are $\lambda_1$ and $\lambda_2$ in the formulation of \ours* (\cref{eq:ours-score-lmbda})? 
For the sake of completeness, the second question we attempt to answer here is whether we can obtain similar performance by just interpolating between GCN and GAT (fixing $\lambda_2 = 0$)?
Note that we theoretically showed in \cref{sec:convolved-GAT,sec:exps-synhetic-data} that \ours fills up a gap between GCN and GAT, making it preferable in certain settings.

To this end, we repeat the experiments for network-initialization robustness in \cref{sec:exp-robustness-init}, since they showed to be the best ones to tell apart the performance across models.
We include three additional models: GCN-GAT, which interpolates between GCN and GAT (or GATv2) by learning $\lambda_1$ and fixing $\lambda_2 = 0$; \ours-$\lambda_1$ which interpolates between GCN and \ours by learning $\lambda_1$ and fixing $\lambda_2 = 1$; and \ours-$\lambda_2$, which interpolates between GAT and \ours by learning $\lambda_2$ and fixing $\lambda_1 = 1$.

Results using GAT and shown in \cref{tab:results-proteins-gat-full}, and using GATv2 in \cref{tab:results-proteins-gatv2-full}.
We can observe that GCN-GAT obtains results in between GCN and GAT for all settings, despite being able to interpolate between both layers in each of the six layers of the network.
Regarding learning $\lambda_1$ and $\lambda_2$, we can observe that there is a clear difference between learning boths (\ours*), and learning a single one.
For both attention models, \ours-$\lambda_1$ obtains better results than \ours-$\lambda_2$ in all settings, but \spname{uniform} with $8$ heads.
Still, the results of both variants are substantially worse than those of \ours* in all cases, \emph{demonstrating the importance of learning to interpolate between the three layer types}.

\begin{table*}[!hbtp]
\centering
\caption{Test accuracy on the \spname{proteins} dataset for GCN~\citep{kipf2016semi} and GAT~\citep{velivckovic2017graph} attention models using two network initializations, and two numbers of heads (\num{1} and \num{8}).}
\label{tab:results-proteins-gat-full}

\resizebox{\textwidth}{!}
{
\begin{tabular}{lccccccc}
\toprule
& GCN & GCN-GAT & GAT &  \ours & \ours* & \ours-${\lambda_1}$ & \ours-${\lambda_2}$ \\
\midrule
& \multicolumn{7}{c}{\spname{uniform} initialization} \\ \midrule
1h & 61.08 $\pm$ 2.86 & \better{70.44 $\pm$ 1.56 } & 59.73 $\pm$ 4.04 & \better{74.19 $\pm$ 0.72 } & \best{\better{77.77 $\pm$ 1.44 }} & \better{71.97 $\pm$ 3.78 } & \better{73.55 $\pm$ 1.36 } \\
8h & $\minusmark$ & \worse{ 68.51 $\pm$ 0.91 } & 72.23 $\pm$ 3.20 & 73.60 $\pm$ 1.27 & \best{\better{78.85 $\pm$ 1.76 }} & \better{76.43 $\pm$ 2.47 } & 72.76 $\pm$ 2.79 \\ \midrule
& \multicolumn{7}{c}{\spname{normal} initialization} \\ \midrule
1h & \best{80.10 $\pm$ 0.61} & 66.51 $\pm$ 3.23 & 66.38 $\pm$ 7.76 & 73.26 $\pm$ 1.84 & \better{78.06 $\pm$ 1.40 } & \better{76.77 $\pm$ 1.91 } & 73.39 $\pm$ 1.25 \\
8h & $\minusmark$ & \worse{ 69.93 $\pm$ 1.93 } & 79.08 $\pm$ 1.64 & \worse{ 74.67 $\pm$ 1.29 } & \best{79.63 $\pm$ 0.79} & 78.86 $\pm$ 1.07 & \worse{ 73.32 $\pm$ 1.15 } \\
\bottomrule
\end{tabular}
}
\end{table*}

\begin{table*}[!hbtp]
\centering
\caption{Test accuracy on the \spname{proteins} dataset for GCN~\citep{kipf2016semi} and GATv2~\citep{brody2021attentive} attention models using two network initializations, and two numbers of heads (\num{1} and \num{8}).}
\label{tab:results-proteins-gatv2-full}

\resizebox{\textwidth}{!}
{
\begin{tabular}{lccccccc}
\toprule
& GCN & GCN-GATv2 & GATv2 &  \ours[v2] & \ours*[v2] & \ours[v2]-${\lambda_1}$ & \ours[v2]-${\lambda_2}$ \\
\midrule
& \multicolumn{7}{c}{\spname{uniform} initialization} \\ \midrule
1h & 61.08 $\pm$ 2.86 & \better{69.69 $\pm$ 1.59 } & 59.85 $\pm$ 3.05 & \better{64.32 $\pm$ 2.61 } & \best{\better{79.08 $\pm$ 1.06 }} & 63.24 $\pm$ 1.55 & \better{73.41 $\pm$ 0.34 } \\
8h & $\minusmark$ & \worse{ 69.94 $\pm$ 1.62 } & 75.21 $\pm$ 1.80 & 74.16 $\pm$ 1.45 & \best{\better{78.77 $\pm$ 1.09 }} & \better{77.61 $\pm$ 1.32 } & 73.96 $\pm$ 1.27 \\ \midrule
& \multicolumn{7}{c}{\spname{normal} initialization} \\ \midrule
1h & \best{80.10 $\pm$ 0.61} & 68.54 $\pm$ 1.63 & 69.13 $\pm$ 9.48 & 74.33 $\pm$ 1.06 & \better{79.07 $\pm$ 1.09 } & 78.41 $\pm$ 0.93 & 74.07 $\pm$ 1.17 \\
8h & $\minusmark$ & \worse{ 68.71 $\pm$ 1.96 } & 78.65 $\pm$ 1.61 & \worse{ 73.40 $\pm$ 0.62 } & \best{79.30 $\pm$ 0.55} & 78.76 $\pm$ 1.41 & \worse{ 73.22 $\pm$ 0.77 } \\
\bottomrule
\end{tabular}
}
\end{table*}

	\clearpage

\section{Extending L-CAT to other GNN models} \label{app:extra-results-baselines}  %

Due to their simplicity and popularity, in the manuscript we focus on the simplest form of GCNs, as described in \cref{sec:preliminaries}.
However, we consider important to remark that L-CAT can be effortless extended to a large range of existing GNN models.
A more general formulation of a message-passing GNN layer than the one given in \cref{eq:gnn} is the following:
\begin{equation}
    \tilde\vh_i = f(\vh_i^\prime) \quad \text{where} \quad \vh_i^\prime \eqdef \bigoplus_{j\in N_i^*} \widehat\gamma_{ij} M\left(\vh_i, \vh_j; \theta_M \right) \;, 
    \label{eqapp:gnn}
\end{equation}
where $\widehat\gamma_{ij}$ is a scalar value, $\bigoplus$ refers to any permutation invariant operation---\eg, sum, mean, maximum, or minimum operations---and $M$ is the message operator, which can be parameterized, and produces a message based on the sender and receiver representations.
This formulation comprises most GNN architectures present in the current literature. 
For example:

\begin{itemize}
	\item GCN~\citep{kipf2016semi}: $\vh_i^\prime =  \sum_{j\in N_i^*} \widehat\gamma_{ij} \mW_v  \vh_j $ where  $\widehat\gamma_{ij} = \frac{1}{|N_i^*|}$ as consider in the main paper, or $\widehat\gamma_{ij} = \frac{1}{\sqrt{d_j d_i}}$, where $d_i$ is the number of neighbors of node $i$ (including self-loops), if we consider the symmetric normalized adjacency matrix instead.
    \item GIN~\citep{xu2018powerful}: $\vh_i^\prime = (1 +  \epsilon) \widehat\gamma_{ii}\vh_i + \sum_{j\in N_i} \widehat\gamma_{ij} \vh_j $ with $\widehat\gamma_{ij} = 1$.
    \item PNA~\citep{corso2020principal}: $\vh_i^\prime = \bigoplus_{j\in N_i^*} \widehat\gamma_{ij} M \left(\vh_i, \vh_j; \theta_M \right)$ where $\widehat\gamma_{ij} = 1$, $M$ is an multi-layer perceptron,  and $\bigoplus$ is a set of permutation invariant operations, e.g., $\bigoplus = [\mu, \sigma, \text{max}, \text{min}]$.
    \item GCNII~\citep{chen2020simple}: $\vh_i^\prime =   \left( \alpha \vh_i^0 +  (1-\alpha) \sum_{j\in N_i^*} \widehat\gamma_{ij} \vh_j \right) \left( (1- \beta) \mI  + \beta   \mW_v  \right)  $ where, just as in the GCN case, $\widehat\gamma_{ij} = \frac{1}{\sqrt{d_j d_i}}$, and $0\leq \alpha,\beta \leq 1$.
\end{itemize}

In all the models above, the values $\widehat\gamma_{ij}$ are taken from the adjacency matrix $A$ (whose entries are $1$ if there exists an edge between nodes $i$ and $j$, and $0$ otherwise), or a matrix derived from it, \eg, the symmetric normalized adjacency matrix.

Note that the attention coefficients $\gamma_{ij}$ defined in \cref{eq:gat} can be understood as an attention-equivalent of the adjacency matrix. Indeed, by defining $A^{att}$ as a matrix whose entries are $|N_i^*|\gamma_{ij}$, one can obtain a row-stochastic matrix that can substitute the adjacency matrix of any GNN model.
This technique to generalize attention models to GNN variants more complex than a GCN has been successfully applied in prior literature~\citep{wang2021bag}.

With this new interpretation of attention-models, the interpolation performed by L-CAT (see \cref{eq:ours-score-lmbda}) can similarly be re-interpreted.
Indeed, L-CAT learns to interpolate between the adjacency matrix $A$, and an attention-based adjacency matrix $A^{att}$, which can be produced by either GAT~(\cref{eq:gat}) or CAT~(\cref{eq:ours-score}), depending on the value of $\lambda_2$.

\subsection{PNA experiments}

In \cref{tabapp:performance_pna_1,tabapp:performance_pna_2}, we show the results---for the datasets described in \cref{app:dataset}---using \ours* and \ours in conjunction with the PNA model~\citep{corso2020principal}.
First, we note that standard PNA works quite well in most cases.
Second, if we focus on \cref{tabapp:performance_pna_1}, we observe that the standard attention models (\ie, PNAGAT and PNAGATv2) perform significantly worse than the other approaches, in particular on datasets with large average degree, \eg, on \spname{Amazon Computers}.
\removed{2}{We suspect that this could an artefact of the current implementation, as we reuse the PNA weights to compute the attention scores to keep the same number of parameters ($\mW_q = \mW_k = \mW_v$).
While this works well in the main paper, the role of $\mW_v$ is completely different from that of $\mW_q$ and $\mW_v$, which may hinder learning.
}%
Finally, we observe that the \ours* models (\ie, L-PNACAT and L-PNACATv2) drastically improve the performance of their attention counterparts and achieve similar performance as the PNA model, with lower performance on the \spname{GitHub} and \spname{Facebook} datasets and higher performance on \spname{Cora} and \spname{CiteSeer}.

	To keep the same number of parameters, we reuse the PNA weights to compute the attention scores ($\mW_q = \mW_k = \mW_v$). However, this could be detrimental, as the role of $\mW_v$ is completely different from that of $\mW_q$ and $\mW_v$. 
	\Cref{tabapp:performance_pna_3,tabapp:performance_pna_4} show the same experiments as before, but using different parameters to compute the keys and queries (\ie, $\mW_q = \mW_k \neq \mW_v$). We observe that the increase of parameters generally helps both CAT and L-CAT models, now outperform the base PNA model in some settings.

\begin{table*}[!ht]
\centering
\caption{Test accuracy (\%) of the PNA \citep{corso2020principal} models for different datasets (sorted by average node degree), averaged over ten runs. Bold numbers are statistically different to their baseline model ($\alpha = 0.05$). Best average performance is underlined.}
\label{tabapp:performance_pna_1}
\resizebox{\textwidth}{!}{
\begin{tabular}{lcccccc}
\toprule
Dataset & \thead{\spname{Amazon}\\\spname{Computers}} & \thead{\spname{Amazon}\\\spname{Photo}} & \spname{GitHub} & \thead{\spname{Facebook}\\\spname{PagePage}} & \thead{\spname{Coauthor}\\\spname{Physics}} & \spname{TwitchEN} \\
    Avg. Deg. & 35.76 & 31.13 & 15.33 & 15.22 & 14.38 & 10.91 \\ 
\midrule
PNA& \best{86.51 $\pm$ 1.22}& \best{93.23 $\pm$ 0.65}& \best{82.33 $\pm$ 0.51}& \best{94.28 $\pm$ 0.34}&  96.09 $\pm$ 0.14& \best{59.25 $\pm$ 1.19} \\
\midrule
PNAGAT&  57.59 $\pm$ 10.19&  74.78 $\pm$ 8.74&  72.77 $\pm$ 2.06&  71.49 $\pm$ 11.23&  96.05 $\pm$ 0.25&  54.22 $\pm$ 3.02 \\
PNACAT &  \bfseries 81.48 $\pm$ 3.81&  \bfseries 91.73 $\pm$ 1.24&  75.55 $\pm$ 3.33&  \bfseries 93.10 $\pm$ 0.41&  96.16 $\pm$ 0.15&  \bfseries 59.11 $\pm$ 1.94 \\
L-PNACAT &  \bfseries 86.45 $\pm$ 1.42&  \bfseries 92.76 $\pm$ 0.74&  \bfseries 78.74 $\pm$ 2.91&  \bfseries 93.59 $\pm$ 0.39& \best{96.24 $\pm$ 0.13}&  \bfseries 59.12 $\pm$ 2.74 \\
\midrule
PNAGATv2&  36.93 $\pm$ 4.07&  60.13 $\pm$ 4.81&  73.93 $\pm$ 1.89&  58.91 $\pm$ 3.42&  95.61 $\pm$ 0.29&  54.45 $\pm$ 1.60 \\
PNACATv2 &  \bfseries 79.08 $\pm$ 2.62&  \bfseries 88.61 $\pm$ 3.24&  75.11 $\pm$ 2.79&  \bfseries 92.77 $\pm$ 0.50&  \bfseries 96.06 $\pm$ 0.18&  \bfseries 56.72 $\pm$ 2.43 \\
L-PNACATv2 &  \bfseries 85.10 $\pm$ 1.70&  \bfseries 92.19 $\pm$ 0.55&  \bfseries 79.79 $\pm$ 1.40&  \bfseries 93.54 $\pm$ 0.36&  \bfseries 96.03 $\pm$ 0.19&  \bfseries 58.19 $\pm$ 1.53 \\
\bottomrule
\end{tabular}
}
\end{table*}

\begin{table*}[!ht]
\centering
\caption{Test accuracy (\%) of the PNA \citep{corso2020principal} models for different datasets (sorted by average node degree), averaged over ten runs. Bold numbers are statistically different to their baseline model ($\alpha = 0.05$). Best average performance is underlined.}
\label{tabapp:performance_pna_2}
\resizebox{\textwidth}{!}{
\begin{tabular}{lccccc}
\toprule
Dataset & \spname{CoauthorCS} & \spname{DBLP} & \spname{PubMed} & \spname{Cora} & \spname{CiteSeer}  \\
Avg. Deg. & 8.93 &  5.97 &  4.5 & 3.9 &  2.74 \\ \midrule
PNA& \best{93.30 $\pm$ 0.31}&  83.37 $\pm$ 0.32&  88.37 $\pm$ 0.73&  84.94 $\pm$ 1.19&  73.92 $\pm$ 0.97 \\
\midrule
PNAGAT&  92.46 $\pm$ 0.95&  83.42 $\pm$ 0.39& \best{88.40 $\pm$ 0.33}&  84.67 $\pm$ 0.69&  74.64 $\pm$ 0.82 \\
PNACAT&  92.90 $\pm$ 0.24&  83.35 $\pm$ 0.40&  88.24 $\pm$ 0.30&  \bfseries 85.58 $\pm$ 1.00&  74.94 $\pm$ 1.68 \\
L-PNACAT&  93.11 $\pm$ 0.24&  83.21 $\pm$ 0.55&  88.22 $\pm$ 0.40& \best{\bfseries 85.77 $\pm$ 1.01}& \best{75.08 $\pm$ 1.05} \\
\midrule
PNAGATv2&  90.14 $\pm$ 0.82&  83.37 $\pm$ 0.34&  88.14 $\pm$ 0.45&  85.04 $\pm$ 0.86&  74.50 $\pm$ 1.18 \\
PNACATv2&  \bfseries 92.78 $\pm$ 0.27&  83.22 $\pm$ 0.38&  88.28 $\pm$ 0.30&  85.41 $\pm$ 0.98&  74.42 $\pm$ 1.11 \\
L-PNACATv2 &  \bfseries 93.02 $\pm$ 0.37& \best{83.54 $\pm$ 0.65}&  88.23 $\pm$ 0.58&  85.48 $\pm$ 0.98&  74.76 $\pm$ 1.57 \\
\bottomrule
\end{tabular}
}
\end{table*}

\begin{table*}[!ht]
	\centering
	\caption{Test accuracy (\%) of the PNA \citep{corso2020principal} extended models with $\mW_q = \mW_k \neq \mW_v$ for different datasets, averaged over ten runs. Bold numbers are statistically different to their baseline model ($\alpha = 0.05$). Best average performance is underlined.}
	\label{tabapp:performance_pna_3}
	\resizebox{\textwidth}{!}{
	\begin{tabular}{lcccccc}
		\toprule
		Dataset & \thead{\spname{Amazon}\\\spname{Computers}} & \thead{\spname{Amazon}\\\spname{Photo}} & \spname{GitHub} & \thead{\spname{Facebook}\\\spname{PagePage}} & \thead{\spname{Coauthor}\\\spname{Physics}} & \spname{TwitchEN} \\
		Avg. Deg. & 35.76 & 31.13 & 15.33 & 15.22 & 14.38 & 10.91 \\ 
		\midrule
		PNA & \best{86.51 $\pm$ 1.22}& \best{93.23 $\pm$ 0.65}& \best{82.33 $\pm$ 0.51}& \best{94.28 $\pm$ 0.34}&  96.09 $\pm$ 0.14& \best{59.25 $\pm$ 1.19} \\ \midrule
		PNAGAT &  48.65 $\pm$ 19.25&  68.01 $\pm$ 20.32&  72.97 $\pm$ 1.07&  70.17 $\pm$ 12.02&  96.02 $\pm$ 0.34&  53.27 $\pm$ 2.54 \\
		PNACAT &  \bfseries 83.45 $\pm$ 2.60&  \bfseries 91.62 $\pm$ 1.30&  75.35 $\pm$ 2.71&  \bfseries 93.31 $\pm$ 0.55&  96.22 $\pm$ 0.13&  \bfseries 59.23 $\pm$ 2.25 \\
		L-PNACAT & {   \bfseries 87.18 $\pm$ 1.22 }&  \bfseries 92.79 $\pm$ 0.63&  \bfseries 79.64 $\pm$ 2.54& {   \bfseries 93.78 $\pm$ 0.39 }& {   \bfseries 96.31 $\pm$ 0.17 }&  \bfseries 59.09 $\pm$ 2.50 \\
		\midrule
		PNAGATv2 &  39.49 $\pm$ 4.09&  62.19 $\pm$ 11.30&  73.97 $\pm$ 1.67&  63.00 $\pm$ 4.95&  95.83 $\pm$ 0.36&  55.21 $\pm$ 1.05 \\
		PNACATv2 & \bfseries 81.20 $\pm$ 3.63&  \bfseries 91.32 $\pm$ 0.80&  74.57 $\pm$ 2.18&  \bfseries 92.98 $\pm$ 0.36&  \bfseries 96.14 $\pm$ 0.16&  56.21 $\pm$ 2.01 \\
		L-PNACATv2 &  \bfseries 86.22 $\pm$ 0.83& {   \bfseries 92.98 $\pm$ 0.89 }& {   \bfseries 79.78 $\pm$ 2.48 }&  \bfseries 93.44 $\pm$ 0.37&  \bfseries 96.13 $\pm$ 0.12& {   \bfseries 60.26 $\pm$ 1.25 } \\
		\bottomrule
	\end{tabular}
	}
\end{table*}

\begin{table*}[!ht]
	\centering
	\caption{Test accuracy (\%) of the PNA \citep{corso2020principal} extended models with $\mW_q = \mW_k \neq \mW_v$ for different datasets, averaged over ten runs. Bold numbers are statistically different to their baseline model ($\alpha = 0.05$). Best average performance is underlined.}
	\label{tabapp:performance_pna_4}
	\resizebox{\textwidth}{!}{
	\begin{tabular}{lccccc}
		\toprule
		Dataset & \spname{CoauthorCS} & \spname{DBLP} & \spname{PubMed} & \spname{Cora} & \spname{CiteSeer}  \\
		Avg. Deg. & 8.93 &  5.97 &  4.5 & 3.9 &  2.74 \\ \midrule
		PNA& \best{93.30 $\pm$ 0.31}&  83.37 $\pm$ 0.32&  88.37 $\pm$ 0.73&  84.94 $\pm$ 1.19&  73.92 $\pm$ 0.97 \\ \midrule
		PNAGAT&  92.50 $\pm$ 0.46&  83.22 $\pm$ 0.45&  88.43 $\pm$ 0.29&  84.89 $\pm$ 1.15&  75.76 $\pm$ 1.29 \\
		PNAGAT&  \bfseries 92.97 $\pm$ 0.47&  83.28 $\pm$ 0.59&  88.27 $\pm$ 0.43&  85.09 $\pm$ 0.70&  75.44 $\pm$ 1.51 \\
		PNAGAT& {   \bfseries 93.17 $\pm$ 0.30 }& {   83.50 $\pm$ 0.29 }& {   88.54 $\pm$ 0.45 }& {   85.63 $\pm$ 0.92 }&  75.22 $\pm$ 1.12 \\
		\midrule
		PNAGATv2&  90.00 $\pm$ 1.01&  83.40 $\pm$ 0.48&  88.14 $\pm$ 0.31&  85.16 $\pm$ 0.91& {   76.14 $\pm$ 1.33 } \\
		PNAGATv2&  \bfseries 92.74 $\pm$ 0.20&  \worse{83.05 $\pm$ 0.49}&  \bfseries 88.38 $\pm$ 0.31&  85.21 $\pm$ 0.83&  75.80 $\pm$ 1.26 \\
		PNAGATv2&  \bfseries 93.02 $\pm$ 0.30&  83.24 $\pm$ 0.44&  88.28 $\pm$ 0.35&  85.04 $\pm$ 0.94&  75.80 $\pm$ 1.19 \\
		\bottomrule
	\end{tabular}
}
\end{table*}

\newpage
\subsection{GCNII experiments}

Similarly, we have run the experiments from \cref{sec:experiments-graphgym}, this time combining GCNII~\citep{chen2020simple} with GAT, \ours, and \ours* as explained above.
Results are shown in \cref{tabapp:performance_gcn2_1,tabapp:performance_gcn2_2}, in which we can observe that the baseline model obtains the best results so far in the manuscript (in comparison with both GCN and PNA).
And just as before, we observe again that \ours and \ours* always improve with respect to their base models, staying on par with the baseline GNCII model and, sometimes, even outperforming the baseline model on average (\eg, \spname{Coauthor Physics}, \spname{TwitchEN}).
As with the experiments for PNA, \cref{tabapp:performance_gcn2_3,tabapp:performance_gcn2_4} shows the results when the attention matrices are different from the value matrix ($\mW_q = \mW_k \neq \mW_v$). We can similarly observe that most of the results are improved with the additional parameters, beating the baseline model in different datasets.

\begin{table*}
	\centering
	\caption{Test accuracy (\%) of the GCNII \citep{chen2020simple} models for different datasets, averaged over ten runs. Bold numbers are statistically different to their baseline model ($\alpha = 0.05$). Best average performance is underlined.}
	\label{tabapp:performance_gcn2_1}
	\resizebox{\textwidth}{!}{
\begin{tabular}{lcccccc}
	\toprule
	Dataset & \thead{\spname{Amazon}\\\spname{Computers}} & \thead{\spname{Amazon}\\\spname{Photo}} & \spname{GitHub} & \thead{\spname{Facebook}\\\spname{PagePage}} & \thead{\spname{Coauthor}\\\spname{Physics}} & \spname{TwitchEN} \\
	Avg. Deg. & 35.76 & 31.13 & 15.33 & 15.22 & 14.38 & 10.91 \\ 
	\midrule
	GCNII & \best{90.82 $\pm$ 0.20}& \best{95.51 $\pm$ 0.48}& \best{84.11 $\pm$ 0.76}& \best{94.03 $\pm$ 0.30}&  96.58 $\pm$ 0.11&  60.94 $\pm$ 1.66 \\ \midrule
	GCNIIGAT&  89.04 $\pm$ 0.87&  94.74 $\pm$ 0.57&  82.34 $\pm$ 0.64&  91.18 $\pm$ 0.82&  96.69 $\pm$ 0.13&  57.76 $\pm$ 1.76 \\
	GCNIICAT&  \better{89.83 $\pm$ 0.42}&  \better{95.31 $\pm$ 0.25}&  \better{83.15 $\pm$ 0.51}&  \better{93.25 $\pm$ 0.37}&  96.69 $\pm$ 0.09&  \better{60.51 $\pm$ 1.12} \\
	L-GCNIICAT &  \better{90.03 $\pm$ 0.42}&  95.23 $\pm$ 0.39&  \better{83.50 $\pm$ 0.57}&  \better{93.71 $\pm$ 0.33}& \best{\better{96.87 $\pm$ 0.14}}& \best{\better{61.14 $\pm$ 1.64}} \\
	\midrule
	GCNIIGATv2&  84.26 $\pm$ 2.80&  89.23 $\pm$ 5.30&  81.23 $\pm$ 0.45&  83.82 $\pm$ 1.24&  96.14 $\pm$ 0.28&  56.25 $\pm$ 1.56 \\
	GCNIICATv2&  \better{89.59 $\pm$ 0.45}&  \better{95.03 $\pm$ 0.55}&  \better{82.45 $\pm$ 0.30}&  \better{92.55 $\pm$ 0.52}&  \better{96.50 $\pm$ 0.09}&  \better{59.04 $\pm$ 1.49} \\
	L-GCNIICATv2&  \better{89.81 $\pm$ 0.48}&  \better{95.24 $\pm$ 0.35}&  \better{83.05 $\pm$ 0.49}&  \better{93.68 $\pm$ 0.35}&  \better{96.75 $\pm$ 0.11}&  \better{61.10 $\pm$ 1.11} \\
	\bottomrule
\end{tabular}
}
\end{table*}

\begin{table*}
	\centering
	\caption{Test accuracy (\%) of the GCNII \citep{chen2020simple} models for different datasets, averaged over ten runs. Bold numbers are statistically different to their baseline model ($\alpha = 0.05$). Best average performance is underlined.}
	\label{tabapp:performance_gcn2_2}
	\resizebox{\textwidth}{!}{
\begin{tabular}{lccccc}
	\toprule
	Dataset & \spname{CoauthorCS} & \spname{DBLP} & \spname{PubMed} & \spname{Cora} & \spname{CiteSeer}  \\
	Avg. Deg. & 8.93 &  5.97 &  4.5 & 3.9 &  2.74 \\ \midrule
	GCNII& \best{95.36 $\pm$ 0.18}&  83.86 $\pm$ 0.14&  89.05 $\pm$ 0.28& \best{86.49 $\pm$ 0.79}&  76.46 $\pm$ 0.71 \\
	\midrule
	GCNIIGAT&  95.32 $\pm$ 0.27&  83.45 $\pm$ 0.60&  88.24 $\pm$ 0.34&  85.72 $\pm$ 1.05&  75.78 $\pm$ 0.77 \\
	GCNIICAT&  95.12 $\pm$ 0.25&  83.86 $\pm$ 0.23&  88.65 $\pm$ 0.40&  85.51 $\pm$ 0.95&  \better{76.60 $\pm$ 0.50} \\
	L-GCNIICAT &  95.30 $\pm$ 0.29&  83.76 $\pm$ 0.26&  \better{88.72 $\pm$ 0.35}&  85.48 $\pm$ 0.96& \best{\better{76.76 $\pm$ 0.51}} \\
	\midrule
	GCNIIGATv2&  93.37 $\pm$ 0.73&  83.70 $\pm$ 0.42&  88.49 $\pm$ 0.34&  85.82 $\pm$ 1.35&  75.98 $\pm$ 0.71 \\
	GCNIICATv2&  \better{95.01 $\pm$ 0.32}&  83.93 $\pm$ 0.36&  88.60 $\pm$ 0.27&  85.72 $\pm$ 1.03&  76.62 $\pm$ 0.63 \\
	L-GCNIICATv2&  \better{95.29 $\pm$ 0.23}& \best{83.93 $\pm$ 0.41}& \best{\better{89.12 $\pm$ 0.35}}&  85.73 $\pm$ 0.88&  76.22 $\pm$ 0.98 \\
	\bottomrule
\end{tabular}
}
\end{table*}

\begin{table*}
	\centering
	\caption{Test accuracy (\%) of the GCNII \citep{chen2020simple} extended models with $\mW_q = \mW_k \neq \mW_v$ for different datasets, averaged over ten runs. Bold numbers are statistically different to their baseline model ($\alpha = 0.05$). Best average performance is underlined.}
	\label{tabapp:performance_gcn2_3}
	\resizebox{\textwidth}{!}{
\begin{tabular}{lcccccc}
	\toprule
		Dataset & \thead{\spname{Amazon}\\\spname{Computers}} & \thead{\spname{Amazon}\\\spname{Photo}} & \spname{GitHub} & \thead{\spname{Facebook}\\\spname{PagePage}} & \thead{\spname{Coauthor}\\\spname{Physics}} & \spname{TwitchEN} \\
	Avg. Deg. & 35.76 & 31.13 & 15.33 & 15.22 & 14.38 & 10.91 \\ 
	\midrule
	GCNII & \best{90.82 $\pm$ 0.20}&  95.51 $\pm$ 0.48& \best{84.11 $\pm$ 0.76}& \best{94.03 $\pm$ 0.30}&  96.58 $\pm$ 0.11& \best{60.94 $\pm$ 1.66} \\
	\midrule
	GCNIIGAT&  89.24 $\pm$ 0.59&  94.66 $\pm$ 0.59&  82.45 $\pm$ 0.65&  90.90 $\pm$ 0.71& \best{96.90 $\pm$ 0.16}&  58.12 $\pm$ 2.02 \\
	GCNIICAT &  \better{89.94 $\pm$ 0.40}&  95.03 $\pm$ 0.37&  \better{83.12 $\pm$ 0.37}&  \better{93.39 $\pm$ 0.31}&  \worse{96.59 $\pm$ 0.07}&  59.60 $\pm$ 0.76 \\
	L-GCNIICAT &  \better{90.35 $\pm$ 0.46}&  \best{\better{95.53 $\pm$ 0.35}} &  \better{83.48 $\pm$ 0.47}&  \better{93.63 $\pm$ 0.39}&  \worse{96.80 $\pm$ 0.09}&  \better{60.77 $\pm$ 2.15} \\
	\midrule
	GCNIIGATv2 &  85.70 $\pm$ 2.58&  91.24 $\pm$ 2.43&  81.43 $\pm$ 0.39&  84.59 $\pm$ 0.79&  96.55 $\pm$ 0.20&  55.23 $\pm$ 2.15 \\
	GCNIICATv2 &  \better{89.56 $\pm$ 0.67}&  \better{95.46 $\pm$ 0.54}&  \better{82.50 $\pm$ 0.47}&  \better{93.04 $\pm$ 0.40}&  96.55 $\pm$ 0.10&  \better{59.57 $\pm$ 1.45} \\
	L-GCNIICATv2 &  \better{90.24 $\pm$ 0.32}& \best{\better{95.53 $\pm$ 0.28}}&  \better{83.34 $\pm$ 0.45}&  \better{93.67 $\pm$ 0.47}&  96.73 $\pm$ 0.14&  \better{60.65 $\pm$ 1.13} \\
	\bottomrule
\end{tabular}
}
\end{table*}

\begin{table*}
	\centering
	\caption{Test accuracy (\%) of the GCNII \citep{chen2020simple} extended models with $\mW_q = \mW_k \neq \mW_v$ for different datasets, averaged over ten runs. Bold numbers are statistically different to their baseline model ($\alpha = 0.05$). Best average performance is underlined.}
	\label{tabapp:performance_gcn2_4}
	\resizebox{\textwidth}{!}{
		\begin{tabular}{lccccc}
			\toprule
				Dataset & \spname{CoauthorCS} & \spname{DBLP} & \spname{PubMed} & \spname{Cora} & \spname{CiteSeer}  \\
			Avg. Deg. & 8.93 &  5.97 &  4.5 & 3.9 &  2.74 \\ \midrule
			GCNII &  95.36 $\pm$ 0.18&  83.86 $\pm$ 0.14& \best{89.05 $\pm$ 0.28}&  \best{86.49 $\pm$ 0.79} &  76.46 $\pm$ 0.71 \\
			\midrule
			GCNIIGAT&  95.36 $\pm$ 0.20&  83.60 $\pm$ 0.32&  88.33 $\pm$ 0.35&  85.26 $\pm$ 1.19&  76.30 $\pm$ 0.78 \\
			GCNIICAT &  \worse{95.20 $\pm$ 0.12}& \best{\better{83.89 $\pm$ 0.29}}&  88.45 $\pm$ 0.29&  86.44 $\pm$ 1.22&  76.70 $\pm$ 0.60 \\
			L-GCNIICAT & \best{95.47 $\pm$ 0.16}&  83.70 $\pm$ 0.40&  88.08 $\pm$ 0.47&  86.02 $\pm$ 1.43&  76.54 $\pm$ 0.59 \\
			\midrule
			GCNIIGATv2&  93.97 $\pm$ 0.57&  83.67 $\pm$ 0.24&  88.24 $\pm$ 0.16&  85.28 $\pm$ 1.11&  76.58 $\pm$ 0.64 \\
			GCNIICATv2 &  \better{95.05 $\pm$ 0.33}&  83.78 $\pm$ 0.35&  88.35 $\pm$ 0.34&  \better{86.49 $\pm$ 0.90}&  \worse{75.28 $\pm$ 0.84} \\
			L-GCNIICATv2 &  \better{95.45 $\pm$ 0.18}&  83.86 $\pm$ 0.24&  88.06 $\pm$ 0.32& \best{86.49 $\pm$ 1.31}& \best{76.80 $\pm$ 0.42} \\
			\bottomrule
		\end{tabular}
	}
\end{table*}

\end{document}